\documentclass[twoside]{article}

\usepackage[accepted]{icml2024}

\usepackage{graphicx}
\usepackage{adjustbox}
\usepackage{pgfplots}
\usepackage{tikz}
\usetikzlibrary{spy}
\usepackage{graphicx}
\usepackage{subcaption}
\usepackage{xspace}
\usepackage{amsfonts}
\usepackage{algorithm}
\usepackage{algorithmic}
\usepackage{mkolar_definitions}
\usepackage{amsthm,amsmath,amssymb}
\usepackage{bbm}
\usepackage{amsfonts}
\usepackage{commath}
\usepackage{color}
\usepackage{xcolor} 
\usepackage{tablefootnote}
\usepackage{enumitem}
\usepackage{booktabs} 
\usepackage{hyperref}
\definecolor{mydarkblue}{rgb}{0,0.08,0.45}
\hypersetup{
    colorlinks=true,
    linkcolor=mydarkblue,
    citecolor=mydarkblue,
    filecolor=mydarkblue,
    urlcolor=mydarkblue,
}

\newcommand{\Unif}{\mathrm{Unif}}
\newcommand{\prob}{\mathrm{Pr}}

\DeclareMathOperator\Onehot{Onehot}

\newcommand{\MDP}{\Mcal}
\newcommand{\pie}{\pi^{\mathrm{exp}}}
\newcommand{\incrdata}{D}
\newcommand{\incrdist}{\Dcal^{\mathrm{exp}}}

\newcommand{\Ensem}{E}
\newcommand{\ensem}{e}
\newcommand{\nrepeat}{T}

\newcommand{\iterloss}{F}
\newcommand{\Reg}{\mathrm{Reg}}
\newcommand{\smooth}{\sigma}
\newcommand{\based}{d_0}
\newcommand{\poisson}{\mathrm{Poi}}

\newcommand{\poi}{X}
\newcommand{\poiparam}{\lambda}
\newcommand{\poipp}{\tilde{\lambda}}
\newcommand{\incrdistbase}{\Dcal_0}
\newcommand{\pertdata}{Q}

\newcommand{\est}{g}

\newcommand{\distp}{\mathcal{P}}
\newcommand{\distq}{\mathcal{Q}}
\newcommand{\distr}{\mathcal{R}}
\newcommand{\tv}{\mathrm{TV}}
\newcommand{\group}{M}
\newcommand{\elem}{\zeta}
\newcommand{\elemc}{\xi}

\newcommand{\estgap}{\ensuremath{\mathrm{EstGap}}\xspace}

\newcommand{\algms}{\ensuremath{\textsc{MFTPL-P}}\xspace}
\newcommand{\algm}{\ensuremath{\textsc{MFTPL}}\xspace}
\newcommand{\alg}{\ensuremath{\textsc{Bootstrap-DAgger}}\xspace}

\newcommand{\MP}{\ensuremath{\textsc{MP-25}}}
\newcommand{\MPX}{\ensuremath{\textsc{MP-25$(X)$}}\xspace}
\newcommand{\MPF}{\ensuremath{\textsc{MP-25(15)}}\xspace}
\newcommand{\bd}{\ensuremath{\textsc{BD}}\xspace}
\newcommand{\bda}{\ensuremath{\textsc{BD-1}}\xspace}
\newcommand{\bdb}{\ensuremath{\textsc{BD-5}}\xspace}
\newcommand{\bdc}{\ensuremath{\textsc{BD-25}}\xspace}

\newcommand{\nsl}{\ensuremath{\textsc{SL}}\xspace}

\newcommand{\algbc}{Behavior Cloning\xspace}
\newcommand{\bc}{\ensuremath{\textsc{BC}}\xspace}
\newcommand{\dagge}{\ensuremath{\textsc{DAgger}}\xspace}

\newcommand{\trainbase}{\ensuremath{\textsc{Train}\textsc{-}\textsc{Base}}\xspace}
\newcommand{\aggpolicy}{\ensuremath{\textsc{Aggregate}\textsc{-}\textsc{Policies}}\xspace}

\newtheorem{assumption}{Assumption}

\begin{document}

%

%


\twocolumn[
\icmltitle{Agnostic Interactive Imitation Learning: New Theory and Practical Algorithms}



\icmlsetsymbol{equal}{*}

\begin{icmlauthorlist}


\icmlauthor{Yichen Li}{yyy}
\icmlauthor{Chicheng Zhang}{yyy}
\end{icmlauthorlist}

\icmlaffiliation{yyy}{Department of Computer Science, University of Arizona, Tucson, AZ, USA}

\icmlcorrespondingauthor{Yichen Li}{yichenl@arizona.edu}
\icmlcorrespondingauthor{Chicheng Zhang}{chichengz@cs.arizona.edu}

\icmlkeywords{Machine Learning, ICML}

\vskip 0.3in
]


\printAffiliationsAndNotice{}  

\begin{abstract}






We study interactive imitation learning, where a learner interactively queries a demonstrating expert for action annotations, aiming to learn a policy that has performance competitive with the expert, using as few annotations as possible.
We focus on the general agnostic setting where the expert demonstration policy may not be contained in the policy class used by the learner. 
We propose a new oracle-efficient algorithm \algms (abbreviation for Mixed Follow the Perturbed Leader with Poisson perturbations) with provable finite-sample guarantees, under the assumption that the learner is given access to samples from some ``explorative'' distribution over states. Our guarantees hold for any policy class, which is considerably broader than prior state of the art. 
We further propose \alg, a more practical variant that does not require additional sample access.
Empirically, \algms and \alg notably surpass online and offline imitation learning baselines in continuous control tasks.

\end{abstract}




\section{Introduction}



Imitation learning (IL) is a learning paradigm for training sequential decision making agents using expert demonstrations~\cite{bagnell2015invitation}. It seeks to learn a policy whose performance is on par with the expert, with as few expert demonstrations as possible. Compared to reinforcement learning which potentially requires intricate reward engineering,
imitation learning sidesteps this challenge,
making it apt for complex decision making problems~\cite{osa2018algorithmic}.


The format of expert demonstrations often comprises of (state, action) pairs, each representing the expert's action taken under respective state. At first sight, imitation learning seems similar to supervised learning, where one would like to train a policy that maps states to actions recommended by the expert. However, it is now well-known that  these two problems are different in nature~\cite{pomerleau1988alvinn}: compared to supervised learning, imitation learning agents are faced with a new fundamental challenge of \emph{covariate shift}: the state distribution the learner sees at the test stage can be very different from that at the training stage. To elaborate, upon deploying the trained policy, its imperfection in mimicking the expert can result in \emph{compounding error}: the learner may 
make an initial mistake and enter a state not covered by the distribution of expert demonstrations, and takes another incorrect action due to lack of training data coverage in this state. This may lead to sequences of consecutive states that were unseen in the expert demonstrations, ultimately resulting in poor performance. Addressing such problem calls for better data collection methods beyond naively collecting expert trajectories; as summarized by~\cite{bojarski2016end} in the context of imitation learning for autonomous driving, 
``Training with data from only the human driver is not sufficient. The network must learn how to recover from mistakes''.




To cope with the covariate shift challenge, the interactive imitation learning paradigm has been proposed and used in practice~\cite{ross2010efficient,ross2011reduction,ross2014reinforcement,sun2017deeply,pan2020imitation,celemin2022interactive}. Instead of having only access to offline expert trajectories, in interactive imitation learning, the learner has the freedom to select states to ask for expert annotations. This allows more targeted feedback and gives the learning agent opportunities to learn to ``recover from mistakes'', and thus achieve better performance.


Recent years have seen many exciting developments on designing provably efficient interactive imitation learning algorithms, using new algorithmic approaches such as 
distribution / moment matching~\cite{ke2020imitation,swamy2021moments} and online classification / regression~\cite{rajaraman2021value,sekhari2023selective}. Most of these works rely on some realizability assumption: they either assume that the learner is given a policy class that contains the expert policy~\cite{rajaraman2021value,sekhari2023selective,sun2017deeply}, or that the learner is given some function class that contains reward or value functions of the underlying environment MDP~\cite{swamy2021moments,swamy2022minimax,sun2019provably}.

Although realizability assumptions are useful in driving the development of theory and practical algorithms, misspecified model settings are common in practice (e.g. in  autonomous driving~\cite{pan2020imitation}). Thus, it is important to design 
policy search-based
imitation learning algorithms that can work in general agnostic (i.e., nonrealizable) settings. 
However, the fundamental computational and statistical limits of imitation learning in the general agnostic  setting are not well-understood. 
An important line of work~\cite{ross2011reduction,ross2014reinforcement,sun2017deeply} tackles this question by establishing a general reduction from agnostic interactive imitation learning to no-regret online learning~\cite{shalev2011online} (see Section~\ref{sec:prelims} below for a recap). Its key intuition is that, with access to a learner that can perform online prediction by swiftly adapting to nonstationary data, one can use it to find a policy that mostly agrees with the expert on its own state visitation distribution. In other words, online learning can be utilized to learn a policy that recovers from its own mistakes.
Utilizing this general reduction framework, agnostic imitation learning algorithms with provable computational 
and finite-sample statistical efficiency guarantees have been proposed~\cite{cheng2019accelerating,cheng2019predictor,li2022efficient}. In the discrete-action setting, the development so far has been limited: state-of-the-art efficient algorithms~\cite{li2022efficient} rely on a strong technical ``small separator'' assumption on the policy class, which is only known to hold for a few policy classes (such as disjunctions and linear classes)~\cite{syrgkanis2016efficient,dudik2020oracle}.
This motivates our main question: can we design computationally and statistically efficient imitation learning algorithms in the general agnostic setting for \emph{general policy classes}?











Our contributions are twofold:





\begin{itemize}
\item Theoretically, we design a provably efficient online imitation learning algorithm that enjoys no-regret guarantees for discrete action spaces for general policy classes. 
The no-regret property guarantees the learning agent's swift adaptation to data distributions it encounters, ensuring competitiveness to the expert policy~\cite{ross2011reduction}.
Specifically, our algorithm, \emph{Mixed Follow the Perturbed Leader with Poisson Perturbations} 
(abbrev. \algms), assumes sample access to a distribution $d_0$ that ``covers'' the state visitation distributions of all policies in the policy class of interest. Algorithmically, \algms can be viewed as maintaining an ensemble of policies, each member of which is trained using  historical expert demonstration data combined with noisy perturbation examples drawn from $d_0$. 

Inspired by recent analysis of efficient smoothed online learning algorithms~\cite{haghtalab2022oracle,block2022smoothed}, 
we prove that \algms: (1) has a sublinear regret guarantee, which can be easily converted into a guarantee of its output policy's suboptimality; (2) is computationally efficient, assuming access to an offline  learning oracle.
Our computational efficiency result relies on arguably much weaker assumptions than previous state-of-the-art efficient learning algorithms, whose guarantees require strong assumptions on the policy class~\cite{li2022efficient} or convexity of loss function~\cite{cheng2018convergence,cheng2019accelerating,cheng2019predictor,sun2017deeply}.





%

\item Empirically, we verify the utility of using sample-based perturbations in \algms. Furthermore, 
we evaluate \algms and show that it outperforms the popular \dagge and \algbc baselines across several continuous control benchmarks. 
Inspired by the ensemble nature of \algms, we also propose and evaluate a practical approximation of it, namely \alg, that avoids sample access to $d_0$ and achieves competitive performance. 
We also investigate the expert demonstration data collected by \alg and show that it gathers pertinent expert demonstration data more efficiently than \dagge.




\end{itemize}

\section{Related Work}
\label{sec:related_works}
\textbf{Imitation Learning with Interactive Expert.} Existing works in interactive imitation learning established a solid foundation to tackle covariate shift,  with the help of an interactive demonstration expert. Early works reduced interactive imitation learning to offline learning~\cite{ross2010efficient,daume2009search} and~\cite{ross2011reduction} proposed a general reduction from interactive imitation learning to online learning.
This major line of research~\cite{ross2011reduction,ross2014reinforcement,sun2017deeply,cheng2018convergence,cheng2019predictor,cheng2019accelerating,rajaraman2021value,li2022efficient} provided provably efficient online regret guarantees, which can be translated to guarantees of learned policy's competitiveness with expert policy. 
It has been shown in~\cite{rajaraman2021value} that interactive imitation learning can be significantly more sample efficient in favorable environments than its offline counterpart.
Recently,~\cite{sekhari2023selective} reduced interactive imitation learning to online regression, which assumes that the expected value of expert annotation as a function of state lies in a real-valued regressor class. 




As discussed in the introduction section, most of these works make some realizability assumptions.
In the absence of such realizability assumption (i.e., the agnostic setting), 
existing  guarantees  require 
convexity of the loss functions with respect to policy's pre-specified parameters (which may not hold for general policy classes) ~\cite{cheng2018convergence,cheng2019accelerating,cheng2019predictor,sun2017deeply}
or are only applicable to a few special policy classes~\cite{li2022efficient}. 

Another line of work studies interactive imitation learning with \emph{online expert intervention feedback}~\cite{zhang2016query,menda2017dropoutdagger,menda2019ensembledagger,cui2019uncertainty,kelly2019hg,hoque2021thriftydagger,spencer2022expert}; here it is assumed that the learner has access to an expert in real time; 
the learner may cede control to the expert to request for  demonstrations (machine-gated setting) or the expert can actively intervene (human-gated setting). Different from our learning objective, these works have a focus on ensuring training-time safety. 

\textbf{Imitation Learning without Interactive Expert.} 
Given only offline expert demonstrations, \algbc (\bc) naively reduced imitation learning to offline classification~\cite{ross2010efficient,syed2010reduction}.
By further assuming the ability to interact with the environment, generative adversarial imitation learning (GAIL)~\cite{ho2016generative} and following line of works~\cite{sun2019provably,spencer2021feedback} applied moment matching to tackle covariate shift.
\cite{chang2021mitigating} replaced the requirement of environment access by combining model estimation and pessimism.
DART \cite{laskey2017dart} accessed neighborhood states of the expert trajectories by injecting noise during the expert's demonstrations, achieving performance comparable to \dagge. 
DRIL~\cite{brantley2019disagreement} trained an ensemble of policies using expert demonstrations, then leveraged the variance among ensemble predictions as a cost, which was optimized through reinforcement learning together with the classification loss on expert set.
These works are different from ours, in that they do not assume access to an expert that can provide interactive demonstrations in the training process. 
~\cite{ke2020imitation,swamy2021moments,swamy2022minimax} formulates
imitation learning as a distribution matching problem, and further reduce it to solving two-player zero-sum games, which can be solved either interactively or offline, however their guarantees only hold under realizability assumptions. 
\textbf{Concentrability and Smoothness.}
One key assumption we make is that the learning agent has sample access of some covering distribution, so that all policy's visitation distributions are ``smooth'' with respect to it (Assumption~\ref{assumption:smooth-base}).
This allows us to design provably-efficient imitation learning algorithms using techniques for smoothed online learning~\cite{haghtalab2022oracle,block2022smoothed}.
Our assumption is related to the boundedness of the concentrability coefficient commonly used in offline reinforcement learning ~\cite{munos2007performance,munos2008finite,chen2019information,xie2020q,xie2021batch},
which was first introduced 
by Munos~\cite{munos2003error}. 
Concentrability in general holds for MDPs with ``noisy'' transitions (i.e. nontrivial probability of transitioning to all potential next states) but can also hold for deterministic transitions~\cite{szepesvari2005finite}. The concentrability assumptions most related to ours are  in~\cite{xie2020q,xie2022role}.
However, note that our Assumption~\ref{assumption:smooth-base} 
is solely with respect to distributions over states instead of (state, action) pairs. After all, unlike standard offline reinforcement learning, in IL, we neither seek to learn the optimal value function nor assume access to a candidate value function class.



\section{Preliminaries}
\label{sec:prelims}

\textbf{Basic notations.} Denote by $[N] = \cbr{1, \ldots, N}$. 
Given a set $\Ecal$, denote by $\Delta(\Ecal)$ the set of all probability distributions over it; when $\Ecal$ is finite, elements in $\Delta(\Ecal)$ can be represented by probability vectors in $\RR^{|\Ecal|}$: $\cbr{ (w_e)_{e \in \Ecal}: \sum_{e \in \Ecal} w_e = 1, w_e \geq 0, \forall e \in \Ecal}$.
Given a function $f: \Ecal \to \RR$, denote by $\| f \|_\infty := \max_{e \in \Ecal} |f(e)|$.

\textbf{Markov decision process.}
We study imitation learning for sequential decision making, where we model the environment as a  Markov decision process (MDP). 
An episodic MDP $\MDP$ is a tuple $(\Scal, \Acal, \rho, P, C, H)$, where $\Scal$ is the state space, $\Acal$ is the action space, $\rho$ is the initial state distribution, $P: \Scal \times \Acal \to \Delta(\Scal)$ is the transition probability distribution, $C: \Scal \times \Acal \to \Delta([0,1])$ is the cost distribution, 
$H \in \NN$ is the length of the time horizon. Without loss of generality, we assume that $\MDP$ is layered, in that $\Scal$ can be partitioned to $\cbr{\Scal_t}_{t=1}^H$, where for every step $t$, $s \in \Scal_t$ and action $a \in \Acal$, $P(\cdot \mid s, a)$ is supported on $\Scal_{t+1}$.

\textbf{Agent-environment interaction and policy.} An agent interacts with MDP in one episode by first observing initial state $s_1 \sim \rho$, and for every step $t = 1, \ldots, H$, takes an action $a_t$, receives cost $c_t \sim C(\cdot \mid s, a)$, and transitions to next state $s_{t+1} \sim \PP(\cdot \mid s_t, a_t)$. 
A stationary (history-independent) policy $\pi: \Scal \to \Delta(\Acal)$ maps a state to a distribution over actions, which can be used by the agent to take actions, i.e. $a_t \sim \pi(\cdot \mid s_t)$ for all $t$.

\textbf{Value Functions.} Given policy $\pi$, for every $t \in [H]$ and $s \in \Scal_t$, denote by $V^t_\pi(s) = \EE\sbr{ \sum_{t'=t}^{H} c_{t'} \mid s_t = s, \pi }$ and $Q^t_\pi(s, a) = \EE\sbr{ \sum_{t'=t}^{H} c_{t'} \mid (s_t, a_t) = (s,a), \pi }$ the state-value function and action-value function of $\pi$, respectively. When it is clear from context, we will abbreviate $V^t_\pi(s)$ and $Q^t_\pi(s,a)$ as $V_\pi(s)$ and $Q_\pi(s,a)$, respectively. Denote by $J(\pi) = \EE_{s_1 \sim \rho} \sbr{ V_\pi(s_1) }$ the expected cost of policy $\pi$.  Given policy $\pi$, denote by $d_\pi^{\MDP} (\cdot) = \frac1H \sum_{t=1}^H \PP( s_t = \cdot )$ the state visitation distribution of $\pi$ under $\MDP$; when it is clear from context, we will abbreviate $d_\pi^{\MDP}$ as $d_\pi$. 

To help the learner make decisions, the learner is given a policy class $\Bcal$, which is a structured set of deterministic policies $\pi: \Scal \to \Acal$. Denote by $\Pi_\Bcal = \cbr{ \pi_w(a \mid s) := \sum_{h \in \Bcal} w_h I(a = h(s)): w \in \Delta(\Bcal) }$ the mixed policy class induced by $\Bcal$. Denote by $\pie$ the expert's policy. The learning setting is said to be \emph{realizable} if it is known apriori that $\pie \in \Bcal$; otherwise, the learning setting is said to be  \emph{agnostic}.
We will use $S$, $A$, $B$ to denote $|\Scal|$, $|\Acal|$, and $|\Bcal|$ respectively.\footnote{Although we assume $S$ and $B$ to be finite, as we will see, our sample complexity results only scale with $\ln S$ and $\ln B$.} 


\begin{definition}
A (MDP, expert policy) pair $(\MDP,\pie)$
is said to be $\mu$-recoverable with respect to loss $\ell$, if 
for all $s \in \Scal$ and $a \in \Acal$, 
$
Q_{\pie}(s,a) - V_{\pie}(s) \leq \mu \cdot \ell( a, \pie(s) ). 
$
\end{definition}

Two notable examples are: 
\begin{itemize}[itemsep=0.3em]
\item $\ell(a, a') = I(a \neq a')$ is the 0-1 loss. Then $\mu$-recoverability implies that for all $a \neq \pie(s)$, $Q_{\pie}(s,a) - V_{\pie}(s) \leq \mu$, which is suitable for discrete-action settings~\cite{ross2011reduction}.

\item $\ell(a, a') = \| a - a' \|$ is the absolute loss. Then a sufficient condition for $\mu$-recoverability is that $Q_{\pie}(s,a)$ is $\mu$-Lipschitz in $a$ with respect to $\| \cdot \|$. 
This is suitable for continuous-action settings~\cite{pan2020imitation}. 
\end{itemize}

In prior works~\cite{ross2010efficient,rajaraman2021value}, it has been demonstrated that for cases where $(M, \pie)$ is $\mu$-recoverable with $\mu \ll H$, i.e., the expert can recover from mistakes with little extra cost,
interactive imitation learning can achieve a much lower sample complexity than offline imitation learning.  



\textbf{Reduction from Interactive Imitation Learning to Online Learning.} \citet{ross2011reduction} proposes a general reduction from interactive imitation learning to no-regret online learning, which we will frequently refer to as the \emph{online imitation learning framework} throughout the paper.
As our algorithm design and performance guarantees will be under this framework, we briefly recap it below. Its key insight is to simulate an $N$-round online learning game between the learner and the environment: at round $n$, the learner chooses a policy $\pi_n$, and the environment responds with loss function $\iterloss_n$. 
The learner then incurs a loss of $\iterloss_n(\pi_n)$,
and observes as feedback a sample-based approximation of $\iterloss_n(\cdot)$.
Here, 
$\iterloss_n (\pi) := \EE_{s \sim d_{\pi_n}, a \sim \pi(\cdot \mid s)} \ell(a, \pie(s))$ is carefully chosen as the expected loss of policy $\pi$ with respect to the expert policy $\pie$ under the state visitation distribution of $\pi_n$. 
$\iterloss_n(\cdot)$ are naturally approximated by
the average loss on supervised learning examples $(s, \pi^E(s))$, whose feature and label parts are sampled from $d_{\pi_n}$ and queried from expert $\pi^E$, respectively. 
\citet{ross2011reduction} shows that, if $\cbr{\pi_n}_{n=1}^N$ has a low online regret,
a policy returned uniformly at random from $\cbr{\pi_n}_{n=1}^N$ has an expected cost competitive with the expert. Formally: 
\begin{theorem}[\citet{ross2011reduction}]
\label{thm:reduction}
Suppose $(\MDP,\pie)$
is $\mu$-recoverable with respect to $\ell$. Define the regret of the sequence of policies $\cbr{\pi_n}_{n=1}^N$ w.r.t. policy class $\Bcal$ as:
\begin{equation}
\Reg(N)
:= 
\sum_{n=1}^N \iterloss_n(\pi_n) - \min_{\pi \in \Bcal} \sum_{n=1}^N \iterloss_n(\pi)
.
\label{eqn:regret-pi-n}
\end{equation}
Then $\hat{\pi}$, which is by choosing a policy uniformly at random from $\cbr{\pi_n}_{n=1}^N$ and adhering to it satisfies:
\[
J(\hat{\pi}) - J(\pie)
\leq 
\mu H \del{ 
\min_{\pi \in \Bcal} \frac{1}{N} \sum_{n=1}^N \iterloss_n(\pi) + \frac{\Reg(N)}{N} }.
\]
\end{theorem}



We will denote $\estgap := \mu H\frac{\Reg(N)}{N}$, which can be viewed as the ``estimation gap'' that bounds the performance gap between $\hat{\pi}$ and the best policy in hindsight. 
Meanwhile, $\min_{\pi \in \Bcal} \frac{1}{N} \sum_{n=1}^N \iterloss_n(\pi)$ measures the expressiveness of policy class $\Bcal$ with respect to the expert policy $\pie$: it is smaller with a larger $\Bcal$. 
In the special case that $\pie \in \Bcal$ (the realizable setting) 
and $\pie$ is deterministic, $\min_{\pi \in \Bcal} \frac{1}{N} \sum_{n=1}^N \iterloss_n(\pi) = 0$. 
For completeness, we provide Theorem~\ref{thm:reduction}'s proof in Appendix~\ref{sec: sec3_proof} and incorporate discussions on limitations of the reduction-based framework.

\section{Oracle-efficient Imitation Learning: Algorithm and Analysis for General Policy Classes}
\label{sec:mftplp}
In this section, we present an interactive imitation learning algorithm for  discrete action space settings with provable computational and statistical efficiency guarantees for general policy classes.
It is based on the aforementioned online IL framework and 
aims to guarantee sublinear regret under the 0-1 loss ($\ell(a,a') = I(a \neq a')$) in the general agnostic setting.

A line of works design practical interactive learning algorithms by assuming access to offline learning oracles~\cite{beygelzimer2010agnostic,dann2018oracle,simchi2022bypassing,dudik2020oracle,syrgkanis2016efficient,rakhlin2016bistro}. Following this, in designing computationally efficient IL algorithms, 
we also assume that our learner has access to an offline learning oracle that can output a policy that minimizes 0-1 classification loss given a input dataset of (state, action) pairs.


\begin{assumption}[Offline learning oracle]
\label{assum:oracle}
There is an offline classification oracle $\Ocal$ for policy class $\Bcal$; specifically, i.e. given any input multiset of classification examples $D = \cbr{(s, a)}$, $\Ocal$ returns 
$
\Ocal(D) = \argmin_{h \in \Bcal} \sum_{(s,a) \in D} I(h(s) \neq a).
$
\end{assumption}

We argue that this is a reasonable assumption -- without a offline learning oracle, one may not even efficiently solve supervised learning, a special case of imitation learning.





Note that without further assumptions, oracle efficient online algorithms are impossible~\cite{hazan2016computational}. 
To allow the design of efficient online algorithms, we also make an assumption that the learner has sampling access to an explorative state distribution that can ``cover'' the state visitation distribution of any policy in the policy class to learn from: 




\begin{assumption}[Sampling access to covering distribution]  
\label{assumption:smooth-base}
The learner has sampling access to a covering distribution $\based \in \Delta(\Scal)$, 
such that there exists  $\sigma \leq 1$, for any $\pi \in \Pi_\Bcal$, 
$d_\pi$ is $\sigma$-smooth with respect to $\based$; formally, 
$\norm{ \frac{d_\pi}{\based} }_\infty \leq \frac{1}{\smooth}$. 
\end{assumption}

Assumption~\ref{assumption:smooth-base} is closely related to the smoothed online learning problem 
setup~\cite{haghtalab2022oracle,block2022smoothed}: under this assumption, in the $N$-round online learning game induced in the online imitation learning framework, $d_{\pi_n}$, the distribution that induces the loss function $F_n$ at round $n$, is $\sigma$-smooth with respect to covering distribution $\based$ for all $n \in [N]$. 
The larger $\sigma$ is, the less variability $d_{\pi_n}$'s can have, indicating that the underlying online learning problem may be more stationary and 
easier to learn. 
The special case $d_0 = d_{\pi^E}$ has also 
been studied in~\citet{spencer2021feedback}, although they do not provide a finite-sample analysis.









\textbf{Challenges in applying existing approaches.} Based on the connection between imitation learning and no-regret online learning mentioned in Section~\ref{sec:prelims}, it may be tempting to directly apply existing oracle-efficient smoothed online learning algorithms \cite{haghtalab2022oracle,block2022smoothed} and establish regret guarantees. However, several fundamental challenges still remain. 
First and most fundamentally,  existing smoothed online learning formulations assume that the sampling distribution at round $n$ is chosen \emph{before} the learner commits to its decision $\pi_n$~\cite{haghtalab2022smoothed,haghtalab2022oracle,block2022smoothed}. 
Unfortunately, this assumption does not hold in the online imitation laerning framework -- specifically, $d_{\pi_n}$ can depend on $\pi_n$. Second, at each round of online imitation learning, 
the learner may collect and learn from a batch of examples, while~\cite{haghtalab2022oracle,block2022smoothed} only addresses the setting when the batch size is 1. 
Lastly, we consider general action set size $A$, meaning 
the learner needs to perform online multiclass classification, while~\cite{haghtalab2022oracle,block2022smoothed} only address binary classification and regression settings.





We address these challenges by proposing the \emph{Mixed Follow the Perturbed Leader with Poission perturbations} algorithm (\algms, Algorithm~\ref{alg:mftpl}). 
Specifically, we address the first challenge by making the following key observation: even though in the online IL framework, the sampling distribution at round $n$ 
can now directly depend on $\pi_n$, as long as the sequence of policies $\cbr{\pi_n}$ has a \emph{deterministic regret guarantee} 
in the original smoothed online learning problem, the same regret guarantee will carry over to the new online imitation learning problem. Such 
deterministic regret guarantee property, to the best of our knowledge, is not known to hold for randomized online learning algorithms such as Follow the Perturbed Leader (FTPL)~\cite{kalai2005efficient}, but 
holds for deterministic online learning algorithms such as an 
in-expectation version of FTPL~\cite{abernethy2014online} or
Follow the Regularized Leader (FTRL)~\cite{shalev2011online}. 


Using this observation, \algms aims to approximate an in-expectation of FTPL to guarantee a sublinear regret. 
It follows the online IL framework: 
at round $n$, in the data collection step (line 8)
, \algms rolls out the currently learned policy $\pi_n$ in the MDP multiple times to 
sample $K$ states from $d_{\pi_n}$. It then requests expert's demonstrations on them, obtaining a dataset $\incrdata_n$ of (state, action) pairs. 
In the policy update step (line 4
to
line 7),
\algms calls the \trainbase function $E$ times on the accumulated dataset $D$, to train a new policy $\pi_{n+1}$, which is an average of $E$ base policies $\cbr{ \pi_{n+1, e} }_{e=1}^E$.
Each $\pi_{n+1, e}$ can be seen as a freshly-at-random output of the FTPL algorithm with \emph{Poisson sample-based perturbations}~\cite{haghtalab2022oracle}: first drawing a Poisson random variable $X$ representing perturbation sample size, then drawing $Q$, a set of $X$ iid examples from covering distribution $\based$ with uniform-at-random labels  from $\Acal$; finally calling the offline oracle $\Ocal$ on the perturbed dataset $D \cup Q$. 
It can now be seen that $\pi_{n+1}$ approximates the output of an in-expectation version of FTPL: a larger $E$ yields a better approximation, which ensures high-probability regret guarantees. 
Finally, \algms returns a policy $\hat{\pi}$ uniformly at random from the historical policies $\cbr{\pi_n}$.

Algorithmically, \algms can be viewed as maintaining an ensemble of $E$ policies in an online fashion and use it to perform strategic collection of expert demonstration data. For this reason, we will refer to $E$ as the ensemble size parameter. 
Similar algorithmic approaches have been proposed in imitation learning with expert intervention feedback~\cite{menda2019ensembledagger}; however, as discussed in Section~\ref{sec:related_works}, these works focus on ensuring safety in training and do not provide finite-sample analysis. 









\begin{algorithm}[t]
\caption{\algms}
\begin{algorithmic}[1]
\STATE \textbf{Input:} MDP $\MDP$, expert $\pie$, policy class $\Bcal$, oracle $\Ocal$, covering distribution $\based$,
sample size per round $K$,  ensemble size $\Ensem$, perturbation budget $\poiparam$.
\STATE Initialize $\incrdata = \emptyset$. 
\FOR{$n=1,2,\ldots,N$}
\FOR{$\ensem=1,2,\ldots,\Ensem$}
\label{hex:train-ensemble-start}
\STATE $\pi_{n, e} \gets \trainbase(D, \based, \poiparam)$
\ENDFOR
\STATE Set $\pi_{n}(a \mid s) := \frac{1}{\Ensem}\sum_{\ensem = 1}^\Ensem I( \pi_{n, e}(s) = a )$.
\label{hex:train-ensemble-end}
\STATE $\incrdata_n = \cbr{ (s_{n,k}, \pie(s_{n,k})) }_{k = 1}^K \gets$ sample $K$ states i.i.d. from $d_{\pi_{n}}$ by rolling out $\pi_{n}$ in $\MDP$, and query expert $\pie$ on these states. 
\label{hex:collect-new-data}
\STATE Aggregate datasets  $\incrdata \gets \incrdata \cup \incrdata_n$.
\ENDFOR
\STATE \textbf{Return} $\hat{\pi} \gets \aggpolicy( \cbr{ \pi_{n,e}}_{n=1, e=1}^{N+1, E} )$

\STATE \textbf{function} \trainbase($D, \based, \poiparam$):
\STATE \hspace{1em} Sample $\poi \sim \poisson(\poiparam)$
\STATE \hspace{1em} Sample $\pertdata \gets $ draw i.i.d. perturbation samples $\{(\tilde{s}_{n,x},\tilde{a}_{n,x})\}_{x=1}^{\poi}$ from $\incrdistbase = \based \otimes \Unif(\Acal)$.

\STATE \textbf{Return} $h \gets \Ocal( D \cup Q)$.


\STATE \textbf{function} \aggpolicy$(\cbr{ \pi_{n,e}}_{n=1, e=1}^{N+1, E})$:

\STATE \hspace{1em} Sample $\hat{n} \sim \Unif([N])$


\STATE \hspace{1em} \textbf{Return} $\pi_{\hat{n}}(a \mid s) := \frac{1}{\Ensem}\sum_{\ensem = 1}^\Ensem I( \pi_{\hat{n}, e}(s) = a )$.

\end{algorithmic}
\label{alg:mftpl}
\end{algorithm}








%




We show the following theorem regarding the regret guarantee of \algms; we defer its full version (Theorem~\ref{thm: mftpl_poisson_main}), along with proofs to  Appendix~\ref{sec:mftplp-pf}. 


\begin{theorem}
\label{thm:soil-main}
For any $\delta \in (0,1]$, for large enough $N$, 
\algms with appropriate setting of its parameters $\Ensem, \lambda$
outputs $\{\pi_{n}\}_{n=1}^N$
that satisfies that with probability at least $1-\delta$, $\Reg(N)$ is at most 
\begin{align*}
\label{eqn:regret-after-tuning-lambda-main}
\tilde{O} \Bigg( & 
\sqrt{N} \rbr{\frac{A (\ln B)^2}{\sigma K^2}}^{\frac14}
+
\sqrt{N} \rbr{\frac{A \ln B}{\sigma}}^{\frac14} 
+
\sqrt{N \ln\frac1\delta } \Bigg),
\end{align*}
where we recall that $B$ denotes the size of the base policy class $\Bcal$.
\end{theorem}

Specialized to $A=2$ and $K=1$, this result is consistent with~\citet{haghtalab2022oracle} in the regret bound is dominated by  
$\sqrt{N} \rbr{\frac{(\ln B)^2}{\sigma}}^{\frac14}$; we remark again though, that our regret analysis needs to get around the three additional challenges mentioned above.




In view of Theorem~\ref{thm:reduction}, Theorem~\ref{thm:soil-main} translates to the following result on the sample complexity of expert demonstrations and the number of calls to the classification oracle, for  $\estgap = \frac{\mu H\Reg(N)}{N}$ to be at most $\epsilon$:



\begin{corollary}
\label{cor:soil-main}
For any small enough $\epsilon > 0$, \algms, by setting its parameters as in Theorem~\ref{thm:soil-main} and number of rounds 
$N = \tilde{O}\rbr{\frac{\mu^2H^2\sqrt{A\ln(B)}}{\epsilon^2 \sqrt{\sigma}}}$ and batch size $K = \sqrt{\ln B}$,
guarantees that $\frac{\mu H\Reg(N)}{N} \leq \epsilon$ with high probability,  using  $\tilde{O}\rbr{\frac{\mu^2H^2\sqrt{A}\ln B}{\epsilon^2 \sqrt{\sigma}}}$ expert demonstrations, and $\tilde{O}\del{ \frac{ \mu^4 H^4 A^2 (\ln B)^2}{ \epsilon^4 \sigma} }$
calls to $\Ocal$.
\end{corollary}

In practice, the batch size $K$ may be considered as part of problem specification and chosen ahead of time; motivated by this, 
we provide a version of this corollary with general $K$, Corollary~\ref{cor:soil-main-full}, in Appendix~\ref{sec:mftplp-pf}. 



 
Table~\ref{table:comparison_main} compares \algms with~\cite{li2022efficient} and \algbc in terms of the number of expert demonstrations for $\estgap \leq \epsilon$, with a focus on comparing their dependences on $\mu$, $H$ and $\epsilon$. 
Both \algms and \cite{li2022efficient} have a coefficient of $\mu^2H^2$, much smaller than $H^4$ for \algbc,
while~\cite{li2022efficient} requires that the policy class $\Bcal$ has a small separator set~\cite{syrgkanis2016efficient,dudik2020oracle,li2022efficient}, which is only known to hold for a few $\Bcal$'s.

\begin{table}[h]
\centering
\caption{Number of expert demonstrations $C(\epsilon)$ for $\estgap \leq \epsilon$. }
\vskip 0.05in
\label{table:comparison_main}
\begin{tabular}{ccc}
\toprule
Algorithms   & $C(\epsilon)$ & Remarks \\
\midrule 
\algms (this work)  &  $\frac{\mu^2 H^2}{\epsilon^2} $ & General $\Bcal$, Access $\based$  
\\
~\citet{li2022efficient} 
& $\frac{\mu^2 H^2}{\epsilon^2} $ 
& $\Bcal$ small separator set 
\\
\algbc  & $\frac{H^4}{\epsilon^2}$ & General $\Bcal$\\
\bottomrule
\end{tabular}
\end{table}


\section{Experiments}
\label{sec:experiment}
In this section, we evaluate \algms and its variant, comparing them with online and offline IL baselines in 4 continuous control tasks from OpenAI Gym \cite{brockman2016openai}.
Our experiments are designed to answer the following questions:
\textbf{Q1:} Does sample-based perturbation provide any benefit in \algms?
\textbf{Q2:} How does the choice of covering distribution $\based$ affect the performance of \algms?
\textbf{Q3:} Does \algms outperform online and offline IL baselines?
\textbf{Q4:} Can we find a practical variant of \algms that achieves similar performance to \algms without additional sample access to some  covering distribution?
\textbf{Q5:} If Q3 and Q4 are true, which component of our algorithms confers this advantage?

Our experiment sections are organized as follows: Section~\ref{sec:experiment_settings} provides an introduction to our experimental settings. Section~\ref{sec: perturbation_utility} presents positive results for \textbf{Q1} and \textbf{Q2}, evaluating \algms on two continuous control tasks using a linear policy class $\Bcal$ with nonrealizable experts. Subsequently, Section~\ref{sec:main_experiment} affirmatively answers \textbf{Q3} and \textbf{Q4} and introduces \alg (abbreviated as \bd), a practical variant of \algms, and demonstrates the efficacy of our algorithms through neural network-based experiments. Across 4 continuous control tasks that encompass  realizable and nonrealizable settings, \bd and \algms outperform both online and offline IL baselines. Finally, for \textbf{Q5}, Section~\ref{sec:benefit_of_ensemble} investigates the underlying reasons for \bd's improvement over the \dagge baseline.

\subsection{Experiment Settings}
\label{sec:experiment_settings}
\textbf{Environment:} Following~\citet{brantley2019disagreement}, we use normalized states. The name and $\{$state dimension, action dimension$\}$ for each continuous control task are: Ant $\cbr{28,8}$, Half-Cheetah $\cbr{18,6}$, Hopper $\cbr{11,3}$, and Walker2D $\cbr{18,6}$. 



\textbf{Expert:} 
For each task, the expert policy $\pie$ is a multilayer perceptron (MLP) with 2 hidden layers of size 64 and corresponding state, action dimension, pretrained by TRPO~\cite{schulman2015trust}. 
We employ TRPO's stochastic policy, sampling expert actions from MLP output with Gaussian noise.

\textbf{Offline Learning Oracle:} 
Our \algms relies on offline learning oracle $\Ocal$; we describe their implementations below. 
In continuous control tasks, the training loss $\tilde{\ell}(a,a')$ is calculated by clipping input actions to the range $\sbr{-1,1}$ and computing the MSE loss~\cite{brantley2019disagreement}.
In Section~\ref{sec: perturbation_utility}, we first work on linear models and implement a deterministic offline learning oracle by outputting the Ordinary Least Squares (OLS) solution using the Moore-Penrose pseudoinverse~\cite{moore1920reciprocal}.  
Later, in Section~\ref{sec:main_experiment} and~\ref{sec:benefit_of_ensemble}, we use MLP for base policies and implement $\Ocal$ by conducting 2000 SGD iterations over its input dataset with batch size 200. See Appendix~\ref{sec: experiment_details} for results of 500 and 10000 iterations.




\textbf{Sampling Oracle:} 
We define the covering distribution $\based$ as the uniform distribution over states obtained from 10 independent runs collected by \dagge.
Note that this gives $\algms$ some unfair advantage over the baselines; we will subsequently propose practical variants of our algorithms that do not require knowledge of $d_0$.
Additionally, in Section~\ref{sec: perturbation_utility}, we consider an alternative $\based$, defined as the uniform distribution over state space.

\textbf{Algorithms:} Due to the sample-efficient nature of IL, we make the tasks more challenging by setting the sample size per round $K = 50$ for all algorithms~\cite{ho2016generative, laskey2017dart}. All policies in the first round $\pi_1$ are initialized at $0$ for linear policy and at random for MLPs.
We choose \dagge and \algbc(\bc) as online and offline IL baselines.
At round $n$, \bc receives $K$ additional (state, action) pairs sampled from expert's trajectories and calls the offline learning oracle on the accumulated dataset. 
In contrast, all other algorithms sample $K$ states from their current policy $\pi_n$'s trajectories and query the expert's action on them, while following dataset aggregation and calling the offline learning oracle to compute policies $\pi_{n+1}$ for the next round.
As a practical implementation of \algms, we choose ensemble size $\Ensem = 25$; in addition,
to facilitate parallel training of the ensembles~\cite{brantley2019disagreement},
instead of drawing sample sizes $X$ from a Poisson distribution, we choose $X$ as fixed numbers\footnote{It is well-known that Poission distribution has good concentration properties~\citep[e.g.][]{canonne2017short}.
so we do not expect this to deviate too much from the original algorithm.} -- we abbreviate this algorithm as \MPX.

\textbf{Evaluation:}
We run each algorithm 10 times with different seeds, treating each round $n$ as the final one and only returning the last trained policy $\pi_{n+1}$ for evaluation~\cite{cheng2019predictor, cheng2019accelerating}.
As in common practice \cite{menda2019ensembledagger,hoque2021thriftydagger,menda2017dropoutdagger}, we return the ensemble mean $\bar{\pi}_n(s) := \frac{1}{E} \sum_{e=1}^E \pi_{n+1,e}(s)$,
which is also known as Bagging~\cite{breiman1996bagging}.
Given a returned policy $\pi$, we roll out $\nrepeat=25$ trajectories (denote by $\cbr{\tau^\pi_i}_{i=1}^T$) and compute their average reward as an estimate of $\pi$'s expected reward.


\subsection{Utility of Sample-based Perturbation}
\label{sec: perturbation_utility}
We use linear policy classes along with OLS offline learning oracle for our first experiment. We study the impact of perturbation size $X$ and the choice of $\based$ on the performance of \MPX. Here, we choose \dagge as the baseline; note that this is equivalent to $\MP(0)$ given that the offline learning oracle returns OLS solutions deterministically.  
We consider two settings of $d_0$ in Section~\ref{sec:experiment_settings}.

\begin{figure}[t!]
  \centering
  \includegraphics[width=1.01\linewidth]{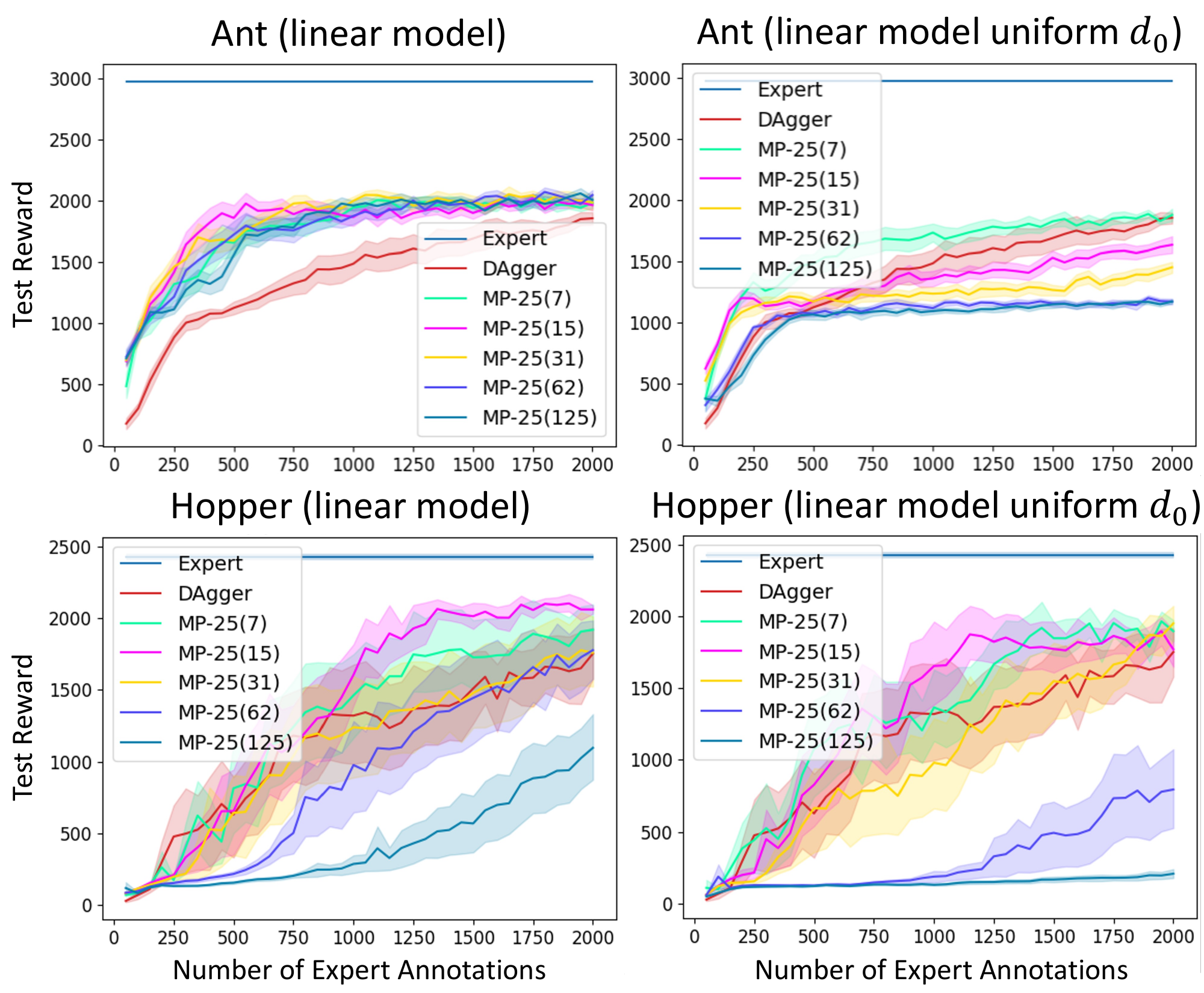}
  \vspace{.03in}
  \caption{Comparative performance of \algms using linear models with nonrealizable MLP experts: variation across different perturbation state sources and set sizes in Ant and Hopper. Shaded region represents range between $10^{\text{th}}$ and $90^{\text{th}}$ quantiles of bootstrap confidence interval~\cite{diciccio1996bootstrap}, computed over 10 runs. On the left, the perturbation example sources are states collected by \dagge on each task, while the right side uses uniform distribution over $[-2,2]^{28}$ (Ant) and $[-2,2]^{11}$ (Hopper). Overall, \MP$\text{s}$ on the left exceed their counterparts on the right. Meanwhile, \MP(15) leads in performance, except in the Ant with uniform $\based$ (upper right).}
  \label{fig:main_linear}
\end{figure}

The average reward of the trained policies as a function of the number of expert annotations for Ant and Hopper are shown in Figure~\ref{fig:main_linear} with $80\%$ bootstrap confidence bands~\cite{diciccio1996bootstrap}.
Surprisingly, though the expert (an MLP policy) is not contained in the policy class, \algms still learns policies with nontrivial performance.  
The overall performance of \MPX initially increases with the perturbation size and then decreases, matching our intuition. For \textbf{Q1}, since \MP{}(7) and \MPF outperform \dagge (\MP{}(0)) in most cases, we have strong evidence that sample-based perturbation benefits performance with proper choices of perturbation sample size. For \textbf{Q2}, by comparing the performance of the same \MPX on the left and right, it is evident that using states collected by \dagge for perturbation results in better performance than uniform samples over state spaces. 
Based on our observations, we focus on evaluating $\MPF$ for the following sections.  
Please see Appendix~\ref{sec:full_perturbation_utility} for performance of other algorithms under this setting.



\subsection{Performance Evaluation of \algms and Its Practical Variant \alg}

\label{sec:main_experiment}
\begin{figure*}[t]
  \centering
  \includegraphics[width=1\linewidth]{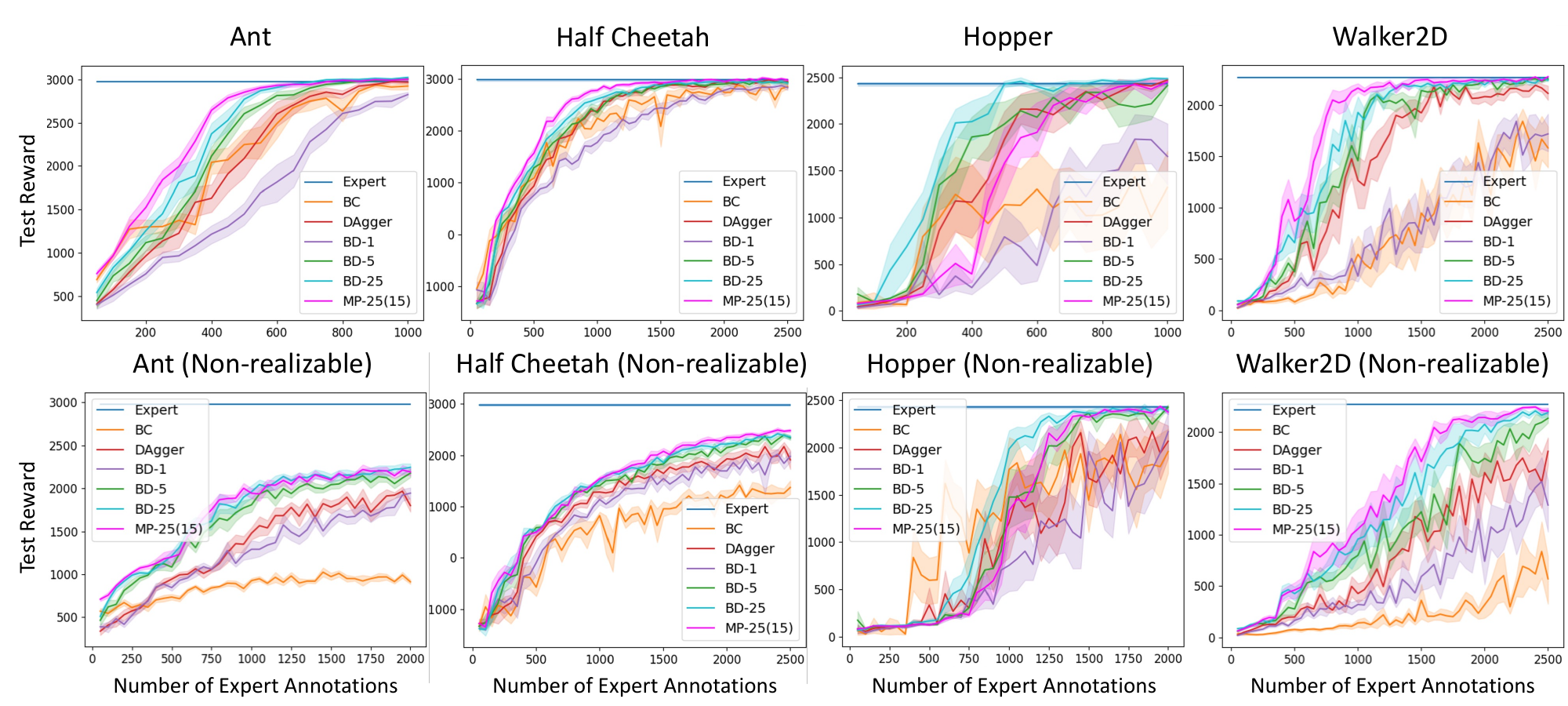}
  \caption{Results on continuous control tasks with realizable and non-realizable experts. Remarkably, \MPF (magenta), \bdc(blue-green) and \bdb (green) surpass baselines under both settings, with distinct performance gaps particularly evident in the non-realizable setting between \MPF, \bdc, \bdb, and the baselines.} 
  \label{fig:main_experiment_plots}
\end{figure*}

Though \algms is provably efficient, it requires additional sample access to $\based$ and proper choice of the perturbation sample size. 
We propose \alg (abbrev. \bd),
a variant of \algms without sample access to $d_0$, and evaluate our algorithms in the 4 continuous control tasks in Section~\ref{sec:experiment_settings}. 
\bd shares the same data collection scheme as \algms and only differs on the \trainbase function.

As can be seen in the \bd's \trainbase function in Algorithm~\ref{alg:bd} (see Appendix~\ref{sec: full_bd_alg} for the full \bd algorithm), \bd trains on different bootstrap subsamples of the accumulated dataset $D$ to obtain a diverse ensemble of policies, 
instead of training on the accumulated dataset union with different sample-based perturbations.
\bd is fundamentally different from~\cite{menda2019ensembledagger}, 
where the diversity of ensembles are attributed solely to the stochasticity of SGD. 
In the following, we study the performance of \bd with increasing size of ensembles, choosing $\Ensem=1,5,25$ (abbreviated as \bda,\bdb,\bdc).


\begin{algorithm}[t!]
\caption{\alg}
\begin{algorithmic}[1]

\STATE \textbf{function} \trainbase($\incrdata$):

\STATE \hspace{1em} $\tilde{\incrdata}$ $\gets$ Sample $|\incrdata|$ i.i.d. samples $\sim \Unif(\incrdata)$ with replacement.

\STATE \hspace{1em} \textbf{Return} $h \gets \Ocal(\tilde{\incrdata})$.
\end{algorithmic}
\label{alg:bd}
\end{algorithm}

We perform evaluations in realizable and non-realizable settings using MLPs as base policy classes. 
In the realizable setting, the base policy class contains the conditional mean function of the expert policy.
Meanwhile, the non-realizable setting considers the base policy class to be MLPs with one hidden layer and limited numbers of nodes (see Appendix~\ref{sec: experiment_details} and ~\ref{sec:full_main_experiment} for details).
As shown in Figure~\ref{fig:main_experiment_plots}, \MPF consistently outperforms others in  most cases.
Overall, \bd shows a notable improvement in performance as the ensemble size grows, with \bdc achieving performance on par with \MPF.
Perhaps unsurprisingly, the naive \bda falls short of matching \dagge's performance. This is attributed to the inherent limitations of bootstrapping, which omits a significant portion of the original sample. However, it is important to highlight the consistent and significant improvements from \bda to \bdb across 4 tasks, as they demonstrate the effectiveness of using model ensembles to mitigate the sample underutilization from bootstrapping.
Notice that the increase in performance from \bdb to \bdc is marginal, with \bdb outperforming the baselines in all cases except in the realizable Hopper, where \dagge achieves a similar level of performance.
Interestingly, as shown in the lower part of Figure~\ref{fig:main_experiment_plots},
\MPF, \bdc and \bdb not only learn faster than the baselines, but also converge to policies with higher performance.





For running time and space requirements, under realizable settings, all algorithms consume similar memory (1400 MB) on GPU, while \bdc and \MPF run 5 times longer than \bdb, \bc, and \dagge (see Appendix~\ref{sec: experiment_details} for details). Notably, \bdb maintains strong performance without imposing significant computational overhead, taking just twice the running time of \dagge. Therefore, we recommend using \bdb for practical applications.

\subsection{Explaining the benefit of \alg}
\label{sec:benefit_of_ensemble}
Though \bdb outperforms \dagge, the underlying reason of this improvement demands further investigation.
We hypothesize two possible factors behind \bdb's success: (1) \bdb collects data of higher quality during the training stage; (2) Given the same expert demonstration dataset, 
\bdb returns a better policy via ensemble averaging, similar to the benefit of Bagging in supervised learning~\cite{breiman1996bagging}.


\begin{figure}[ht!]
  \centering
  \includegraphics[width=1.01\linewidth]{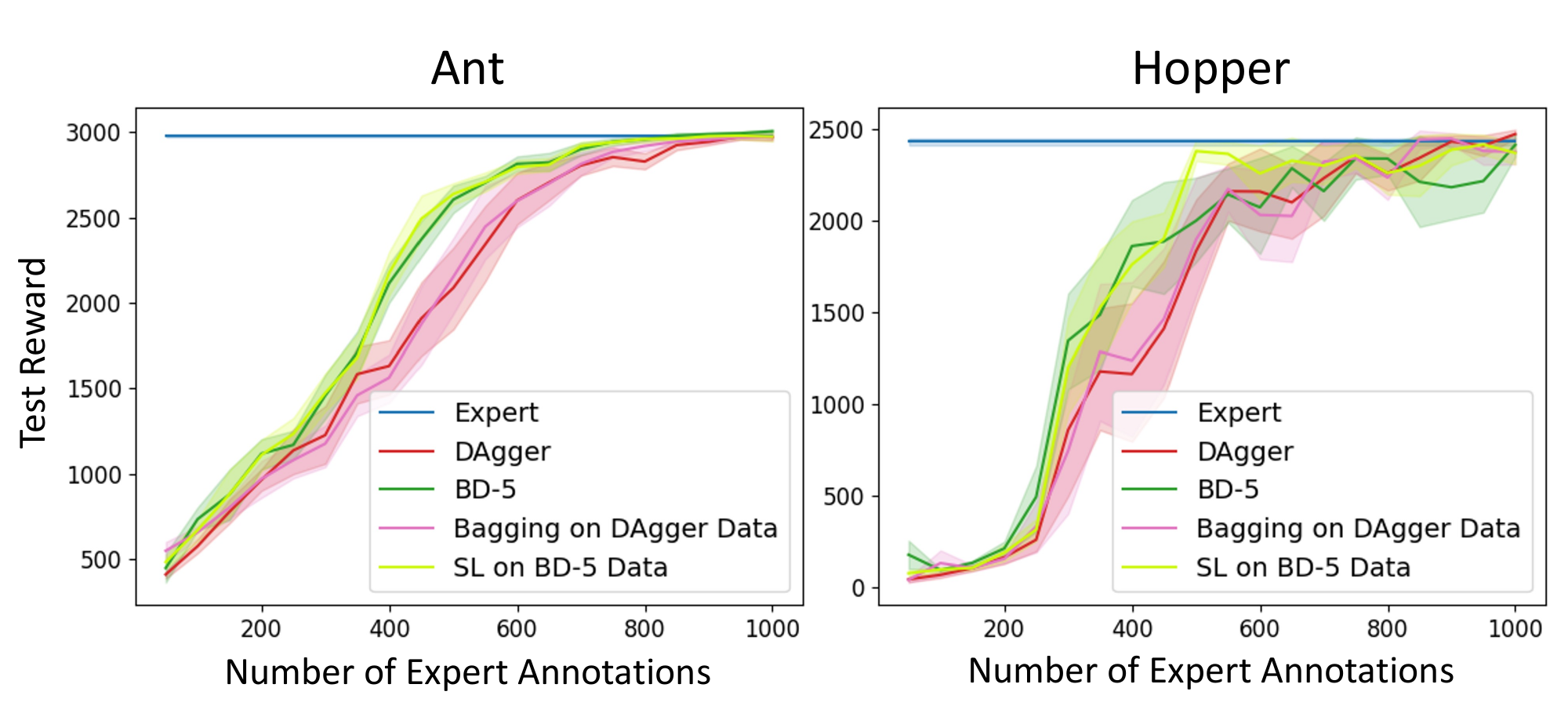}
  \vspace{.05in}
  \caption{Results on comparing \bdb and \dagge, along with the two additional approaches in Section~\ref{sec:benefit_of_ensemble},  over Ant and Hopper. Bagging on data collected by \dagge yields pink learning curves that align closely with \dagge's performance (red). Meanwhile, naive supervised learning on data collected by \bdb produce lime green learning curves that match the performance of \bdb (green).
  Overall the two methods (red and pink) that uses ensembles to perform data collection has better performance than those two that does not (green and lime green). 
  This suggests that \bdb improves over \dagge by collecting better data.} 
  \label{fig:ablation_study_main}
  \vspace*{-0.1cm}
\end{figure}

To test these, 
we evaluate two additional approaches: (a) naive supervised learning (abbreviated as \nsl) on data collected by \bdb; (b) Bagging (bootstrap and return 5-ensemble average) on data collected by DAgger. 
As shown in Figure~\ref{fig:ablation_study_main}, switching the final policy training between Bagging and naive supervised learning does not change the policy performance significantly. In contrast, using ensemble for data collection significantly increases the trained policy's performance. This verifies hypothesis (1) and invalidates hypothesis (2). We further visualize states queried by different algorithms in Appendix~\ref{sec: tsne}, which implies more efficient exploration by ensembles.

\section{Conclusion}
We propose and evaluate \algms, a computationally and statistically efficient IL algorithm for general policy classes. We also propose a practical variant, \alg that we recommend for practical applications.

Our work is built on the online imitation learning reduction framework~\cite{ross2011reduction,ross2014reinforcement}. As we discuss in Appendix~\ref{sec: sec3_proof}, in the agnostic setting, this framework has the drawback of only  providing runtime-dependent guarantees, as well as not ensuring global optimality. We leave overcoming these drawbacks as important open problems. 




\section*{Impact Statement}
This paper presents work whose goal is to advance the field of Imitation Learning. 
To the best of our knowledge, we are not aware of negative social impacts by our work.

\section*{Acknowledgements}
We thank the anonymous reviewers for their constructive feedback, and thank Yao Qin for helpful discussions on visualizing the observation distributions encountered in imitation learning data collection.
This research is partially supported by the University of Arizona FY23 Eighteenth Mile TRIF Funding.

\bibliographystyle{plainnat}

\bibliography{il}


%
%





%

%

\setcounter{section}{0} 
\renewcommand{\thesection}{\Alph{section}} 

\onecolumn



\icmltitle{Agnostic Interactive Imitation Learning: New Theory and Practical Algorithms \\
Supplementary Materials}

\section{The Online Imitation Learning Reduction Framework}
\label{sec: sec3_proof}
\begin{theorem}[Restatement of Theorem~\ref{thm:reduction}, originally from~\cite{ross2011reduction}, Theorem 3.2]
Suppose $(\MDP,\pie)$
is $\mu$-recoverable with respect to $\ell$. In addition, a sequence of policies $\cbr{\pi_n}_{n=1}^N$ satisfies the following online regret guarantee with respect to base policy class $\Bcal$:
\[
\sum_{n=1}^N \iterloss_n(\pi_n) - \min_{\pi \in \Bcal} \sum_{n=1}^N \iterloss_n(\pi)
\leq 
\Reg(N).
\]
Then $\hat{\pi}$, which is by choosing a policy uniformly at random from $\cbr{\pi_n}_{n=1}^N$ and adhering to it satisfies:
\[
J(\hat{\pi}) - J(\pie)
\leq 
\mu H \del{ 
\min_{\pi \in \Bcal} \frac{1}{N} \sum_{n=1}^N \iterloss_n(\pi) + \frac{\Reg(N)}{N} }.
\]
\label{thm:reduction_full}
\end{theorem}

\begin{proof}
Our proof is similar to Proposition 2 of~\cite{li2022efficient}. Since $(\MDP,\pie)$ and $\ell$ satisfies for all $s \in \Scal$ and $a \in \Acal$, 
$
Q_{\pie}(s,a) - V_{\pie}(s) \leq \mu \cdot \ell( a, \pie(s) )
$,
We apply the performance difference lemma (Lemma~\ref{lem: performance difference lemma} below) to the sequence of $\{\pi_n\}_{n=1}^N$ and $\pie$, obtaining
\[
\begin{aligned}
\frac1N \sum_{n=1}^N J(\pi_n) - J(\pie) 
=&
\frac H N  \sum_{n=1}^N \EE_{s \sim d_{\pi_n}}\EE_{a \sim \pi_n(\cdot|s)} \sbr{Q_{\pie}(s, a) - V_{\pie}(s)}\\
\leq & 
\frac{\mu H}{N} \sum_{n=1}^N  \EE_{s \sim d_{\pi_n}}\EE_{a \sim \pi_n(\cdot|s)} \sbr{\ell( a, \pie(s) )} \\
= &
\frac{\mu H}{N} \sum_{n=1}^N  \iterloss_n(\pi_n)
\leq
\mu H \del{ 
\min_{\pi \in \Bcal} \frac{1}{N} \sum_{n=1}^N \iterloss_n(\pi) + \frac{\Reg(N)}{N} },
\end{aligned}
\]
where the last line comes from the definition of $\iterloss_n (\pi) := \EE_{s \sim d_{\pi_n}, a \sim \pi(\cdot \mid s)} \ell(a, \pie(s))$ and $\Reg(N)$.

Now, it suffices to show $\frac1N \sum_{n=1}^N J(\pi_n) = J(\hat{\pi})$. Since $\hat{\pi}$ is executed by choosing a policy uniformly at random from $\cbr{\pi_n}_{n=1}^N$ and adhering to it, we conclude the proof by
$$
J(\hat{\pi}) = \EE_{s_1 \sim \rho} \sbr{ V_{\hat{\pi}}(s_1) } 
=
\EE_{s_1 \sim \rho} \sbr{\frac1N \sum_{n=1}^N  V_{\pi_n}(s_1) } 
=
\frac1N \sum_{n=1}^N \EE_{s_1 \sim \rho} \sbr{ V_{\pi_n}(s_1) } 
=
\frac1N \sum_{n=1}^N J(\pi_n).
$$


\end{proof}

\begin{lemma} [Performance Difference Lemma, Lemma 4.3 of~\cite{ross2014reinforcement}]
\label{lem: performance difference lemma}
For two stationary policies $\pi$ and $\pie$ $: \Scal \to \Delta(\Acal)$, we have
\[
J(\pi) - J(\pie) = H \cdot\EE_{s \sim d_{\pi}}\EE_{a \sim \pi(\cdot|s)} \sbr{Q_{\pie}(s, a) - V_{\pie}(s)}.
\]
\end{lemma}

\subsection{Limitations of the reduction-based framework}

Our positive result relies on the reduction framework of~\cite{ross2011reduction},
which bounds the learned policy's suboptimality by the sum of estimation gap and policy class bias $\mu H \min_{\pi \in \Bcal} \frac1N \sum_{n=1}^N F_n(\pi)$ (Theorem~\ref{thm:reduction}).
Importantly, the latter term is runtime dependent and one usually do not have a good control unless additional assumptions are imposed (e.g.,  there exists some $\pi \in \Bcal$ that disagrees with $\pi^E$ with low probability under some covering distribution $d_0$). We believe designing 
agnostic interactive imitation learning algorithms with runtime-independent guarantees is an important problem. 

Moreover, we show in Proposition~\ref{prop:reduction_limit} below that it is possible in the agnostic setting that any no-regret policy sequence 
$\cbr{\pi_n}_{n=1}^N$
with respect to the cost-sensitive classification losses $\cbr{F_n(\cdot)}_{n=1}^N$
converges to a \emph{globally suboptimal} policy with respect to the ground truth expected reward function $J(\cdot)$. 
We leave designing agnostic imitation learning algorithms with global optimality guarantees as an important question; without further assumption on the expert policy $\pie$,  
we believe this problem may be as hard as policy search-based agnostic reinforcement learning~\cite{jia2024agnostic}, where only limited positive results are currently known.

In the following, suppose we study the classification-based imitation learning setting when the loss function $\ell(s,a) = A_{\pie}(s,a):= Q_{\pie}(s, a) - V_{\pie}(s)$; this is the setting initially studied by~\cite{ross2014reinforcement}. 
As a result, 
$F_n(\pi) = \EE_{s \sim d_{\pi_n}} \EE_{a \sim \pi(\cdot|s)} \sbr{Q_{\pie}(s, a) - V_{\pie}(s)}$, while $F_n(\pi_n) = \frac{1}{H}(J(\pi_n)-J(\pie))$ (by Lemma~\ref{lem: performance difference lemma}).

\begin{proposition}
\label{prop:reduction_limit}
There exists a policy class $\Bcal$ of size 2, an MDP $\Mcal$, an expert policy $\pi^E$, such that any policy sequence $\cbr{\pi_n}_{n=1}^N \subseteq \Bcal$ guaranteeing a sublinear regret 
\[
\sum_{n=1}^N F_n(\pi_n) 
-
\min_{\pi \in \Bcal} \sum_{n=1}^N F_n(\pi)
= 
o(N)
\]
satisfies that 
\[
\sum_{n=1}^N J(\pi_n) - \min_{\pi \in \Pi} J(\pi)
= 
\Omega(N)
\]
\end{proposition}

\begin{proof}
As shown in Figure~\ref{fig:lowerbound}, we define MDP $\Mcal$ with:
\begin{itemize}
    \item State space $\Scal = \cbr{S_0,S_1,S_2,S_3,S_4}$ and action space $\Acal = \{L,R\}$.
    \item Initial state distribution  $\rho(S_0) = 1$
    \item Deterministic Transition dynamics:  $P_1(S_1|S_0,L)=1$, $P_1(S_2|S_0,R)=1, P_2(S_3|S_2,L)=1, P_2(S_4|S_0,R)=1$, while $\forall t \in [H]$, $\forall a \in \Acal$, $P_t(S_1|S_1,a)=P_t(S_3|S_3,a)=P_t(S_4|S_4,a)=1$, which are self-absorbing before termination.
    \item Cost function $c(S_0, R) = c(S_2, L) =  c(S_3,\cdot) =0$, $c(S_0, L) = c(S_1, \cdot) = \frac{1}{H}$, $ c(S_2, R) = c(S_4, \cdot)=1$.
\end{itemize}  

Meanwhile, let:
\begin{itemize}
\item Base policy class $\Bcal = \{h_1,h_2\}$, where $h_1(S_0) = L$ and $h_2(S_0)=R$, while $h_1(S_2) = h_2(S_2)=R$.
\item Deterministic expert $\pie$ such that $\pie(S_0) =R$, $\pie(S_2)= L$. 
\end{itemize}

For this MDP example, it can be seen that $J(\pie)=0$, $J(h_1) =1$, $J(h_2)= H-1$.
Also, $V_{\pi}(S_0) = J(\pie)=0$, $Q_{\pie}(S_0, L) = 1$, $Q_{\pie}(S_0, R) = 0$, we have $A_{\pie}(S_0,h_1(S_0)) = 1$, $A_{\pie}(S_0,h_2(S_0)) = 0$.

Consider any sequence of policy $\cbr{\pi_n}_{n=1}^N  \subseteq \Bcal$, inducing loss function $\cbr{F_n(\pi)}_{n=1}^N$.
First, we observe that for every $n$,
$
\argmin_{\pi \in \Bcal} F_n(\pi)
= h_2
$.
This is because 
the only difference between $h_1$ and $h_2$ is the action taken at $S_0$, and so 
for any $\pi_n$, 
$F_n(h_1) - F_n(h_2) = \frac1H ( A_{\pie}(S_0,h_1(S_0)) - A_{\pie}(S_0,h_2(S_0)) ) = \frac1H > 0$.


Therefore we conclude that given any any sequence $\cbr{\pi_n}_{n=1}^N$ that guarantees a sublinear regret, we have that at least $1-o(1)$ fraction of the $\pi_n$'s must be $h_2$.

Then, for large enough $N$,
\[
\sum_{n=1}^N \rbr{J(\pi_n) - \min_{\pi \in \Bcal} J(\pi)}
\geq
\sum_{n=1}^N J(h_2)-o(N) - \sum_{n=1}^N\min_{\pi \in \Bcal} J(\pi)
=
N(H-1)-N-o(N)
=
\Omega(N)
\]

\begin{figure*}[htbp]
  \centering
  \includegraphics[width=0.4\linewidth]{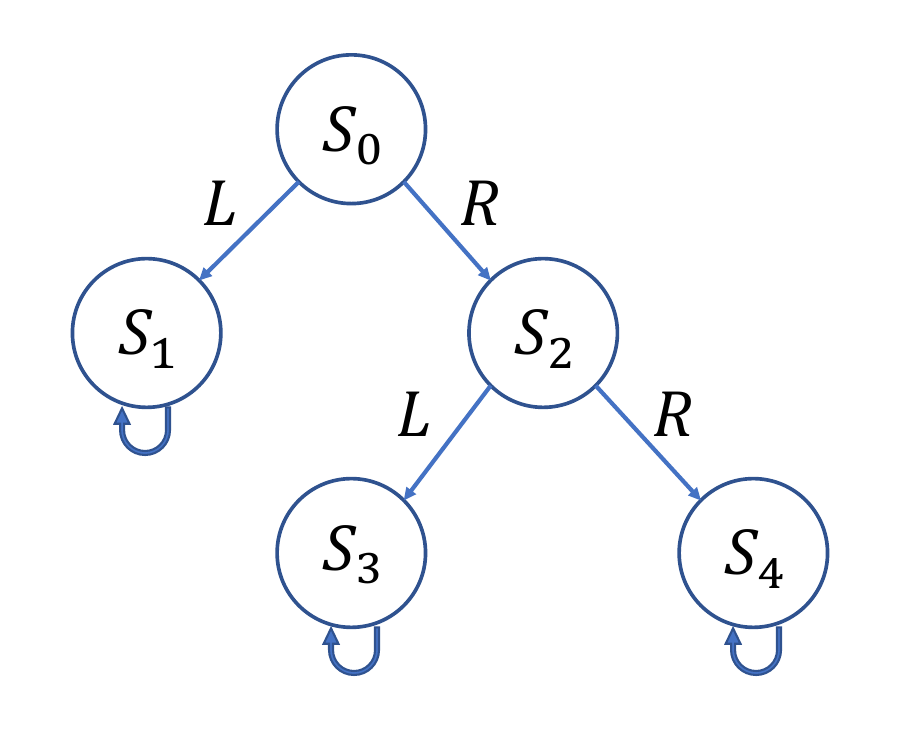}
  \vspace{.1in}
  \caption{Example MDP showing the limitation of reduction-based framework.}
  \label{fig:lowerbound}
\end{figure*}
\end{proof}

\begin{remark}
The proposition continues to hold when the loss function is the 0-1 loss: $\ell(s,a) = I( a \neq \pie(s) )$. To see this, note that $F_n(h_1) - F_n(h_2) = \frac1H \del{ I( h_1(S_0) \neq \pie(S_0) ) - I( h_2(S_0) \neq \pie(S_0) ) } = \frac{1}{H} > 0$, and the rest of the proof carries over verbatimly. 
\end{remark}

\begin{remark}
The above proposition can be generalized to allow $\cbr{\pi_n}_{n=1}^N \subseteq \Pi_\Bcal$; the proof will carry over except that we argue that at least $1-o(1)$ of the total weights in $\pi_n$'s representation will be on $h_2$.  
\end{remark}

\begin{remark}
We thank an anonymous ICML reviewer who originally provided an example~\citep[][Problem (3)]{pitfall} on this issue. To put this in a real-world example, we can view this learning to ski from an expert demonstrator. The expert chooses a fast route $S_0 \to S_2 \to S_3 \to .. \to S_3$. Policy $h_1$ takes an ``easy route'' that deviates from the expert at step 1, and incurs a small but nonzero cost. Policy $h_2$ tries to mimic the expert by first choosing to take the fast route; however it fails to mimic the expert from step 2 on and incurs a catastrophically high cost. Although $h_2$ has a smaller imitation loss than $h_1$, both policies' inability to keep up with $\pi^E$ subsequently makes $h_1$ actually a better choice.
\end{remark}

\section{Proofs for Section~\ref{sec:mftplp}}
\label{sec:mftplp-pf}

In the following, we 
provide detailed proofs for Theorem~\ref{thm:soil-main} and Corollary~\ref{cor:soil-main} in Section~\ref{sec:proofs}.
We first briefly review the interactive imitation learning for discrete action space setting in Section~\ref{sec:recap}.

\subsection{Notations and algorithm}
\label{sec:recap}
In this section, we first review some basic notations for interactive imitation learning introduced in Sections~\ref{sec:prelims} and~\ref{sec:mftplp} and then introduce additional notations for our analysis.

\textbf{Review of notations.}
The framework proposed by Ross et al.\cite{ross2011reduction} reduces finding a policy $\hat{\pi}$ with a small performance gap compared to the expert policy $J(\hat{\pi}) - J(\pie)$ into minimization of online regret. As shown in Theorem~\ref{thm:reduction}, to find a policy competitive with $\pie$, it suffices to find a sequence of policies $\cbr{\pi_n}_{n=1}^N$ that optimize the regret defined as $\Reg(N) = \sum_{n=1}^N \iterloss_n(\pi_n) - \min_{\pi \in \Bcal }\sum_{n=1}^N \iterloss_n(\pi)$, 
where $\iterloss_n(\pi) := \EE_{s \sim d_{\pi_n}}  \EE_{a\sim \pi(\cdot | s)}\sbr{I(a \neq \pie(s))}$.

We propose the \algms algorithm (Algorithm~\ref{alg:mftpl}) to achieve sublinear regret, assuming sample access to some covering distribution $\based$ (Assumption~\ref{assumption:smooth-base}) that satisfies that for any $\pi \in \Pi_\Bcal$ and $ s \in \Scal$, $\frac{d_{\pi}(s)}{\based(s)}\leq \frac{1}{\smooth}$. 
Meanwhile, we assume access to an offline classification oracle $\Ocal$ (Assumption~\ref{assum:oracle}), which, given a (multi)set of classification examples, returns the policy in the base policy class that has the smallest empirical classification error. 

Let $\Bcal$ be the base policy class that contains $B$ deterministic policies.
For $u \in \Delta(\Bcal)$, define $u[h]$ as the coordinate of $u$ corresponding to the $h \in \Bcal$.
Recall the definition of mixed policy class $\Pi_\Bcal := \cbr{\pi_u(a| s) = \sum_{h\in \Bcal} u[h] \cdot I(a = h(s)): u \in \Delta(\Bcal)}$. 

For completeness, we present Algorithm~\ref{alg:mftplp_full}, which integrates
the two functions in \algms for a more straightforward representation.



\begin{algorithm}[h]
\caption{\algms (Mixed Following The Perturbed Leader with Poisson Perturbations)}
\begin{algorithmic}[1]
\STATE \textbf{Input:} MDP $\MDP$, expert $\pie$, policy class $\Bcal$, offline classification oracle $\Ocal$, covering distribution $\based$,
sample size per iteration $K$, ensemble size $\Ensem$,  perturbation budget $\poiparam$.
\STATE Initialize $\incrdata = \emptyset$.
\FOR{$n=1,2,\ldots,N$}
\FOR{$\ensem=1,2,\ldots,\Ensem$}
\STATE Sample $\poi_{n,\ensem} \sim \poisson(\poiparam)$.
\STATE Sample $\pertdata_{n,\ensem} \gets $ draw i.i.d. perturbation samples $\{(\tilde{s}_{n,\ensem,x},\tilde{a}_{n,\ensem,x})\}_{x=1}^{\poi_{n,\ensem}}$ from $\incrdistbase = \based \otimes \Unif(\Acal)$.
\STATE Compute $h_{n,\ensem} \gets \Ocal( D \cup Q_{n,\ensem})$.
\label{line:compute_single_policy}
\ENDFOR
\STATE Set $\pi_{n}(a \mid s) := \frac{1}{\Ensem}\sum_{\ensem = 1}^\Ensem I( h_{n, e}(s) = a )$.
\label{line:ensemble_policy}
\STATE $\incrdata_n = \cbr{ (s_{n,k}, \pie(s_{n,k})) }_{k = 1}^K \gets$ sample $K$ states i.i.d. from $d_{\pi_{n}}$ by rolling out $\pi_{n}$ in $\MDP$, and query expert $\pie$ on these states. 
\STATE Aggregate datasets  $\incrdata \gets \incrdata \cup \incrdata_n$.
\ENDFOR
\STATE \textbf{Return} $\pi_{\hat{n}}(a \mid s) := \frac{1}{\Ensem}\sum_{\ensem = 1}^\Ensem I( \pi_{\hat{n}, e}(s) = a )$, where $\hat{n} \sim \Unif[N]$.
\end{algorithmic}
\label{alg:mftplp_full}
\end{algorithm}


\textbf{Additional notations.} \cite{li2022efficient} provides a framework for designing and analyzing regret-efficient interactive imitation learning algorithm for discrete action spaces. In a nutshell, the framework views the original classification-based regret minimization problem over $\Pi_\Bcal$ as an online linear optimization problem over $\Delta(\Bcal)$. Our design and analysis of \algms also adopt this framework, and thus we introduce the necessary notations in the context of \algms that facilitate this view.

In the following, we denote $\Onehot(h) \in \Delta(\Bcal)$ as the delta mass on a single policy $h$ within the base policy class $\Bcal$. 
We use $\incrdata_{1:n}$ as a shorthand for 
$\cup_{i=1}^{n} \incrdata_i$. 

Using the notations $\pi_u$ and $\Onehot$, in line 7
, we can write the policy returned from the oracle in the form of mixed policy, i.e. $h_{n,\ensem}=\pi_{u_{n,\ensem}}$, where $u_{n,\ensem} = \Onehot(\Ocal( \incrdata_{1:n-1} \cup \pertdata_{n,\ensem}) )$.


We define $\incrdist_{\pi}$ as the distribution of $(s,\pie(s))$, obtained by rolling out $\pi$ in $\MDP$ and querying the expert $\pie$. Denote $g_n^* := \rbr{\EE_{s \sim d_{\pi_n}} \sbr{I(h(s) \neq \pie(s))}}_{h \in \Bcal}$, which is a $B$ dimensional cost vector. 
We can rewrite $\iterloss_n (\pi_u)$ in the form of inner product as:
$$
\iterloss_n(\pi_u) 
:= \EE_{s \sim d_{\pi_n}}  \EE_{a\sim \pi_u(\cdot | s)}\sbr{I(a \neq \pie(s))}
= \EE_{s \sim d_{\pi_n}} \sum_{h \in \Bcal} u[h] \sbr{I(h(s) \neq \pie(s))} = \inner{g_n^*}{u}.
$$

Thus, the regret can be rewritten in an inner product form:
\begin{equation}
\Reg(N) 
=
\sum_{n=1}^N \iterloss_n(\pi_n) - \min_{\pi \in \Bcal }\sum_{n=1}^N \iterloss_n(\pi)
=
\sum_{n=1}^N \inner{g_n^*}{u_n} - \min_{u \in \Delta(\Bcal)}\sum_{n=1}^N \inner{g_n^*}{u},
\label{eqn:reg-olo}
\end{equation}

An equivalent representation of $\pi_{n+1}$ (line 9
) in the form of mixed policy is $\pi_{n+1} = \pi_{u_{n+1}}$, where $u_{n+1} = \frac1\Ensem \sum_{\ensem=1}^{\Ensem} u_{n+1,\ensem}$. 
By abusing $\incrdata_n$ to denote the uniform distribution over it, we define %
\begin{equation}
   g_n := \rbr{\EE_{(s,\pie(s)) \sim \incrdata_n} \sbr{I(h(s) \neq \pie(s))}}_{h \in \Bcal}, \;
\tilde{g}_{n,\ensem} =  \rbr{\frac{1}{K}\sum_{(\tilde{s},\tilde{a}) \in \pertdata_{n,\ensem}} \rbr{ I(h(\tilde{s}) \neq \tilde{a}) - \frac{A-1}{A} }  }_{h \in \Bcal}, 
\label{eqn: g_n}
\end{equation}
which stand for the cost vectors on $\incrdata_n$ and $Q_{n,\ensem}$ respectively.
With these notations,
we can rewrite $u_n$ as a sample-average version of the ``Follow-the-Perturbed-Leader'' algorithm~\cite{kalai2005efficient} over $E$ independent trials:
\begin{equation}
u_n = \frac 1 \Ensem \sum_{\ensem=1}^{\Ensem} \argmin_{u \in \Delta(\Bcal)} \inner{\sum_{i=1}^{n-1}g_i + \tilde{g}_{n,\ensem}}{u}.
\label{eqn:u-n-argmin}
\end{equation}
We give a formal proof of Eq.~\eqref{eqn:u-n-argmin} in Lemma~\ref{lem:rewrite_un}.

We define two $\sigma-$algebras for data and policies accumulated through the learning procedure of \algms:
\begin{equation}
    \Fcal_n := \sigma\rbr{ u_1 ,\incrdata_1,  u_2 ,\incrdata_2, \cdots u_{n}, \incrdata_{n}}, \:
    \Fcal^+_n := \sigma\rbr{ u_1 ,\incrdata_1,  u_2 ,\incrdata_2,  \cdots , u_{n}, \incrdata_{n},  u_{n+1}},
    \label{eqn: sigma_algebra}
\end{equation}
where it can be verified that filtration $(\Fcal_ n)_{n=1}^N$ and $(\Fcal^+_n)_{n=1}^N$ satisfies
$\Fcal_1 \subset \Fcal^+_1 \subset \Fcal_2 \subset \Fcal^+_2 \subset \cdots$.

Following the definition of perturbation sets $\pertdata_{n,\ensem}$ in Algorithm~\ref{alg:mftplp_full}, 
given $\poiparam > 0$, for any $n,n' \in [N]$ and any $\ensem,\ensem' \in [\Ensem]$, $\pertdata_{n,\ensem} $ and $ \pertdata_{n',\ensem'}$ are equal in distribution. With this observation, we introduce a random variable $\pertdata_n$ that has the same distribution as $\pertdata_{n,\ensem}$ and 
\[
\tilde{g}_n =  \rbr{\frac{1}{K}\sum_{(\tilde{s},\tilde{a}) \in \pertdata_{n}} \rbr{ I(h(\tilde{s}) \neq \tilde{a}) - \frac{A-1}{A} }  }_{h \in \Bcal}
\]
which has the same distribution as 
$\tilde{g}_{n,\ensem}$. Without loss of generality, $\forall n\in [N], \ensem \in [\Ensem]$,  for any function $f$ of $(\pertdata_{n,\ensem}, \incrdata_{1:n-1})$,we abbreviate 
$\EE\sbr{f(\pertdata_{n,\ensem},\incrdata_{1:n-1})| \Fcal_{n-1}}$ 
as $\EE_{\pertdata_n}\sbr{f(\pertdata_n,\incrdata_{1:n-1})}$ throughout and define
\begin{equation}
u_{n}^* :=
\EE\sbr{ u_{n,e} | \Fcal_{n-1} }
\label{eqn:u_star}   
\end{equation}
Similar to Eq.~\eqref{eqn:u-n-argmin} for $u_n$, we rewrite 
\[
u_n^* = \EE_{\pertdata_n}\sbr{\Onehot(\Ocal( \incrdata_{1:n-1} \cup \pertdata_n)} = \EE_{\pertdata_n} \sbr{ \argmin_{u \in \Delta(\Bcal)} \inner{\sum_{i=1}^{n-1}g_i + \tilde{g}_n}{u}   },
\]


Meanwhile, given any function $f'$ of $(\incrdata_n, \incrdata_{1:n-1})$, we abbreviate $\EE \sbr{ f'( \incrdata_n, \incrdata_{1:n-1}) | \Fcal^+_{n-1}}$ as $\EE_{\incrdata_n}\sbr{ f'(\incrdata_n, \incrdata_{1:n-1}) }$. We further define  
\begin{equation}
    u_{n+1}^{**} :=
    \EE\sbr{ u_{n+1,e} \mid \Fcal_{n-1}^+ }
    .
    \label{eqn:u_star_star}
\end{equation}
By the law of iterated expectation, this can be also written as 

\begin{equation}
u_{n+1}^{**} = 
  \EE \sbr{ \EE \sbr{ u_{n+1,e} | \Fcal_{n} } | \Fcal^+_{n-1}}
  =
  \EE\sbr{ u_{n+1}^* \mid \Fcal_{n-1}^+ }
  = 
  \EE_{\incrdata_n} \sbr{ u_{n+1}^* }
  = 
  \EE_{\incrdata_n} \EE_{Q_{n+1,e}} \sbr{ u_{n+1,e} }
\label{eqn:u-star-star-2}
\end{equation}
where the second equality follows from the definition of $u_{n+1}^*$, and the
third equality uses the observation that $u_{n+1}^*$ is a function of $(\incrdata_n, \incrdata_{1:n-1})$, and the last equality is from that $u_{n+1}^* = \EE_{Q_{n+1,e}} \sbr{ u_{n+1,e} }$. 

By this observation, $u_{n+1}^{**}$ can be rewritten as 
$$ u_{n+1}^{**}  = \EE_{ \incrdata_{n}} \EE_{\pertdata_{n+1}} \sbr{ \argmin_{u \in \Delta(\Bcal)} \inner{\sum_{i=1}^{n-1}g_i + g_{n}+ \tilde{g}_{n+1}}{u}   }.$$


 We remark that the notations $u_n, u_n^*, u_n^{**}$, as well as $g_n, g_n^*, \tilde{g}_n$, are introduced solely for analytical purposes. 
 
 As a quick recap, we provide a dependency graph of important variables that appear in the analysis in Figure~\ref{fig:dependency_graph}, 
 while summarizing frequently-used notations in Table~\ref{notation-table} below.
 

\begin{table}[!htb]
    \caption{A review of  notations in this paper.} 
  \label{notation-table}
      \centering
      \makebox[0pt]{
        \begin{tabular}{llll}
    \textbf{Name}     & \textbf{Description} & \textbf{Name} & \textbf{Description}    \\
    \hline \\
    $\MDP$ & Markov decision process & $ \Ocal$  & Classification oracle\\
    $H$ & Episode length &$\Pi_\Bcal$ &  Mixed policy class\\
    $t$ & Time step in $\MDP$  & $u$ & Ensemble policy probability weight\\
    $\Scal$ & State space  &$\pi_u$ & Ensemble policy induced by $u$\\
    $S$ & State space size & $u[h]$ & Ensemble weight for $h$ in $u$  \\
    $s$ & State & $u_n$ & Ensemble policy weight at round $n$\\
    $\Acal$ & Action space & $K$  & Sample budget per round \\
    $A$ & Action space size & $k$ & Sample iteration index\\
    $a$ & Action & $D$ & Aggregated dataset \\
    $\rho$ &  Initial distribution &$\incrdata_n$ & Set of Classification examples at round $n$\\
    $P$ & Transition probability distribution & $\est_n$ & Loss vector induced by $\incrdata_n$\\
    $C$ & Cost distribution  & $\incrdist_{\pi}$ & $(s,\pie(s))$ distribution induced by $\pi$, $\MDP$ and $\pie$ \\
    $c$ & Cost& $g^*_n$ & Expected loss vector induced by $\pi_n$, $\MDP$ and $\pie$\\
    $\pi$ & Stationary policy  & $\based$ & Covering base distribution\\
    $\pi(\cdot | s)$ & Action distribution of $\pi$ given state $s$ & $\incrdistbase$ & $(s,a)$ distribution induced by $\based \otimes \Unif(\Acal)$ \\
    $d_{\pi}$  & State occupancy distribution  & $\sigma$ &  Smooth factor \\
    $\tau$  & Trajectory   &  $\Ensem$ &  Ensemble size\\
    $J(\pi)$ & Expected cumulative cost  &  $\ensem$ &  Ensemble index\\
    $Q_\pi$ &  Action value function & $\poisson(\poiparam)$ & Poisson distribution\\
    $V_\pi$ &  State value function &$\lambda$ &  Perturbation budget \\
    $\pie$ & Expert policy  &$X_{n,e}$ &  Perturbation set size \\
    $\ell$ & Loss function & $x$ & Sample index within a perturbation set \\
    $\mu$ & Recoverability of $(\pie,M)$ for $\ell$  &$Q_{n,e}$ &  Perturbation set  \\
    $N$ & Number of learning rounds & $\tilde{g}_{n,e}$ & Perturbation loss vector in $\RR^B$ induced by $Q_{n,e}$\\
    $n$  & Learning round index &  $\Fcal_n,\Fcal^+_n$ & $\sigma$-algebras induced by $\{u_i\}_{i = 1}^{n}$ and $\{\incrdata_i\}_{i = 1}^{n-1}$\\
    $\iterloss_n(\pi)$ & Online loss function &  $\EE_{\incrdata_n}$ & Expectation w.r.t. $\incrdata_n \sim (\incrdata_{\pi_n}^{\pie})^K$ \\
    $\Bcal$ &  Deterministic base policy class  &  $\EE_{Q_n}$ & Expectation w.r.t. $Q_n \sim (\incrdistbase)^X,$ where $X \sim \poisson(\poiparam)$ \\
    $B$ &  Base policy class size  & $u^*_n$ & Expectation of $u_n$ w.r.t $Q_n$\\
    $h$ &  Deterministic stationary policy in $\Bcal$ &  $u^{**}_n$ & Expectation of $u_n$ w.r.t $Q_n$ and $\incrdata_{n-1}$ \\
    $\Reg(N)$  &  Online regret & $\sbr{N}$ &Set $\{1,2,\cdots,N\}$\\
    $\Unif(\Ecal)$& Uniform distribution over $\Ecal$  & $\Delta(\Ecal)$  & All probability distributions over  $\Ecal$\\
    $\prob(U)$ & Probability of event  $U$  & $\Onehot(\Bcal)$ & Delta mass (one-hot vector) on $h\in \Bcal$\\
    $\delta$ & Failure probability  &$I(\cdot)$ & Indicator function\\
    \end{tabular}
    }
\end{table}

\begin{figure*}[htbp]
  \centering
  \includegraphics[width=1\linewidth]{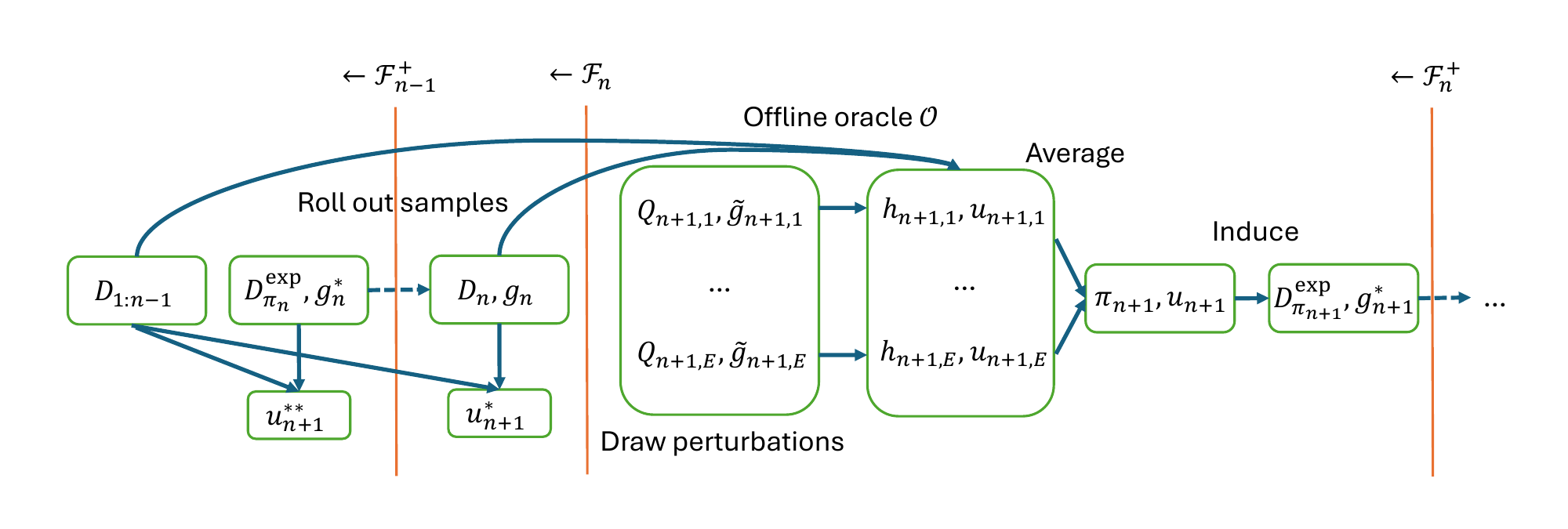}
  \vspace{.1in}
  \caption{Dependency graph of notations that appear in the analysis. Solid and dashed arrows indicate deterministic and stochastic dependence, respectively. Note that all $(Q_{n+1,e})_{e \in [E]}$'s are drawn independently from fixed sample perturbation distributions and can be treated as fresh iid random examples.
  }
  \label{fig:dependency_graph}
\end{figure*}

\newpage

\subsubsection{Auxiliary Lemmas}

\begin{lemma}
$\pi_n = \pi_{u_n}$, where 
\[
u_n = \frac 1 \Ensem \sum_{\ensem=1}^{\Ensem} \argmin_{u \in \Delta(\Bcal)} \inner{\sum_{i=1}^{n-1}g_i + \tilde{g}_{n,\ensem}}{u}.
\]
\label{lem:rewrite_un}
\end{lemma}
\begin{proof}


\begin{equation}
\begin{aligned}
    u_n =& \frac{1}{\Ensem} \sum_{\ensem=1}^{\Ensem} u_{n,\ensem}
    =\frac1\Ensem \sum_{\ensem=1}^{\Ensem}\Onehot(\Ocal( \incrdata_{1:n-1} \cup \pertdata_{n,\ensem}) ) \\
 =&
 \frac1\Ensem \sum_{\ensem=1}^{\Ensem}\Onehot(
 \argmin_{h \in \Bcal} \EE_{(s,a) \sim \incrdata_{1:n-1} \cup \pertdata_{n,\ensem}} \sbr{  I(h(s) \neq a) }  )\\
 =&
 \frac1\Ensem \sum_{\ensem=1}^{\Ensem}\Onehot \rbr{
 \argmin_{h \in \Bcal} \rbr{ \frac{1}{(n-1)K+\poi_{n,\ensem}} \rbr{ \sum_{i=1}^{n-1}\sum_{(s,a) \in {\incrdata_i}}  \rbr{I(h(s) \neq a)}
 +
 \sum_{{(\tilde{s},\tilde{a}) \in \pertdata_{n,\ensem} } } \rbr{ I(h(\tilde{s}) \neq \tilde{a})}   }}  }\\
 =&
 \frac1\Ensem \sum_{\ensem=1}^{\Ensem}\Onehot \rbr{
 \argmin_{h \in \Bcal} \rbr{ \sum_{i=1}^{n-1} \EE_{(s,a) \sim {\incrdata_i}} \sbr{  I(h(s) \neq a) }
 +
 \frac{1}{K} \sum_{\sum_{(\tilde{s},\tilde{a}) \in \pertdata_{n,\ensem} } } \rbr{  I(h(\tilde{s}) \neq \tilde{a}) }   }  }\\
 =&
 \frac1\Ensem \sum_{\ensem=1}^{\Ensem} \argmin_{u \in \Delta(\Bcal)} \inner {
 \rbr{ \sum_{i=1}^{n-1} \EE_{(s,a) \sim {\incrdata_i}} \sbr{  I(h(s) \neq a) }
 +
 \frac{1}{K} \sum_{\sum_{(\tilde{s},\tilde{a}) \in \pertdata_{n,\ensem} } } \rbr{  I(h(\tilde{s}) \neq \tilde{a}) }   }_{h \in \Bcal} }{u}  \\
 =&
 \frac1\Ensem \sum_{\ensem=1}^{\Ensem} \argmin_{u \in \Delta(\Bcal)} \inner {
 \rbr{ \sum_{i=1}^{n-1} \EE_{(s,a) \sim {\incrdata_i}} \sbr{  I(h(s) \neq a) }
 +
 \frac{1}{K} \sum_{\sum_{(\tilde{s},\tilde{a}) \in \pertdata_{n,\ensem} } } \rbr{  I(h(\tilde{s}) \neq \tilde{a}) - \frac{A-1}{A} }  }_{h \in \Bcal} }{u}  \\
 =& \frac 1 \Ensem \sum_{\ensem=1}^{\Ensem} \argmin_{u \in \Delta(\Bcal)} \inner{\sum_{i=1}^{n-1}g_i + \tilde{g}_{n,\ensem}}{u},
\label{eqn:u_n}
\end{aligned}
\end{equation}
where we apply the invariant property of $\argmax$ operator on positive scaling and shifting. Note that $\pertdata_{n,\ensem}$ contains $\poi_{n,\ensem}$ perturbation examples and each $\incrdata_n$ contains $K$ examples.
\end{proof}

\subsection{Proof of Theorem~\ref{thm:soil-main}}
\label{sec:proofs}
The proofs in this section follows the flowchart in Figure~\ref{fig:flowchart}, which is divided to three stages:

\begin{itemize}

\item \textbf{At stage 1}, we apply the existing results~\cite{li2022efficient} to reduce bounding the distribution-dependent online regret
$\Reg(N) = \sum_{n=1}^N \inner{g_n^*}{u_n} - \min_{u \in \Delta(\Bcal)}\sum_{n=1}^N \inner{g_n^*}{u}$ to bounding the data-dependent online regret $\sum_{n = 1}^N \inner{g_n}{u_n^*} - \min_{u \in \Delta(\Bcal)} \sum_{n = 1}^N \inner{g_n}{u}$ using standard martingale concentration inequalities.
By the end of stage 1, it remains to bound the regret of the idealized sequence of predictors  $\{u_n^*\}_{n=1}^N$ on the observed linear losses $\{g_n\}_{n=1}^N$.

\item \textbf{At stage 2}, a bound on the ``ideal regret'' is established by a standard analysis of an in-expectation version of the ``Follow the perturbed Leader'' algorithm. By Lemma~\ref{lem: online_mftpl_poisson} and~\ref{lem: ftpl}, we prove that
$$ \sum_{n = 1}^N \inner{g_n}{u_n^*} - \min_{u \in \Delta(\Bcal)} \sum_{n = 1}^N \inner{g_n}{u} 
\leq
\EE_{\pertdata_1} \sbr{ \max_{u \in \Delta(\Bcal)} \inner{-\tilde{g}_1}{u}}
+
\sum_{n=1}^N \inner{g_n}{u_n^* - u_{n+1}^*}.
$$
The first term on the right hand side can be straightforwardly bounded by Lemma~\ref{lem: smooth_bias}. It remains to bound $\sum_{n=1}^N \inner{g_n}{u_n^* - u_{n+1}^*}$.

\item \textbf{At stage 3}, we aim to control $\sum_{n=1}^N \inner{g_n}{u_n^* - u_{n+1}^*}$.
Existing smoothed online learning analysis~\citep{haghtalab2022oracle} (implicitly) provide bounds on $\EE\sbr{ \sum_{n=1}^N \inner{g_n}{u_n^* - u_{n+1}^*} }$, which is insufficient for our goal of establishing high-probability bounds. Furthermore,~\citet{haghtalab2022oracle} only considers the online learning setting where one example is given at each round and the action space is binary, which is insufficient for 
batch mode multiclass online classification setting for our imitation learning application. 
To bridge the gap between existing techniques and our problem, we further decompose 
$\sum_{n=1}^N \inner{g_n}{u_n^* - u_{n+1}^*}$ to three terms: stability term, generalization error, and an approximation term. Our analysis of stability term and generalization error generalizes the analysis of~\citet{haghtalab2022oracle} to multiclass batch setting (Lemmas~\ref{lem: stability} and~\ref{lem: generalization}). For the new approximation term, we observe that it has martingale structure and thus concentrates well (see proof of Lemma~\ref{lem: divergence_bound}). With these, we have all terms bounded and conclude Theorem~\ref{thm:soil-main}.
\end{itemize}

We provide a roadmap of our analysis of the the three stages in Figure~\ref{fig:flowchart}, highlighting the key quantities and the key lemmas, as well as their relationships. 


\begin{figure*}[htbp]
  \centering
  \includegraphics[width=1\linewidth]{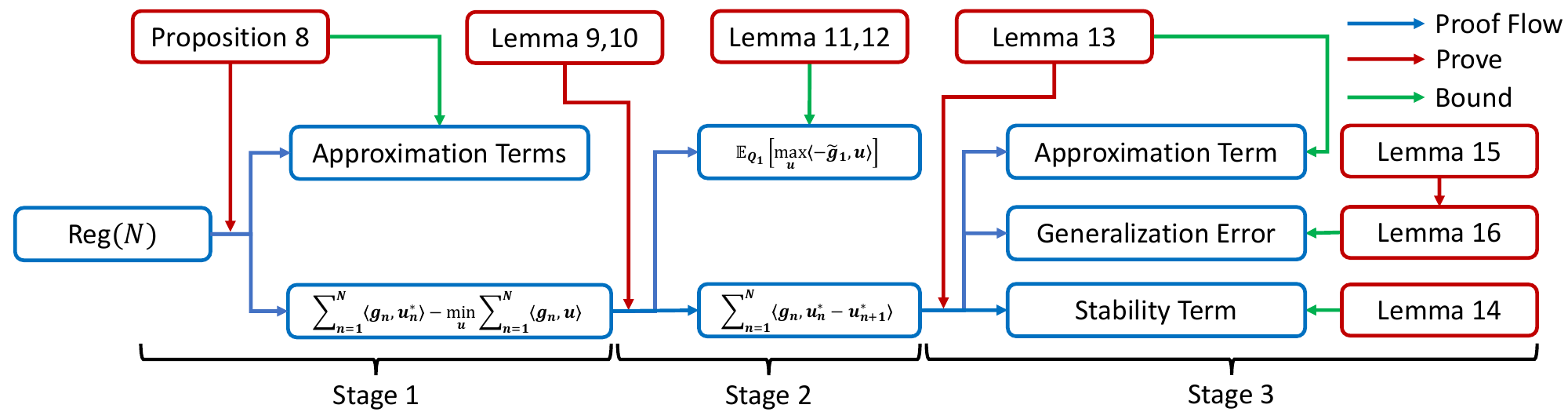}
  \vspace{.1in}
  \caption{Flowchart of the proofs to bound the regret.}
  \label{fig:flowchart}
\end{figure*}


\subsubsection{Proof for Stage 1}
Following the previous results \algm, we guarantee that our algorithm \algms satisfies:

\begin{proposition}
    For any $\delta \in (0,1]$, \algms outputs policies $\cbr{\pi_n}_{n=1}^N$ such that with probability at least 
    $1- \delta/2$, 
    \[
    \Reg(N) 
    \leq
    \sum_{n = 1}^N \inner{g_n}{u_n^*} - \min_{u \in \Delta(\Bcal)} \sum_{n = 1}^N \inner{g_n}{u}
    +
     O\rbr{  \sqrt{\frac{N\ln(B/\delta)}{K}} }
     +
     N\sqrt{\frac{ 2A\rbr{\ln(NS)+\ln(\frac{12}{\delta})}}{\Ensem}} .
    \]
    \label{prop: mftpl}
\end{proposition}

\begin{proof}
By the inner product form of the regret (Eq.~\eqref{eqn:reg-olo}), we have the following decomposition

\[
\begin{aligned}
\Reg(N) 
=&
\sum_{n=1}^N \inner{g_n^*}{u_n} - \min_{u \in \Delta(\Bcal)}\sum_{n=1}^N \inner{g_n^*}{u}\\
=&
\sum_{n=1}^N \inner{g_n}{u_n} - \min_{u \in \Delta(\Bcal)}\sum_{n=1}^N \inner{g_n}{u}
+
\sum_{n=1}^N \inner{g_n^* - g_n}{u_n} 
+
\min_{u \in \Delta(\Bcal)}\sum_{n=1}^N \inner{g_n}{u}
- \min_{u \in \Delta(\Bcal)}\sum_{n=1}^N \inner{g_n^*}{u} \\
\leq &
\sum_{n=1}^N \inner{g_n}{u_n} - \min_{u \in \Delta(\Bcal)}\sum_{n=1}^N \inner{g_n}{u}
+
\sqrt{\frac{2N \ln(\frac{12}{\delta})}{K}}
+
\sqrt{2N\frac{\ln (B)+\ln(\frac{12}{\delta})}{K}}\\
=&
\sum_{n=1}^N \inner{g_n}{u_n^*} 
-
\min_{u \in \Delta(\Bcal)}\sum_{n=1}^N \inner{g_n}{u}
+
\sum_{n=1}^N \langle {\est_n}, {u_n - u_n^*} \rangle 
 + 
 O\rbr{  \sqrt{\frac{N\ln(B/\delta)}{K}} } \\
\leq&
\sum_{n=1}^N \inner{g_n}{u_n^*} 
-
\min_{u \in \Delta(\Bcal)}\sum_{n=1}^N \inner{g_n}{u}
+
 O\rbr{  \sqrt{\frac{N\ln(B/\delta)}{K}} }
+
 N\sqrt{\frac{ 2A\rbr{\ln(NS)+\ln(\frac{12}{\delta})}}{\Ensem}},
\end{aligned}
\]
where the first and second inequalities are from propositions~\ref{prop:logger} and~\ref{prop:sparsification} in Appendix~\ref{sec:deferred_proofs}
respectively, which in turn are from the proposition 6 and the proof of lemma 8 in~\cite{li2022efficient}.
\end{proof}

\subsubsection{Proof for Stage 2}
In this section, we prove a bound on the regret of the ``idealized policy sequence'' $\cbr{\pi_{u_n^*}}_{n=1}^N$, i.e.,  
$\sum_{n=1}^N \inner{g_n}{u_n^*} - \min_{u \in \Delta(\Bcal)}\sum_{n=1}^N \inner{g_n}{u}$.
Such result should be well-known in the context of analysis of the ``Follow the Perturbed Leader'' algorithm in online linear optimization~\cite{kalai2005efficient}; we provide full details here since we cannot find in the literature this exact lemma statement we need. 
An in-expectation version of a similar bound has been implicitly shown in \citet{haghtalab2022oracle}.
in the language of admissible relaxations~\cite{rakhlin2012relax}. 

\begin{lemma}
For $g_n$ induced by \algms, MDP $\MDP$ and expert $\pie$, the sequence of $u_n^*$ defined in equation~\eqref{eqn:u_star} satisfies
\begin{equation}
\sum_{n=1}^N \inner{g_n}{u_n^*} 
-
\min_{u \in \Delta(\Bcal)}\sum_{n=1}^N \inner{g_n}{u}
\leq
\EE_{\pertdata_1} \sbr{ \max_{u \in \Delta(\Bcal)} \inner{-\tilde{g}_1}{u}}
+
\sum_{n=1}^N \inner{g_n}{u_n^* - u_{n+1}^*}, 
\label{eqn:online_eftpl}
\end{equation}
where $\tilde{g}_1 :=  \rbr{\frac{1}{K}\sum_{(\tilde{s},\tilde{a}) \in \pertdata_1} \rbr{ I(h(\tilde{s}) \neq \tilde{a}) - \frac{A-1}{A} }  }_{h \in \Bcal}$, $\pertdata_1 := \cbr{\{( \tilde{s}_{1,x},\tilde{a}_{1,x})\}_{x=1}^ \poi :  \poi_1 \sim Poi(\poiparam), (\tilde{s}_{x},\tilde{a}_{x}) \sim \incrdistbase }$.

\begin{remark}
Intuitively, the right hand side of Eq.~\eqref{eqn:online_eftpl}
exhibits a bias-variance tradeoff: $\EE_{\pertdata_1} \sbr{ \max_{u \in \Delta(\Bcal)} \inner{-\tilde{g}_1}{u}}$ and $\sum_{n=1}^N \inner{g_n}{u_n^* - u_{n+1}^*}$ are the ``bias'' and ``variance'' terms that increases and decrease with respect to the amount of perturbation noise $\lambda$, respectively. We bound the first term in Lemma~\ref{lem: smooth_bias}
and the second term in Section~\ref{sec:pf-stage3}, respectively.

\end{remark}


\label{lem: online_mftpl_poisson}
\end{lemma}

\begin{proof}
    
Notice that $\sum_{n=1}^N \inner{g_n}{u_n^*}$ appears on both sides of equation \eqref{eqn:online_eftpl}, by arranging the terms, it suffices to show
\[
\sum_{n=1}^N \inner{g_n}{ u_{n+1}^*}
\leq
\min_{u \in \Delta(\Bcal)}\sum_{n=1}^N \inner{g_n}{u}
+
\EE_{\pertdata_1} \sbr{ \max_{u \in \Delta(\Bcal)} \inner{-\tilde{g}_1}{u}}
\]
By Lemma~\ref{lem: ftpl}, we have $\forall n \in [N]$,
\[
\inner{g_n}{u_{n+1}^*}
\leq 
\EE_{\pertdata_n} \sbr{ \max_{u \in \Delta(\Bcal)} \inner{-\sum_{i=1}^{n-1} g_i -\tilde{g}_n}{u}}
- 
\EE_{\pertdata_{n+1}} \sbr{ \max_{u \in \Delta(\Bcal)} \inner{-\sum_{i=1}^{n} g_i -\tilde{g}_{n+1}}{u}}.
\]
Summing the above inequality over $n \in [N]$ (and noting that the right hand side is a telescoping sum) gives that 
\begin{equation}
\sum_{n=1}^N \inner{g_n}{ u_{n+1}^*}
\leq
\EE_{\pertdata_1} \sbr{ \max_{u \in \Delta(\Bcal)} \inner{-\tilde{g}_1}{u}}
-
\EE_{\pertdata_{N+1}} \sbr{ \max_{u \in \Delta(\Bcal)} \inner{-\sum_{n=1}^{N} g_n -\tilde{g}_{N+1}}{u}},
\label{eqn: telescoping}
\end{equation}

Meanwhile, we further notice that 
\begin{equation}
\EE_{\pertdata_{N+1}} \sbr{ \max_{u \in \Delta(\Bcal)} \inner{-\sum_{n=1}^{N} g_n -\tilde{g}_{N+1
}}{u}}
\geq
\max_{u \in \Delta(\Bcal)} \EE_{\pertdata_{N+1}} \sbr{ \inner{-\sum_{n=1}^{N} g_n -\tilde{g}_{N+1}}{u}}
=
\max_{u \in \Delta(\Bcal)} \inner{-\sum_{n=1}^{N} g_n }{u},
\label{eqn: jensen_ineq}
\end{equation}
where we apply Jensen's inequality and observe that $\forall s \in \Scal, h \in \Bcal$, $\EE_{y \sim \Unif(\Acal)} \sbr{ I(h(s) \neq y)} = (A-1)/A$, meaning that 
$\forall u \in \Delta(\Bcal)$, 
\[
\EE_{\pertdata_{N+1}} \sbr{ \inner{-\tilde{g}_{N+1}}{u}} 
=
\EE_{\pertdata_{N+1}} \sbr{  \inner{-\rbr{\frac{1}{K}\sum_{(\tilde{s},\tilde{a}) \in \pertdata_{N+1}}  \rbr{ I(h(\tilde{s}) \neq \tilde{a}) -(A-1)/A  } }_{h \in \Bcal}}{u}}
=0.
\]

Therefore, we conclude the proof by plugging equation~\eqref{eqn: jensen_ineq} in \eqref{eqn: telescoping}:
\[
\sum_{n=1}^N \inner{g_n}{ u_{n+1}^*}
\leq
\EE_{\pertdata_1} \sbr{ \max_{u \in \Delta(\Bcal)} \inner{-\tilde{g}_1}{u}}
-
\max_{u \in \Delta(\Bcal)} \inner{-\sum_{n=1}^{N} g_n }{u}
=
\EE_{\pertdata_1} \sbr{ \max_{u \in \Delta(\Bcal)} \inner{-\tilde{g}_1}{u}}
+
\min_{u \in \Delta(\Bcal)} \inner{\sum_{n=1}^{N} g_n }{u}.
\]

\end{proof}

\begin{lemma}
For $\cbr{g_n}_{n=1}^N$ induced by \algms and $\cbr{u_n^*}_{n=1}^N$ defined in Eq.~\eqref{eqn:u_star}, 
\[
\inner{g_n}{u_{n+1}^*}
\leq 
\EE_{\pertdata_{n}} \sbr{ \max_{u \in \Delta(\Bcal)} \inner{-\sum_{i=1}^{n-1} g_i -\tilde{g}_{n}}{u}}
- 
\EE_{\pertdata_{n+1}} \sbr{ \max_{u \in \Delta(\Bcal)} \inner{-\sum_{i=1}^{n} g_i -\tilde{g}_{n+1}}{u}}.
\]
    \label{lem: ftpl}
\end{lemma}

\begin{proof}

Note that $Q_n$ and $Q_{n+1}$ have identical probability distributions. Therefore, the lemma statement is equivalent to:
\[
\inner{g_n}{u_{n+1}^*}
\leq 
\EE_{\pertdata} \sbr{ \max_{u \in \Delta(\Bcal)} \inner{-\sum_{i=1}^{n-1} g_i -\tilde{g}}{u}}
- 
\EE_{\pertdata} \sbr{ \max_{u \in \Delta(\Bcal)} \inner{-\sum_{i=1}^{n} g_i -\tilde{g}}{u}},
\]
where 
$\tilde{g} :=  \rbr{\frac{1}{K}\sum_{(\tilde{s},\tilde{a}) \in \pertdata} \rbr{ I(h(\tilde{s}) \neq \tilde{a}) - \frac{A-1}{A} }  }_{h \in \Bcal}$.

By the definition of $u_n^*$ in equation~\eqref{eqn:u_star}, we have: 
\[
u_n^* 
=
\EE_{\pertdata} \sbr{ \argmin_{u \in \Delta(\Bcal)} \inner{\sum_{i=1}^{n-1}g_i + \tilde{g}}{u}  }
=
\EE_{\pertdata} \sbr{ \argmax_{u \in \Delta(\Bcal)} \inner{\sum_{i=1}^{n-1} - g_i - \tilde{g}}{u}  }.
\]   

By denoting $u_{n,\pertdata} := \argmax_{u \in \Delta(\Bcal)} \inner{-\sum_{i=1}^{n-1}g_i - \tilde{g}}{u}$, we notice that $\EE_{\pertdata} \sbr{u_{n,\pertdata}} = u_{n}^*$and write:
\[
\begin{aligned}
\EE_{\pertdata} \sbr{ \max_{u \in \Delta(\Bcal)} \inner{-\sum_{i=1}^{n-1} g_i -\tilde{g}}{u}}
+
\inner{-g_n}{u_{n+1}^*}
= &
\EE_{\pertdata} \sbr{ \inner{-\sum_{i=1}^{n-1} g_i -\tilde{g}}{u_{n,\pertdata}}}
+
\inner{-g_n}{u_{n+1}^*}\\
\geq &
\EE_{\pertdata} \sbr{ \inner{-\sum_{i=1}^{n-1} g_i -\tilde{g}}{u_{n+1,\pertdata}}}
+
\inner{-g_n}{u_{n+1}^*}\\
= &
\EE_{\pertdata} \sbr{ \inner{-\sum_{i=1}^{n-1} g_i -\tilde{g}}{u_{n+1,\pertdata}}}
+
\EE_{\pertdata} \sbr{\inner{-g_n}{u_{n+1,\pertdata} }}\\
= &
\EE_{\pertdata} \sbr{ \inner{-\sum_{i=1}^{n} g_i -\tilde{g}}{u_{n+1,\pertdata}}} \\
= &
\EE_{\pertdata} \sbr{ \max_{u \in \Delta(\Bcal)} \inner{-\sum_{i=1}^{n} g_i -\tilde{g}}{u}}
\end{aligned}
\]
where the inequality is by the optimality of $u_{n,\pertdata}$. We conclude our proof by rearranging the terms.
\end{proof}

\begin{lemma}
    \[
    \EE_{\pertdata_1} \sbr{ \max_{u \in \Delta(\Bcal)} \inner{-\tilde{g}_1}{u}}
    \leq
    \sqrt{\frac{\poiparam \ln(B) }{ 2K^2} }
    \]
    \label{lem: smooth_bias}
\end{lemma}
\begin{proof}
We first recall the definition of $\tilde{g}_1 = \rbr{\frac{1}{K}\sum_{(\tilde{s},\tilde{a}) \in \pertdata_1} \rbr{ I(h(\tilde{s}) \neq \tilde{a}) - \frac{A-1}{A} }  }_{h \in \Bcal}$ in equation~\eqref{eqn: g_n} and rewrite
\begin{align}
\EE_{\pertdata_1} \sbr{ \max_{u \in \Delta(\Bcal)} \inner{-\tilde{g}_1}{u}}
= &
  \EE_{\pertdata_1} \sbr{ \max_{h \in \Bcal} \frac{1}{K}\sum_{(\tilde{s},\tilde{a}) \in \pertdata_1} \rbr{
  \frac{A-1}{A} 
  - 
  I(h(\tilde{s}) \neq \tilde{a}) }      } \nonumber \\
  = &
\frac{1}{K}\EE_{\pertdata_1} \sbr{
\max_{h \in \Bcal}  \rbr{ 
\poi_1 \frac{A-1}{A} 
-\sum_{(\tilde{s},\tilde{a}) \in \pertdata_1} I(h(\tilde{s}) \neq \tilde{a})}     }  
\label{eqn:bias-rhs}
\end{align}

where the size of $\pertdata_1$ is denoted by $\poi_1$. When $(\tilde{s}, \tilde{a}) \sim \Dcal_0$, $\tilde{a} \sim \Unif(\Acal)$, therefore, it is not hard to see $\forall h \in \Bcal$,
\[
\EE_{(\tilde{s}, \tilde{a}) \sim \Dcal_0} \sbr{I(h(\tilde{s}') \neq \tilde{a}')} 
=
\frac{A-1}{A}.
\]

Thus, conditioned on $X_1$, for every $h \in \Bcal$,  $\poi_1 \frac{A-1}{A} 
-\sum_{(\tilde{s},\tilde{a}) \in \pertdata_1} I(h(\tilde{s}) \neq \tilde{a})$ is a zero-mean, $\frac{X_1}{4}$-subgaussian random variable. 
Therefore, Massart's Lemma (see Lemma~\ref{lem:massart} below) implies that 
\[
\EE_{\pertdata_1} \sbr{
\max_{h \in \Bcal}  \rbr{ 
\poi_1 \frac{A-1}{A} 
-\sum_{(\tilde{s},\tilde{a}) \in \pertdata_1} I(h(\tilde{s}) \neq \tilde{a})}  \mid X_1  }
\leq 
\sqrt{ \frac{X_1 \ln B}{2} }.
\]
By the law of iterated expectation and $\EE\sbr{X_1} = \lambda$, we have that Eq.~\eqref{eqn:bias-rhs} can be bounded by 
\[
\frac{1}{K} \EE\sbr{ \sqrt{ \frac{X_1 \ln B}{2} } }
\leq 
\frac{1}{K} \sbr{ \sqrt{ \frac{\EE\sbr{X_1} \ln B}{2} } }
= 
\sqrt{ \frac{\lambda \ln B}{2 K^2} }.
\qedhere
\]
\end{proof}

\begin{lemma}[Massart's Lemma (Lemma 26.8 of~\citet{shalev2014understanding})]
\label{lem:massart}
Suppose $X_1, \ldots, X_B$ is collection of zero-mean, $\sigma^2$-subgaussian random variables. Then, 
\[
\EE\sbr{ \max_{i=1}^B X_i }
\leq 
\sigma \sqrt{2 \ln(B)}. 
\]
\end{lemma}

\subsubsection{Proof for Stage 3}
\label{sec:pf-stage3}

\begin{lemma}
For any $\delta \in [0,1]$, $\poiparam \geq \max\cbr{ \frac{2AK^2}{\sigma}, \frac{8 AK \ln(KN)}{\sigma} }$, the sequence of $\cbr{g_n}$ defined in equation~\eqref{eqn: g_n} and $\cbr{u_n^*}$ defined in equation~\eqref{eqn:u_star} satisfies that
with probability at least $1-\delta/2$,
\begin{equation}
\sum_{n=1}^N \inner{g_n}{u_n^* - u_{n+1}^*}
\leq
N\sqrt{\frac{AK^2}{\poiparam \sigma}} 
+
N\sqrt{\frac{2A \ln(N)\ln(BN^2)}{\poiparam \sigma}} 
    + 
    \frac{\poiparam \sigma}{4A N \ln(N)}
+
4\sqrt{2N\ln(4/\delta)}.
\end{equation}
\label{lem: divergence_bound}
\end{lemma}

\begin{proof}
To begin with, we recall that $u_{n+1}^{**} = \EE_{\incrdata_{n}} \sbr{u_{n+1}^*}  $ as shown in Eq.~\eqref{eqn:u-star-star-2} and decompose $\sum_{n=1}^N \inner{g_n}{u_n^* - u_{n+1}^*}$ into three parts as follows:


\begin{equation}
\begin{aligned}
\sum_{n=1}^N \inner{g_n}{u_n^* - u_{n+1}^*}    
=&
\sum_{n=1}^N \inner{g_n}{u_n^* - u_{n+1}^{**}} + \sum_{n=1}^N \inner{g_n}{u_{n+1}^{**} - u_{n+1}^*} \\
=&
\underbrace{\sum_{n=1}^N \inner{g_n}{u_n^* - u_{n+1}^{**}}}_{\text {Stability Term}}
+
\underbrace{\sum_{n=1}^N \EE_{\incrdata_n}\sbr{ \inner{g_n^* - g_n}{ u_{n+1}^*}   } }_{\text {Generalization Error}}\\
& +
\underbrace{ \sum_{n=1}^N \inner{g_n}{u_{n+1}^{**} - u_{n+1}^*} 
-
\sum_{n=1}^N \EE_{\incrdata_n}\sbr{ \inner{g_n^* - g_n}{ u_{n+1}^*}  }  }_{\text {Approximation Term}}.
\label{eqn: stability_eneralization}
\end{aligned}
\end{equation}

As shown in equation~\eqref{eqn: stability_eneralization}, we apply a decomposition similar to Lemma 4.4 of \citet{haghtalab2022oracle}, which also involves a  stability term and a generalization error term. 
Our decomposition uniquely 
introduces a new approximation term due to the need in establishing high probability regret bounds. We generalize \citet{haghtalab2022oracle} to multi-class classification. 
By Lemma~\ref{lem: stability} (deferred after this proof), the stability term satisfies
\[
\sum_{n=1}^N \inner{g_n}{u_n^* - u_{n+1}^{**}}
\leq
N\sqrt{\frac{AK^2}{\poiparam \sigma}}.
\]
Similarly, we follow the proof idea of Lemma 4.6 of~\citet{haghtalab2022oracle} and bound the generalization error by Lemma~\ref{lem: generalization} (deferred after this proof):
\[
\sum_{n=1}^N \EE_{\incrdata_n}\sbr{ \inner{g_n^* - g_n}{ u_{n+1}^*}   }
\leq
N\sqrt{\frac{2A \ln(N)\ln(BN^2)}{\poiparam \sigma}} 
    + 
    \frac{\poiparam \sigma}{4A N \ln(N)}
    +
    NK e^{-\frac{\poiparam}{8K}}
\]
by our assumption that $\lambda \geq \frac{8 AK \ln(KN)}{\sigma}$, the last term $NK e^{-\frac{\poiparam}{8K}} \leq 1$.

In the following, we bound the approximation term. Before going into details, we first show that $g_n^*, u_{n+1}^{**}$ are functions of $(u_n, \incrdata_{1:n-1})$ and is thus
is $\Fcal^+_{n-1}$-measurable. Indeed, 
\[g_n^* = \rbr{\EE_{s \sim d_{\pi_n}} \sbr{I(h(s) \neq \pie(s))}}_{h \in \Bcal} = \EE_{\incrdata_n} \sbr{ \rbr{\EE_{s \sim \incrdata_n} \sbr{I(h(s) \neq \pie(s))}}_{h \in \Bcal} } =  \EE_{\incrdata_n} \sbr{g_n },\]
which is a function of $u_n$ .
Similarly, 
$$u_{n+1}^{**} =
\EE_{\incrdata_{n}} \sbr{ u_{n+1}^* } = 
\EE_{\incrdata_{n}}  \EE_{\pertdata_{n+1}}\sbr{\Onehot(\Ocal( \incrdata_{1:n-1} \cup \incrdata_{n} \cup \pertdata_{n+1})} 
$$
which is a function of $u_n$ and $\incrdata_{1:n-1}$ .


\textbf{Approximation Term:} Define $Y_n :=  
\inner{g_n}{u_{n+1}^{**} - u_{n+1}^*} -
\EE_{\incrdata_n}\sbr{ \inner{g_n^* - g_n}{ u_{n+1}^*}  }$.
%
In the following, we show that 
$\cbr{Y_n}_{n=1}^N$ is a martingle difference sequence with respect to filtration $\cbr{\Fcal_{n-1}^+}_{n=1}^{N}$, i.e.,
$\forall n \in \mathbb{N}$, $\EE\sbr{Y_n | \Fcal^+_{n-1}} = 0$.
First, 
\[
\EE\sbr{ \inner{g_n}{u_{n+1}^{**}} \mid \Fcal_{n-1}^+ }
= 
\inner{ \EE\sbr{g_n \mid \Fcal_{n-1}^+} }{ u_{n+1}^{**} }
= 
\inner{g_n^*}{ u_{n+1}^{**} }
= 
\inner{g_n^*}{ \EE_{\incrdata_n}[u_{n+1}^*] }
= 
\EE_{\incrdata_n} \sbr{ \inner{g_n^*}{u_{n+1}^*} }
\]
where the first equality is from that $u_{n+1}^{**}$ is $\Fcal_{n-1}^+$-measurable and linearity of expectation.
Second,
\[
\EE\sbr{ \inner{g_n}{u_{n+1}^{*}} \mid \Fcal_{n-1}^+ }
= 
\EE_{\incrdata_n} \sbr{ \inner{g_n}{u_{n+1}^*} },
\]
since conditioned on $\Fcal_{n-1}^+ = \sigma(u_1, \incrdata_1, \ldots, u_{n-1}, \incrdata_{n-1}, u_n)$, the only randomness in the expression $\inner{g_n}{u_{n+1}^*}$ comes from their dependence on $\incrdata_n$. 

Together we have,
\[
\begin{aligned}
 \EE\sbr{Y_n | \Fcal^+_{n-1}}
 =&
 \EE_{\incrdata_n} \sbr{ \inner{g_n^*}{u_{n+1}^*} } 
 -
 \EE_{\incrdata_n} \sbr{ \inner{g_n}{u_{n+1}^*} }
 -
  \EE\sbr{ \EE_{\incrdata_n}\sbr{ \inner{g_n^* - g_n}{ u_{n+1}^*}  }| \Fcal^+_{n-1}}\\
 =&
 \EE_{\incrdata_n}\sbr{\inner{g_n}{u_{n+1}^{**}} }
 -
 \EE_{\incrdata_n}\sbr{\inner{g_n}{u_{n+1}^*} } 
 -
 \EE_{\incrdata_n}\sbr{ \inner{g_n^*}{ u_{n+1}^*}  }
 +
 \EE_{\incrdata_n}\sbr{ \inner{g_n}{ u_{n+1}^*}  } \\
 =&
 0.
\end{aligned}
\]

Meanwhile, since each entry of $g_n$ and $g_n^*$ are upper-bounded by 1 and lower-bounded by 0, we have 
\[
|Y_n| \leq 
\|g_n\|_\infty \cdot \|u_{n+1}^{**}-u_{n+1}^*\|_1  
+
\EE_{\incrdata_n}\sbr{\|g_n-g_n^*\|_\infty \cdot \|u_{n+1}^{*}\|_1 } 
\leq 2 + 1
= 3.
\]
With the martingale difference sequence conditions satisfied, by Azuma-Hoeffding's inequality, for any $\delta \in (0,1]$, with probability $1-\delta/2$,
\[
\abs{ \sum_{n=1}^N \rbr{ \inner{g_n}{u_{n+1}^{**} - u_{n+1}^*} -
\EE_{\incrdata_n}\sbr{ \inner{g_n^* - g_n}{ u_{n+1}^*} } } }
=
\abs{ \sum_{n=1}^N Y_n }
\leq
3\sqrt{2N\ln(4/\delta)}.
\]
Combining bounds for three terms in equation~\eqref{eqn: stability_eneralization}, we conclude the proof.
\end{proof}

\begin{lemma}
 Under the notation of \algms, when $\poiparam \geq \frac{2AK^2}{\sigma}$, $\forall n \in \mathbb{N}$,  $g_n$,$u_n^*$,$u_{n+1}^{**}$ defined in equation~\eqref{eqn: g_n},~\eqref{eqn:u_star}, and~\eqref{eqn:u_star_star} satisfies
 \[
 \inner{g_n}{u_n^* - u_{n+1}^{**}} \leq 2\sqrt{\frac{AK^2}{\poiparam \sigma}} .
 \]
 \label{lem: stability}
\end{lemma}

\begin{proof}
 Since $\|g_n\|_\infty \leq 1$, it is straightforward to see that
 \[
  \inner{g_n}{u_n^* - u_{n+1}^{**}} 
  \leq 
  \|g_n\|_\infty \cdot \|u_n^* - u_{n+1}^{**}\|_1
  \leq
  \|u_n^* - u_{n+1}^{**}\|_1.
 \]
 
Our proof structure is similar to Lemma 4.4 and Lemma 4.5 of~\citet{haghtalab2022oracle}, where we bound $\|u_n^* - u_{n+1}^{**}\|_1$ by the discrepancy of distributions of two datasets. We generalize the results of~\citet{haghtalab2022oracle} to multiclass classification and online learning with batches of $K$ samples at each round to keep track of the number of copies of each $(s,a) \in \Scal \times \Acal$. 



The main technical challenge here lies in using batches of examples. While in the batch size 1 case, \citet{haghtalab2022oracle} 
reduced bounding $\| u_n^* - u_{n+1}^{**} \|$ to bounding the discrepancy between an $SA$-dimensional product Poisson distribution and its one-sample shifted version, the same approach becomes difficult to compute when dealing with batches of more than one examples. Specifically, a straightforward calculation leads to the total variation (TV) distance between an $SA$-dimensional product Poisson distribution and a mixture of product shifted Poisson distributions, where the shifts are drawn from a multinomial distribution. This mixture significantly complicates the computation, making it a much harder to solve and present.
\footnote{
Concretely, we can define $p_{n,e}$ to be a $SA$ dimensional random variable
that represents the ``histogram'' of all examples in $\cup_{i=1}^{n-1}D_i \cup Q_{n,e}$; specifically, 
$p_{n,e}(s,a) = \sum_{i=1}^{n-1} \sum_{(s', a') \in D_i} I(s = s', a = a') + \sum_{(s',a') \in Q_{n,e}} I(s = s', a = a')$. 
By the data-processing inequality, $\| u_n^* - u_{n+1}^{**} \|_1$ is upper bounded by the TV distance between $(p_{n,e}(s,a))_{s \in \Scal, a \in \Acal} \mid \Fcal_{n-1}$ and $(p_{n+1,e}(s,a))_{s \in \Scal, a \in \Acal} \mid \Fcal_{n-1}^+$,
which is equal to the TV distance between a product Poisson and a mixture of product shifted Poisson distribution.
}


We work around this challenge by further dividing dataset $D$ into $K$ groups by the arrival index $k \in [K]$ within batch. 
We denote the (singleton) dataset that contains the $k^{\text{th}}$ draw in $\incrdata_n$ by $\incrdata_{n,k}$ and the union of $\incrdata_{i,k}$ for $i=\{1,2,\cdots,n\}$ as $\cup_{i=1}^n \incrdata_{i,k}$. The perturbation samples are treated similarly: we partition $\pertdata_{n,\ensem}$ to $K$ groups, associating a group index (an auxiliary random variable) $\tilde{k}_{n,\ensem,x} \sim \Unif([K])$ to each perturbation example $(\tilde{s}_{n,\ensem,x},\tilde{a}_{n,\ensem,x})$. 

Specifically, for $n \in [N]$, we define a $S \cdot A \cdot K$-dimensional random variable $p_{n,\ensem}$. The role of $p_{n,\ensem}(s,a,k)$ is to count the occurrences $(s,a)$ within 
 $\cup_{i=1}^n \incrdata_{i,k}$ 
as well as the $k$-th subgroup of $\pertdata_{n,\ensem}$.
Its formal definition is as follows:
\begin{equation}
    p_{n,\ensem}(s,a,k) := \sum_{i=1}^{n-1} I(s = s_{i,k},a = \pie(s_{i,k})) 
    + \sum_{x=1}^{\poi_{n,\ensem}} I(s = \tilde{s}_{n,\ensem,x}, a = \tilde{a}_{n,\ensem,x}, k = \tilde{k}_{n,\ensem,x}) .
    \label{eqn: data_distribution}
\end{equation}

By recalling the definition $u_{n,\ensem}$ in \algms, we rewrite $u_{n,\ensem}$ as a function of $p_{n,\ensem}$:
\[
\begin{aligned}
u_{n,\ensem} =& \Onehot(\Ocal( \incrdata_{1:n-1} \cup \pertdata_{n,\ensem}) )\\
=&
\Onehot(
 \argmin_{h \in \Bcal} \EE_{(s,a) \sim \incrdata_{1:n-1} \cup \pertdata_{n,\ensem}} \sbr{  I(h(s) \neq a) }  )\\
=&
\Onehot(
 \argmin_{h \in \Bcal} \sum_{(s,a) \in \incrdata_{1:n-1} \cup \pertdata_{n,\ensem}} \sbr{  I(h(s) \neq a) } )\\
=&
\Onehot(
 \argmin_{h \in \Bcal} \sum_{(s,a,k) \in \Scal \times \Acal \times [K]} p_{n,\ensem}(s,a,k) \sbr{  I(h(s) \neq a) } ).
\end{aligned}
\]
Observe that $u_{n}^* = \EE_{\pertdata_{n,\ensem}}\sbr{u_{n,\ensem}}$ and
$u_{n+1}^{**} = \EE_{\incrdata_{n}}\EE_{\pertdata_{n+1,\ensem}}\sbr{u_{n+1,\ensem}}$ can also be viewed as the conditional distributions of $h_{n,\ensem} \mid \Fcal_{n-1}$ and $h_{n+1,\ensem} \mid \Fcal^+_{n-1}$, respectively. 
We define $\distp_n(\cdot|\Fcal_{n-1})$ as the conditional distribution of $p_{n,\ensem} \mid \Fcal_{n-1}$ and define $\distp_{n+1}(\cdot|\Fcal^+_{n-1})$ represent the conditional distribution of $p_{n+1,\ensem} \mid \Fcal^+_{n-1}$.
By applying data-processing inequality~\cite{beaudry2011intuitive}, we obtain
\[
\|u_n^* - u_{n+1}^{**}\|_1
= 
2 \tv( u_n^*, u_{n+1}^{**} )
\leq
2 \tv(\distp_{n}(\cdot|\Fcal_{n-1}),\distp_{n+1}(\cdot|\Fcal^+_{n-1})).
\]
Note that $\distp_n(\cdot|\Fcal_{n-1})$ and $\distp_{n+1}(\cdot|\Fcal^+_{n-1})$ depend on the same historical dataset $\incrdata_{1:n-1}$, we further define 
\begin{align}
    &q_{n,\ensem}(s,a,k) := \sum_{x=1}^{\poi} I(s = \tilde{s}_{n,\ensem,x}, a = \tilde{a}_{n,\ensem,x}, k = \tilde{k}_{n,\ensem,x}), \\
    &r_{n,\ensem}(s,a,k) := I(s = s_{n-1,k},a = \pie(s_{n-1,k})) +
    \sum_{x=1}^{\poi} I(s = \tilde{s}_{n,\ensem,x}, a = \tilde{a}_{n,\ensem,x}, k = \tilde{k}_{n,\ensem,x}),
    \label{eqn: data_difference_distribution}
\end{align}
It is not hard to see that following the definition in equation~\eqref{eqn: data_distribution},
\[
\begin{aligned}
 p_{n,\ensem} 
 =&
 q_{n,\ensem} + \rbr{ \sum_{i=1}^{n-1} I\rbr{ s = s_{i,k},a = \pie(s_{i,k})} }_{(s,a,k) \in \Scal \times \Acal \times [K]},  \\
 p_{n+1,\ensem} 
 =&
 r_{n+1,\ensem} + \rbr{ \sum_{i=1}^{n-1} I \rbr{ s = s_{i,k},a = \pie(s_{i,k})} }_{(s,a,k)\in \Scal \times \Acal \times [K]}
 \end{aligned}
\]
Notice that the distribution of $q_{n,\ensem} \mid \Fcal_{n-1}$ is 
independent of both $n$ and $\ensem$. 
Indeed, $q_{n,\ensem}$ is a function of $\pertdata_{n,\ensem}$ and random variables $(\tilde{k}_{n,\ensem,x})_{x =1}^{\poi_{n,\ensem}}$. By the subsampling property of Poisson distribution, we can view each entry of $q_{n,\ensem}$ as a independent Poisson random variable, following $q_{n,\ensem}(s,a,k) \sim \poisson(\poipp(s))$, 
where $\poipp(s) := \frac{ \poiparam \based(s)}{AK}$. 
Therefore, $q_{n,e} \mid \Fcal_{n-1}$ is drawn from a product of Poisson distributions:
\[
\prod_{s \in \Scal} \prod_{a \in \Acal} \prod_{k \in [K]} 
\poisson(q(s,a,k);\poipp(s))
=: \distq(q), 
\]
where $\poisson(q, \lambda) = e^{-\lambda} \frac{\lambda^q}{q!} I(q \in \NN)$ denotes the probability mass function for Poisson distribution.

Therefore, in subsequent proofs,  
we denote the distribution of $q_{n,\ensem} \mid \Fcal_{n-1}$ by  $\distq$ for simplicity. Meanwhile, since the conditional distribution of $r_{n+1,\ensem} \mid \Fcal^+_{n-1}$ is constant over all $e$
, we denote the conditional distribution of $r_{n+1,\ensem} \mid \Fcal^+_{n-1}$ by $\distr_{n+1,e}(\cdot|\Fcal^+_{n-1})$, and use the notation $\distr_{n+1}$ for simplicity when it is clear from the context. 
Here, we apply the translation invariance property of $\tv$ distance and obtain
\[
\tv(\distp_{n}(\cdot|\Fcal_{n-1}),\distp_{n+1}(\cdot|\Fcal^+_{n-1}))
= 
\tv(\distq(\cdot), \distr_{n+1,e}(\cdot|\Fcal^+_{n-1}))
=
\tv(\distq, \distr_{n+1}).
\] 



Next, we rewrite $\distr_{n+1}$ by the tower property: 
\[
\begin{aligned}
\distr_{n+1,e}(r|\Fcal^+_{n-1})
=&
\PP( r_{n+1,e} = r \mid \Fcal^+_{n-1} )\\
=&
\EE\sbr{ \PP( r_{n+1,e} = r \mid \incrdata_n, \Fcal^+_{n-1} ) \mid \Fcal^+_{n-1} } \\
=&
\EE\sbr{ \distr_{n+1}(r|\incrdata_{n},\Fcal^+_{n-1}) \mid \Fcal^+_{n-1} }
= 
\EE_{\incrdata_n} \sbr{ \distr_{n+1}(r \mid \incrdata_n,\Fcal^+_{n-1}) }.
\end{aligned}
\]

Now, it suffices to bound 
$\tv(\distq, \EE_{\incrdata_n}\sbr{\distr_{n+1}(\cdot | \incrdata_n,\Fcal^+_{n-1})})$
, 
which we bound in a way similar to bounding $\tv(P, Q)$ in Section 4.2.1 of \cite{haghtalab2022oracle}.
By this observation, we have 

\begin{equation}
\begin{aligned}
\inner{g_n}{u_n^* - u_{n+1}^{**}}
\leq &
2 \tv(\distq, \EE_{\incrdata_n}\sbr{\distr_{n+1}(\cdot | \incrdata_n,\Fcal^+_{n-1})})\\
\leq &
\sqrt{2 \chi^2\rbr{\EE_{\incrdata_n}\sbr{\distr_{n+1}(\cdot | \incrdata_n,\Fcal^+_{n-1})}, \distq}}\\
= &
\sqrt{2 \rbr{
 \EE_{\incrdata_n, \incrdata'_n} \sbr{
 \EE_{q \sim \distq} \sbr{\frac{\distr_{n+1}(q|\incrdata_n,\Fcal^+_{n-1}) \cdot \distr_{n+1}(q|\incrdata_n',\Fcal^+_{n-1})}{\distq(q)^2}}}
-1} },
\end{aligned}
\label{eqn: tvbound}
\end{equation}

where we apply similar technique in Section 4.2.1 of~\citet{haghtalab2022oracle} by using $\chi^2$ distance (Lemma E.1 of~\citet{haghtalab2022oracle}) and Ingster's method (Lemma E.2 of~\citet{haghtalab2022oracle}). Note that all examples in $\incrdata_n=\cbr{ (s_{n,k}, \pie(s_{n,k})) }_{k = 1}^K$, $\incrdata_n' = \cbr{ (s'_{n,k}, \pie(s'_{n,k})) }_{k = 1}^K$ are i.i.d. drawn from $\incrdist_{\pi_n}$.

Recall that $\pertdata_{n,\ensem}{\; \buildrel d \over = \;} \pertdata_{n+1,\ensem}$, meaning the difference between the distributions of $q_{n,\ensem}$ and $r_{n+1,\ensem}$ is induced by the $K$ examples from $\incrdata_n$. Conditioned on $\incrdata_n = \cbr{ (s_{n,k}, \pie(s_{n,k})) }_{k = 1}^K$, we have
\[
q_{n,\ensem}
 {\; \buildrel d \over = \;}
  r_{n+1,\ensem} 
  -
  \rbr{ I(s = s_{n,k},a =\pie(s_{n,k}))}_{(s,a,k)\in \Scal \times \Acal \times [K]}.
\]

This allows us to 
write the probability mass function of $\distr_{n+1}(\cdot|\incrdata_{n},\Fcal^+_{n-1})$ as: 
\[
\distr_{n+1}(r|\incrdata_{n},\Fcal^+_{n-1}) 
=
\prod_{s \in \Scal} \prod_{a \in \Acal} \prod_{k\in [K]} 
\poisson\rbr{ r(s,a,k)-I \rbr{ s=s_{n,k}, a=\pie(s_{n,k})};\poipp(s) } .
\]
 Let $\incrdata_n, \incrdata_n'$ 
 be the pair of datasets of size $K$ in equation~\eqref{eqn: tvbound}, some algebraic calculations yield that
\[
\frac{\distr_{n+1}(q|\incrdata_n,\Fcal^+_{n-1}) \cdot \distr_{n+1}(q|\incrdata_n',\Fcal^+_{n-1})}{\distq(q)^2}
=
 \prod_{k \in [K]} \frac{q\rbr{s_{n,k},\pie(s_{n,k}),k}}{\poipp(s_{n,k})} \cdot \frac{q\rbr{s'_{n,k},\pie(s'_{n,k}),k'} }{\poipp(s'_{n,k})}.
\]
By taking expectation over $q \sim \distq$ we have
\[
\EE_{q \sim \distq} \sbr{\frac{\distr_{n+1}(q|\incrdata_n,\Fcal^+_{n-1}) \cdot \distr_{n+1}(q|\incrdata_n',\Fcal^+_{n-1})}{\distq(q)^2}}
=
\prod_{k \in [K]} \rbr{ 1 + \frac{I(s_{n,k} = s'_{n,k})}{\poipp(s_{n,k})} }.
\]
Furthermore, we take conditional expectation with respect to $\incrdata_n, \incrdata_n'$ and obtain

\begin{equation}
\begin{aligned}
\EE_{\incrdata_n, \incrdata'_n} \sbr{ \EE_{q \sim \distq} \sbr{\frac{\distr_{n+1}(q|\incrdata_n,\Fcal^+_{n-1}) \cdot \distr_{n+1}(q|\incrdata_n',\Fcal^+_{n-1})}{\distq(q)^2}} }
= &
\prod_{k \in [K]} \rbr{ 1 + \frac{AK}{\poiparam} \sum_{s \in \Scal} \frac{ d_{\pi_n}(s)^2 }{\based(s)}   } \\
\leq &
\rbr{ 1 + \frac{AK}{\poiparam} \sum_{s \in \Scal} \frac{ d_{\pi_n}(s) }{\smooth} }^K
=
\rbr{ 1 + \frac{AK}{\poiparam \smooth} }^K
,
\end{aligned}
\label{eqn: ingster}
\end{equation}
where in the above derivation, we recall that $\poipp(s) := \frac{ \poiparam \based(s)}{AK}$, 
$\EE_{\incrdata_n, \incrdata'_n}\sbr{\frac{I(s_{n,k} = s'_{n,k})}{\based(s_{n,k})}} =  \sum_{s \in \Scal} \frac{d_{\pi_n}(s)^2}{\based(s)}$, 
and  $\frac{d_{\pi}(s)}{\based(s)}\leq \frac{1}{\smooth}$.

Finally, we conclude the proof by plugging equation~\eqref{eqn: ingster} into equation~\eqref{eqn: tvbound} and setting $\poiparam \geq \frac{2AK^2}{\sigma}$:
\[
\inner{g_n}{u_n^* - u_{n+1}^{**}}
\leq 
\sqrt{2  \rbr{ \rbr{ 1 + \frac{AK}{\poiparam \sigma } }^K -1 }}
\leq 
2K\sqrt{ \frac{A}{\poiparam \sigma }  },
\]
which is due to $\forall x \in [0,\frac{1}{2}]$, $ 1+x \leq e^x \leq 1+2x$, meaning when $\frac{AK}{\poiparam \sigma } \leq \frac{AK^2}{\poiparam \sigma } \leq \frac 1 2$,
\[
\rbr{ 1 + \frac{AK}{\poiparam \sigma } }^K  \leq \rbr{ \exp\rbr{\frac{AK}{\poiparam \sigma }}}^K = \exp\rbr{\frac{AK^2}{\poiparam \sigma }} \leq 1 + \frac{2AK^2}{\poiparam \sigma }.
\]

\end{proof}
To simplify the expression in the following proofs, we introduce shorthand $z_{n,k} := (s_{n,k},\pie(s_{n,k}))$ and $\tilde{z}_{n,\ensem,x} := (\tilde{s}_{n,\ensem,x},\tilde{a}_{n,\ensem,x}))$. By definitions in \algms, $z_{n,k} \sim \incrdist_{\pi_n}$, $\tilde{z}_{n,\ensem,x} \sim \incrdistbase$.
We provide a generalized coupling lemma similar to~\cite{haghtalab2022oracle, haghtalab2022smoothed}, showing that multiple draws from the covering distribution can be seen as containing a batch of examples from a covering distribution with high probability.

\begin{lemma}[Generalized coupling] 
Let $G \in \mathbb{N}$ and $\tilde{z}_1, \cdots, \tilde{z}_G \sim \incrdistbase$. For all $\pi$ that satisfies $\forall s \in \Scal$, $\frac{d_{\pi}(s)}{\based(s)}\leq \frac{1}{\smooth}$.
By some external randomness $R$, there exists an index $\Ical= \Ical\left(\tilde{z}_1, \cdots, \tilde{z}_G, R\right) \in[G]$ and a success event $U=U\left(\tilde{z}_1, \cdots, \tilde{z}_G, R\right)$ such that $\prob\left[U^c\right] \leq(1-\sigma/A )^G$, and
$$
\left(
\tilde{z}_{\Ical} \mid U, \tilde{z}_{\backslash \Ical}
\right) 
\sim \incrdist_{\pi},
$$
where $\tilde{z}_{\backslash \Ical}$ denotes $\cbr{\tilde{z}_1, \cdots, \tilde{z}_G}\backslash \cbr{\tilde{z}_{\Ical}} $.
\label{coro: external_randomness}
\end{lemma}

\begin{proof}
    For all $\pi$ that satisfies $\forall s \in \Scal$, $\frac{d_{\pi}(s)}{\based(s)}\leq \frac{1}{\smooth}$ and its corresponding $\incrdist_{\pi}$, we have $z = \rbr{s, \pie(s)} \sim \incrdist_{\pi}$, following $s \sim d_{\pi}$, $a = \pie(s)$. Since $\tilde{z} = (\tilde{s}, \tilde{a}) \sim \incrdistbase$ follows $\tilde{s} \sim \based$, $\tilde{a} \sim \Unif(\Acal)$. By their definition It is straight forward to see 
    \[
      \frac{\incrdist_{\pi}(z)}{\incrdistbase(z)} 
      =
      \frac{d_{\pi}(s)}{\based(s)} \cdot
      \frac{1}{\Unif(a ; \Acal)}
      \leq 
      \frac{A}{\sigma}.
    \]
    We conclude the proof by 
    letting $X = \incrdist_{\pi}$ and $Y = \incrdistbase$ in
    Lemma 4.7 of~\citet{haghtalab2022oracle}.
\end{proof}

\begin{lemma}
    \label{lem: generalization}
$\forall n \in \{1, \cdots, N\}$,  \algms with $\poiparam \geq \frac{4AK \ln(N)}{\sigma}$ achieves 
    \[
    \EE_{\incrdata_{n}}\sbr{ \inner{g_{n}^* - g_{n}}{ u_{n+1}^*}  } 
    \leq
    \sqrt{\frac{2A \ln(N)\ln(BN^2)}{\poiparam \sigma}} 
    + 
    \frac{\poiparam \sigma}{4A N^2 \ln(N)}
    +
    K e^{-\frac{\poiparam}{8K}}.
    \]
\end{lemma}


\begin{remark}
When applying this lemma downstream, we will treat the last two terms as lower order terms:  
First, our final setting of $\lambda$ will be such that $\lambda = O(N)$, in which case
$\frac{\poiparam \sigma}{4A N^2 \ln(N)}
\leq 
O(\frac{1}{N})
$
; 
Second, we will focus on the regime that $\poiparam \geq \frac{8 AK \ln(N)}{\sigma} \geq 8K \ln N$, in which case
$K e^{-\frac{\poiparam}{8K}} =O(\frac{K}{N})$.  

\end{remark}

\begin{proof}
To begin with, we first 
rewrite 
\[
 \EE_{\incrdata_{n}}\sbr{ \inner{g_{n}^* - g_{n}}{ u_{n+1}^*}  } 
 =
 \EE_{\incrdata_{n}} \EE_{\pertdata_{n+1, \ensem}} \sbr{ \inner{g_{n}^* - g_{n}}{ u_{n+1,\ensem}}  }
 .
\]

    Following the same method in Lemma~\ref{lem: stability},
    for the $\ensem^{\text{th}}$ perturbation set $\pertdata_{n+1,\ensem}$ at round $n$, we divide it to $K$ subsets $(\pertdata_{n+1,\ensem,k})_{k=1}^K$ by randomly assigning $\tilde{k}_{n+1,\ensem,x} \sim \Unif([K])$ to each example $(\tilde{s}_{n+1,\ensem,x},\tilde{a}_{n+1,\ensem,x})$ in $\pertdata_{n+1,\ensem}$ for $x \in \poi_{n+1,\ensem}$. Note that the divisions in this proof is only for analytical use.
    
    By the subsampling property of Poisson distribution, we have the size of each $\pertdata_{n+1,\ensem,k}$ denoted by $\poi_{n+1,\ensem,k}$ follows $\poisson(\poiparam/K)$. For notational simplicity, we use $\poi_{k}$ to denote $\poi_{n+1,\ensem,k}$, $\pertdata_k = (\tilde{s}_{x},\tilde{a}_{x})_{x = 1}^{\poi_{k}} = (\tilde{z}_x)_{x = 1}^{\poi_{k}}$ to denote $\pertdata_{n+1,\ensem,k}$, and 
    $z_k = (s_k,\pie(s_k))$ to denote $z_{n+1,k}
    $
    in the following proof. 
    
    By our assumption that $\poiparam \geq \frac{4AK \ln(N)}{\sigma}$, without loss of generality, for now we assume $ \frac{\poiparam}{2K}$ to be integral multiple of $G := \lceil \frac{2 A \ln(N)}{\sigma} \rceil $, meaning $\frac{\poiparam}{2K} = \group G$ for some $\group \in \mathbb{N}$. By defining event $\tilde{U}_k := \cbr{ \poi_k \geq \frac{\poiparam}{2K} }$ and $\tilde{U} := \cap_{k \in [K]} \tilde{U}_k$, we have 
    \[
    \prob(\tilde{U}_k^c) 
    =
    \prob(\poi_k < \frac{\poiparam}{2K})
    \leq 
    \exp(-\frac{\poiparam}{8K}), \;
     \prob(\tilde{U}^c)
     \leq
    \sum_{k = 1}^K \prob(\tilde{U}_k^c)
     \leq 
     K e^{-\frac{\poiparam}{8K}},
    \]
    where we use the fact that for $X \sim \poisson(\poiparam' = \poiparam/ K)$, $\prob(X < \poiparam'/2) \leq \exp(- \poiparam'/8)$, and apply union bound for $\tilde{U}^c$.
    
    
    At round $n$, conditioned on the favorable event $\tilde{U}$ happening, 
    we 
    further divide $\pertdata_k$ into $\group
    $ groups denoted by $\pertdata_{k,m}$ for $m \in [\group]$, where each group has size greater or equal to $G$. 
    
    Conditioned on $\tilde{U}$, we apply Lemma~\ref{coro: external_randomness} to each $\pertdata_{k,m}$ with distribution $\incrdist_{\pi_n}$ (induced by $\pi_n$), obtaining $\group$ independent events 
    $U_{k,m}$ for $m \in [\group]$, where 
    \[
    \prob(U_{k,m}^c | \tilde{U})
    \leq
    (1-\sigma/A)^G
    =
    (1-\frac{\sigma}{A})^{\frac{A}{\sigma}\cdot 2\ln(N)}
    \leq
    e^{-2 \ln(N)}
    =
    N^{-2}.
    \]
    Conditioned on $U_{k,m}$, there exist an element $\elem_{k,m} \in \pertdata_{k,m}$
    such that
    \[
    (\elem_{k,m} | U_{k,m},  \pertdata_{k,m} \backslash \cbr{ \elem_{k,m} } ) 
    \sim
    \incrdist_{\pi_n}.
    \]
    Define event
    $U := \cap_{k \in [K], m \in [\group]} U_{k,m}$ to be the intersection of those independent events (at round $n$), where by union bound and the definition of $\group$ we have 
    \[
    \prob(U^c | \tilde{U}) \leq \frac{K\group}{N^2} 
    \leq 
    \frac{K}{N^2} \cdot \frac{\poiparam}{2K \lceil \frac{2 A \ln(N)}{\sigma} \rceil}
    \leq
    \frac{\poiparam \sigma}{4A N^2 \ln(N)}.
    \]
    Now we introduce shorthand
    $\elemc_{k,m} :=  \pertdata_{k,m} \backslash \cbr{ \elem_{k,m} }$, $\elemc := \cup_{k \in [K], m \in [\group]} \elemc_{k,m}$
    and write
    \[
    (z_{n,1},  \cdots, z_{n,K}, \elem_{1,1}, \cdots, \elem_{1,\group}, \elem_{2,1}, \cdots, \elem_{K,\group}
    |  \elemc, U, \tilde{U}, \Fcal^+_{n-1} )
    {\; \buildrel \text{i.i.d.} \over \sim \;}
    \incrdist_{\pi_n},
    \]







    
    which is by the independence between each group as well as the samples from $\incrdist_{\pi_n}$.
    
    Now, we can split the generalization error into three terms by the law of total expectation: 
    \begin{equation}
    \begin{aligned}
    \EE_{\incrdata_{n}, \pertdata_{n+1,\ensem}} \sbr{ \inner{g_{n}^* - g_{n}}{ u_{n+1,\ensem} }  } 
    = &
    \prob(U \cap \tilde{U}) \cdot \EE_{\incrdata_n, \pertdata_{n+1,\ensem}}\sbr{ \inner{g_n^* - g_n}{ u_{n+1,\ensem}} | U, \tilde{U} }
    + 
    \prob(U^c \cup \tilde{U}^c)
    \\
    \leq &
    \EE_{\incrdata_n, \pertdata_{n+1,\ensem}}\sbr{ \inner{g_n^* - g_n}{ u_{n+1,\ensem}} | U, \tilde{U} }  
    +
    \prob( U^c \cap \tilde{U} ) 
    + 
    \prob(\tilde{U}^c )\\
    \leq  &
    \EE_{\incrdata_n, \pertdata_{n+1,\ensem}}\sbr{ \inner{g_n^* - g_n}{ u_{n+1,\ensem}} | U, \tilde{U} }  
    +
    \frac{\poiparam \sigma}{4A N^2 \ln(N)}
    +
    K e^{-\frac{\poiparam}{8K}},
    \end{aligned}
    \label{eqn: gen_split}
    \end{equation}
    where we apply the fact that $\inner{g_n^* - g_n}{ u_{n+1}^*} \leq \|g_n^* - g_n\|_{\infty} \cdot \|u_{n+1}^*\|_1 \leq 1$, 
    and bring in the bounds for $\prob(\tilde{U}^c)$ and $\prob(U^c | \tilde{U})$ shown above.
    
    For the remaining term $\EE_{\incrdata_n, \pertdata_{n+1,\ensem}}\sbr{ \inner{g_n^* - g_n}{ u_{n+1,\ensem}} | U , \tilde{U}} $, we abbreviate it as 
    $\EE_{\incrdata_n, \pertdata_{n+1,\ensem} | U, \tilde{U} } \sbr{ \inner{g_n^* - g_n}{ u_{n+1,\ensem}}} $ 
    and split it by the linearity of expectation
    \[
    \EE_{\incrdata_n, \pertdata_{n+1,\ensem} | U, \tilde{U}} \sbr{ \inner{g_n^* - g_n}{ u_{n+1,\ensem}}} 
    =
    \underbrace{\EE_{\incrdata_n, \pertdata_{n+1,\ensem} | U, \tilde{U}} \sbr{ \inner{g_n^* }{ u_{n+1,\ensem}}} }_{\textbf{I}}
    -
    \underbrace{\EE_{\incrdata_n, \pertdata_{n+1,\ensem} | U, \tilde{U}} \sbr{ \inner{g_n}{ u_{n+1,\ensem}}} }_{\textbf{II}}.
    \]
    We first focus on term II. For now, we abbreviate $\pertdata_{n+1,\ensem}$ as $\pertdata_{n+1}$ when it is clear from the context.
    By introducing shorthand of $h_{n+1} = \Ocal(\cup_{i = 1}^{n} \incrdata_{i} \cup \pertdata_{n+1})$ corresponding to the only policy in the support of $u_{n+1,\ensem}$, and denote $\ell(h_{n+1},(s,a)) := I(h_{n+1}(s) \neq a)$, we rewrite \textbf{II} as
\[
\begin{aligned}
\EE_{\incrdata_n, \pertdata_{n+1} | U, \tilde{U}} \sbr{ \inner{g_n}{ u_{n+1,\ensem}}} 
 =&
 \EE_{\incrdata_n, \pertdata_{n+1} | U, \tilde{U}} \sbr{   \frac{1}{K} \sum_{k = 1}^K I(h_{n+1}(z_{n,k}) \neq \pie(z_{n,k}))   } \\
 =&
 \frac{1}{K} \sum_{k = 1}^K \EE_{\incrdata_n, \pertdata_{n+1} | U, \tilde{U}} \sbr{ \ell(h_{n+1},z_{n,k}) }. 
 \end{aligned}
\]
Here we further denote
$\elem =  \{\elem_{k,m}\}_{k \in [K], m \in [\group]}$. With this notation, 
$\pertdata_{n+1} = \elem \cup \elemc$.  The following holds for all 
$m \in [\group]$:
\begin{equation}
    \begin{aligned}
        \textbf{II} 
        =&
        \frac{1}{K} \sum_{k = 1}^K \EE_{\elemc | U, \tilde{U}, \Fcal_{n-1}^+}\EE
        \sbr{ \ell(h_{n+1},z_{n,k}) 
        | \elemc, U, \tilde{U}, \Fcal^+_{n-1} }\\
        =&
        \frac{1}{K} \sum_{k = 1}^K \EE_{\elemc | U, \tilde{U}, \Fcal_{n-1}^+}\EE
        \sbr{ \ell( \Ocal(\cup_{i = 1}^{n-1} \incrdata_{i} 
         \cup
         \incrdata_n 
         \cup
         \pertdata_{n+1} ))
        ,z_{n,k}) 
        | \elemc, U, \tilde{U}, \Fcal^+_{n-1} }\\
        =&
        \frac{1}{K} \sum_{k = 1}^K \EE_{\elemc | U, \tilde{U}, \Fcal_{n-1}^+}\EE
        \sbr{ \ell(
         \Ocal(\cup_{i = 1}^{n-1} \incrdata_{i} 
         \cup
         (\incrdata_n \backslash \cbr{z_{n,k}}) 
         \cup
         (\pertdata_{n+1} \backslash \cbr{\elem_{k,m}})
         \cup
         \{z_{n,k}\} \cup \{\elem_{k,m}\}))
        ,z_{n,k}) 
        | \elemc, U, \tilde{U}, \Fcal^+_{n-1}
        }\\
        =&
        \frac{1}{K} \sum_{k = 1}^K \EE_{\elemc | U, \tilde{U}, \Fcal_{n-1}^+}\EE
        \sbr{ \ell(
         \Ocal(\cup_{i = 1}^{n-1} \incrdata_{i} 
         \cup
         (\incrdata_n \backslash \cbr{z_{n,k}}) 
         \cup
         (\pertdata_{n+1} \setminus \cbr{\elem_{k,m}})
         \cup
         \{z_{n,k}\} \cup \{\elem_{k,m}\}))
        ,\elem_{k,m}) 
        | \elemc, U, \tilde{U}, \Fcal^+_{n-1}
        }\\
        =&
        \frac{1}{K} \sum_{k = 1}^K \EE_{\incrdata_n, \pertdata_{n+1} | U, \tilde{U}}
        \sbr{ \ell(h_{n+1},\elem_{k,m}) 
        }
        ,
    \end{aligned}
\end{equation}
where in the fourth equality we apply the independence between samples and exchangeability between $z_{n,k}$ and $\elem_{k,m}$ conditioned on $U$, $\tilde{U}$, $\Fcal_{n-1}^+$ and $\elemc$. By slightly abusing notations and denoting $z_{n,k}$ as $\elem_{k,0}$, this implies:
\[\textbf{II} 
=
\frac{1}{\group+1}\sum_{m = 0}^\group \rbr{\frac{1}{K} \sum_{k = 1}^K \EE_{\incrdata_n, \pertdata_{n+1} | U, \tilde{U}}
        \sbr{ \ell(h_{n+1},\elem_{k,m}) 
        } }
=
 \EE_{\incrdata_n, \pertdata_{n+1} | U, \tilde{U}}
    \sbr{  \frac{1}{K(\group+1)}   \sum_{k = 1}^K \sum_{m = 0}^\group  \ell(h_{n+1},\elem_{k,m}) }.
\]
Meanwhile, we denote $z \sim \incrdist_{\pi_n}$
and rewrite \textbf{I} with the same $h_{n+1}$ notation:
\[\textbf{I} 
=
 \EE_{\incrdata_n, \pertdata_{n+1} | U, \tilde{U}}
    \sbr{    
    \EE_{z \sim \incrdist_{\pi_n}} \sbr{\ell(h_{n+1},z) }}.
\]

Combining what we have, we finally conclude that $\textbf{I}-\textbf{II}$ is bounded by
\begin{equation}
\begin{aligned}
\textbf{I}-\textbf{II}
= &
\EE_{\incrdata_n, \pertdata_{n+1} | U, \tilde{U}}
    \sbr{  
    \EE_{z \sim \incrdist_{\pi_n}} \sbr{\ell(h_{n+1},z) }
    -
     \frac{1}{K(\group+1)} \sum_{k = 1}^K \sum_{m = 0}^\group  \ell(h_{n+1},\elem_{k,m})
    }\\
= &
    \EE_{\xi \mid U, \tilde{U}, \Fcal_{n-1}^+}
    \EE \sbr{  
    \EE_{z \sim \incrdist_{\pi_n}} \sbr{\ell(h_{n+1},z) }
    -
     \frac{1}{K(\group+1)} \sum_{k = 1}^K \sum_{m = 0}^\group  \ell(h_{n+1},\elem_{k,m})
     \mid 
     \xi, U, \tilde{U}, \Fcal_{n-1}^+
    }
\\
\leq &
 \EE_{\xi \mid U, \tilde{U}, \Fcal_{n-1}^+}
    \EE \sbr{  
    \sup_{h \in \Bcal} \rbr{
    \EE_{z \sim \incrdist_{\pi_n}} \sbr{\ell(h ,z) }
    -
     \frac{1}{K(\group+1)} \sum_{k = 1}^K \sum_{m = 0}^\group  \ell(h,\elem_{k,m}) }
     \mid 
     \xi, U, \tilde{U}, \Fcal_{n-1}^+
     }\\
\leq &
\sqrt{\frac{\ln(B)}{2K(\group+1)}}
\leq
\sqrt{\frac{2A \ln(N)\ln(B)}{\poiparam \sigma}},
\end{aligned}
\label{eqn: generalization bound}
\end{equation}
where 
in the second equality, we use the law of iterated expectations; in the first inequality, we upper bound the random variable of interest
by the supremum over the policy class, since $h_{n+1} \in \Bcal$.
The second inequality is from Massart's Lemma (Lemma~\ref{lem:massart}).

We finish the proof by bringing equation~\eqref{eqn: generalization bound} into equation~\eqref{eqn: gen_split}.
\end{proof}




\begin{theorem}
\label{thm: mftpl_poisson_main}
For any $\delta \in (0,1]$, \algms with any 
$\poiparam \geq \max\cbr{ \frac{2AK^2}{\sigma}, \frac{8 AK \ln(KN)}{\sigma} }$
and $\Ensem = N A \ln (NS)$ 
outputs $\{\pi_{n}\}_{n=1}^N$
that satisfies that with probability at least $1-\delta$,
\begin{equation}
\label{eqn:regret-before-tuning-lambda}
\Reg(N)
\leq 
\tilde{O}
\rbr{
\sqrt{\frac{\lambda \ln B}{K^2}}
+
N \sqrt{ \frac{A K^2}{\lambda \sigma} }
+
N \sqrt{ \frac{A \ln B}{\lambda \sigma} }
+
\frac{\lambda \sigma}{A N}
+ \sqrt{N \ln \frac{1}{\delta}}
}.
\end{equation}
Specifically, if 
$N \geq \tilde{O}\rbr{ \sqrt{\frac{A}{\sigma}} \sqrt{ \min(\ln B, K^2)} \vee \frac{K^2}{A} \vee \frac{K^4}{A \ln B} }$, 
setting 
$\poiparam = \Theta\rbr{ N K \sqrt{\frac{A}{\sigma}} + N K^2 \sqrt{ \frac{A}{\sigma \ln B} }}$ gives 



\begin{equation}
\label{eqn:regret-after-tuning-lambda}
\Reg(N)
\leq
\tilde{O} \rbr{ 
\sqrt{N} \rbr{\frac{A (\ln B)^2}{\sigma K^2}}^{\frac14}
+
\sqrt{N} \rbr{\frac{A \ln B}{\sigma}}^{\frac14}
+
\sqrt{N \ln(1/\delta) } }.
\end{equation}
\end{theorem}

\begin{proof}
Fix $\delta \in (0,1)$. By combining Lemmas~\ref{lem: online_mftpl_poisson},~\ref{lem: smooth_bias}, and~\ref{lem: divergence_bound}, when $\poiparam \geq \max\cbr{ \frac{2AK^2}{\sigma}, \frac{8AK \ln(K N)}{\sigma} }$, 
we have that
, with probability at least $1-\delta/2$,

\begin{equation}
    \begin{aligned}
     & \sum_{n = 1}^N \inner{g_n}{u_n^*} - \min_{u \in \Delta(\Bcal)} \sum_{n = 1}^N \inner{g_n}{u} \\
    \leq &
    \EE_{\pertdata_1} \sbr{ \max_{u \in \Delta(\Bcal)} \inner{-\tilde{g}_1}{u}}
     +
    \sum_{n=1}^N \inner{g_n}{u_n^* - u_{n+1}^*}\\
    \leq &
    \sqrt{\frac{\poiparam \ln(B)}{2K^2}} 
    +
    N\sqrt{\frac{AK^2}{\poiparam \sigma}} 
+
N\sqrt{\frac{2A \ln(N)\ln(BN^2)}{\poiparam \sigma}} 
    + 
    \frac{\poiparam \sigma}{4A N \ln(N)}
    +
    4\sqrt{2N\ln(4/\delta}).
    \end{aligned}
    \label{eqn: bound_without_loggersparse}
\end{equation}


We now apply Proposition~\ref{prop: mftpl}, which, by the choice of $E$, 
gives that with probability $1-\delta/2$,
\[
\sum_{n=1}^N \inner{g_n}{u_n - u_n^*}
\leq 
N\sqrt{\frac{ 2A\rbr{\ln(NS)+\ln(\frac{12}{\delta})}}{\Ensem}} 
\leq 
O\rbr{ \sqrt{N \ln \frac1\delta} }.
\]

Therefore, combining the two inequalities and using the union bound, with probability at least $1-\delta$, Eq.~\eqref{eqn:regret-before-tuning-lambda} holds.


Now, for proving Eq.~\eqref{eqn:regret-after-tuning-lambda}, we first note that by the assumption that $N \geq \tilde{O}\rbr{ \sqrt{\frac{A}{\sigma}} \sqrt{ \min(\ln B, K^2)}}$, the choice of $\lambda = \Theta\rbr{ N K \sqrt{\frac{A}{\sigma}} + N K^2 \sqrt{ \frac{A}{\sigma \ln B} }}$ satisfies that $\poiparam \geq \max\cbr{ \frac{2AK^2}{\sigma}, \frac{8 AK \ln(KN)}{\sigma} }$; so Eq.~\eqref{eqn:regret-before-tuning-lambda} applies.
It can be checked by algebra that the first three terms in Eq.~\eqref{eqn:regret-before-tuning-lambda} evaluates to 
$\tilde{O} \rbr{ 
\sqrt{N} \rbr{\frac{A (\ln B)^2}{\sigma K^2}}^{\frac14}
+
\sqrt{N} \rbr{\frac{A \ln B}{\sigma}}^{\frac14} }$.

Furthermore, our choice of $\lambda$ and the assumption that $N \geq \tilde{O}\rbr{\frac{K^2}{A} \vee \frac{K^4}{A \ln B} }$ implies that 
\[
\frac{\lambda \sigma}{AN}
\leq 
O\rbr{
\frac{K}{\sqrt{A}}
+
\frac{K^2}{\sqrt{A \ln B}}
}
\leq 
O\rbr{\sqrt{N}};
\]
therefore, the last two terms in Eq.~\eqref{eqn:regret-before-tuning-lambda} is at most $O(\sqrt{N \ln \frac1\delta})$. 
This concludes the proof of Eq.~\eqref{eqn:regret-after-tuning-lambda}. 

\end{proof}

Theorem~\ref{thm: mftpl_poisson_main} immediately implies the following corollary: 

\begin{corollary}
\label{cor:soil-main-full}
For $\alpha > 0$ small enough, \algms with 
\begin{equation} 
N \geq \tilde{O} \rbr{ \frac{1}{\alpha^2}\sqrt{\frac{A \ln B}{\sigma}} },
\text{ and }
N K \geq \tilde{O} \rbr{ \frac{1}{\alpha^2}\sqrt{\frac{A ( \ln B)^2}{\sigma}} },
\label{eqn: req-n-k}
\end{equation}
is such that, with the choices of parameters 
$\poiparam = \Theta\rbr{ N K \sqrt{\frac{A}{\sigma}} + N K^2 \sqrt{ \frac{A}{\sigma \ln B} }}$
and 
$\Ensem = N A \ln (NS)$, 
achieves $\frac{\Reg(N)}{N}\leq \alpha$  with probability $1-\delta$; 
its number of calls to the offline oracle is $NE = \tilde{O}(N^2 A)$. 
\end{corollary}

Corollary~\ref{cor:soil-main-full} implies Corollary~\ref{cor:soil-main} in the main text, as we show below: 
\begin{proof}[Proof of Corollary~\ref{cor:soil-main}]
By the choices of $N$ and $K$, Eq.~\eqref{eqn: req-n-k} is satisfied with $\alpha = \frac{\epsilon}{\mu H}$. Therefore, with the choices of parameters given by Corollary~\ref{cor:soil-main-full}, $\frac{\Reg(N)}{N}\leq \alpha \implies \frac{\mu H \Reg(N)}{N}\leq \epsilon$. 

The total number of demonstrations requested is  $NK = \frac{\mu^2 H^2}{\epsilon^2} \sqrt{ \frac{A (\ln B)^2}{\sigma} }$, 
and the total number of calls to the offline oracle is 
$O(N^2 A) = \tilde{O}\del{ \frac{ \mu^4 H^4 A^2 (\ln B)^2}{ \epsilon^2 \sigma} }$
.
\end{proof}

\subsection{Deferred proofs from Section~\ref{sec:proofs}}
\label{sec:deferred_proofs}

The proposition below is used in the analysis of stage 1; 
its proof is 
straightforward and largely follows the proof of Proposition 6 in~\cite{li2022efficient}.

\begin{proposition}
For any $\delta \in (0,1]$, the sequence of $\{u_n\}$, $\{g_n\}$, $\{g_n^*\}$ induced by \algms, MDP $\MDP$ and expert $\pie$, satisfies that with probability at least 
    $1- \delta/3$, it holds simultaneously that: 
\[
\sum_{n=1}^N \inner{g_n^* - g_n}{u_n} \leq 
\sqrt{\frac{2N \ln(\frac{12}{\delta})}{K}},
\]
\[
\min_{u \in \Delta(\Bcal)}\sum_{n=1}^N \inner{g_n}{u}
- \min_{u \in \Delta(\Bcal)}\sum_{n=1}^N \inner{g_n^*}{u} 
\leq 
\sqrt{2N\frac{\ln (B)+\ln(\frac{12}{\delta})}{K}}.
\]
\label{prop:logger}
\end{proposition}
\begin{proof}
It suffices to show for any $\delta \in (0,1]$, (1). $\sum_{n=1}^N \inner{g_n^* - g_n}{u_n} \leq \sqrt{\frac{2N \ln(\frac{12}{\delta})}{K}}$ with probability at least $1- \delta/6$, (2). $\min_{u \in \Delta(\Bcal)}\sum_{n=1}^N \inner{g_n}{u}
- \min_{u \in \Delta(\Bcal)}\sum_{n=1}^N \inner{g_n^*}{u}  \leq \sqrt{2N\frac{\ln (B)+\ln(\frac{12}{\delta})}{K}}$ with probability at least $1- \delta/6$.

For (1), we define $Y_{n,k} = F_n(\pi_n) - \EE_{a \sim \pi_n(\cdot|s_{n,k})}\sbr{I(a \neq \pie(s_{n,k}))}$, which satisfies
\[
\begin{aligned}
\inner{g_n^* - g_n}{u_n}
=&
F_n(\pi_n) -  \EE_{(s,\pie(s)) \sim \incrdata_n}\EE_{a \sim \pi_n(\cdot|s)} \sbr{I(a \neq \pie(s))}\\
=&
\frac{1}{K}\sum_{k=1}^K \rbr{F_n(\pi_n) - \EE_{a \sim \pi_n(\cdot|s_{n,k})}\sbr{I(a \neq \pie(s_{n,k}))}}
=
\frac{1}{K}\sum_{k=1}^K Y_{n,k},
\end{aligned}
\]
where we apply $\inner{g_n^*}{u_n} = F_n(\pi_n)$ and $\incrdata_n = (s_{n,k},\pie(s_{n,k}) )_{k=1}^K$.

Now, it suffices to bound $\sum_{n=1}^N\sum_{k=1}^K Y_{n,k}$, which can be verified to be a martingale difference sequence. By the definition of $\iterloss_n(\pi) := \EE_{s \sim d_{\pi_n}}  \EE_{a\sim \pi(\cdot | s)}\sbr{I(a \neq \pie(s))}$, it can be shown that $\EE\sbr{Y_{n,k}|Y_{1,1},Y_{1,2}, \cdots,Y_{1,K},Y_{2,1},\cdots,Y_{n,k-1}} = 0$, while $| Y_{n,k} | \leq 1$. By applying Azuma-Hoeffding’s inequality,  with probability at least $1- \delta/6$,
\[
\left| \sum_{n=1}^N \inner{g_n^* - g_n}{u_n} \right| = \frac{1}{K} \left|\sum_{n=1}^N\sum_{k=1}^K Y_{n,k}\right| \leq  \sqrt{\frac{2N \ln(\frac{12}{\delta})}{K}}.
\]

For (2), we define $\hat{Y}_{n,k}(h) = F_n(h) - I\rbr{h(s_{n,k}) \neq \pie(s_{n,k})}$, where $h \in \Bcal$, $s_{n,k} \sim d_{\pi_n}$. Following similar proof as (1), it can be verified that $\EE\sbr{\hat{Y}_{n,k}(h)|\hat{Y}_{1,1}(h),\cdots, \hat{Y}_{n,k-1}(h)} = 0$ and $| \hat{Y}_{n,k}(h) | \leq 1$. Again we apply Azuma-Hoeffding’s inequality and show that given any $\delta \in (0,1]$ and $h \in \Bcal$, with probability at least $1-\frac{\delta}{6B}$, 
$$
\left |\sum_{n=1}^N\rbr{F_n(h) - \frac{1}{K}\sum_{k=1}^K I\rbr{h(s_{n,k}) \neq \pie(s_{n,k})}} \right |  
=
\left |\frac{1}{K}\sum_{n=1}^N\sum_{k=1}^K \hat{Y}_{n,k}(h) \right| 
\leq 
\sqrt{2N\frac{\ln(B)+\ln(\frac{12}{\delta})}{K}}.
$$
By applying union bound over all policies in $\Bcal$, we have for all $h \in \Bcal$, given any $\delta \in (0,1]$, with probability at least $1-\frac{\delta}{6}$, it satisfies that
$$
\frac{1}{K}\sum_{k=1}^K I\rbr{h(s_{n,k}) \neq \pie(s_{n,k})}
-
\sum_{n=1}^NF_n(h) 
\leq 
\sqrt{2N\frac{\ln(B)+\ln(\frac{12}{\delta})}{K}}.
$$

Since $\min_{u \in \Delta(\Bcal)}\sum_{n=1}^N \inner{g_n^*}{u} = \min_{h \in \Bcal}\sum_{n=1}^N F_n(h)$, while $\min_{u \in \Delta(\Bcal)}\sum_{n=1}^N \inner{g_n}{u} = \min_{h \in \Bcal}\frac{1}{K}\sum_{n=1}^N\sum_{k=1}^K I\rbr{h(s_{n,k}) \neq \pie(s_{n,k})}$, we denote $h^* = \argmin_{h \in \Bcal}\sum_{n=1}^N F_n(h)$ and conclude the proof for (2) by
\[
\begin{aligned}
\min_{u \in \Delta(\Bcal)}\sum_{n=1}^N \inner{g_n}{u}
- \min_{u \in \Delta(\Bcal)}\sum_{n=1}^N \inner{g_n^*}{u} 
=&
\min_{h \in \Bcal}\frac{1}{K}\sum_{n=1}^N\sum_{k=1}^K I\rbr{h(s_{n,k}) \neq \pie(s_{n,k})}
-\sum_{n=1}^N F_n(h^*) \\
=&
\min_{h \in \Bcal}\frac{1}{K}\sum_{n=1}^N\sum_{k=1}^K I\rbr{h(s_{n,k}) \neq \pie(s_{n,k})}\\
&-
\frac{1}{K}\sum_{n=1}^N\sum_{k=1}^K I\rbr{h^*(s_{n,k}) \neq \pie(s_{n,k})}\\
&+
\frac{1}{K}\sum_{n=1}^N\sum_{k=1}^K I\rbr{h^*(s_{n,k}) \neq \pie(s_{n,k})}
-
\sum_{n=1}^N F_n(h^*) \\
\leq&
\sqrt{2N\frac{\ln(B)+\ln(\frac{12}{\delta})}{K}}.
\end{aligned}
\]
Finally, by applying union bound on (1) and (2) we conclude the proof.
\end{proof}


The proposition below is used in the analysis of stage 1; 
its proof is 
straightforward and largely follows from Lemma 7 and 8 of~\cite{li2022efficient}.
    
\begin{proposition}
 For any $\delta \in (0,1]$, the sequence of $\{u_n\}$, $\{g_n\}$, $\{u_n^*\}$ induced by \algms, MDP $\MDP$ and expert $\pie$,  satisfies that with probability at least 
    $1- \delta/6$,     
$$\sum_{n=1}^N \langle {\est_n}, {u_n - u_n^*} \rangle  \leq  N\sqrt{\frac{ 2A\rbr{\ln(NS)+\ln(\frac{12}{\delta})}}{\Ensem}}.$$
\label{prop:sparsification}
\end{proposition}
\begin{proof}
To begin with, we first denote $\pi_{n}^* := \pi_{u_n^*}$ and $\pi_{n}^*(\cdot|s)$ the action distribution of $\pi_{n}^*$ on state $s$. Given the expert annotation $\pie(s)$ on state $s$, we denote the $A$ dimensional classification loss vector by $\overset{\rightarrow}{c}(\pie(s))$, whose entries are all $1$ except that it takes $0$ in the $\pie(s)$-th coordinate.
With the newly introduced notions, we rewrite and bound $ \langle {\est_n}, {u_n - u_n^*} \rangle$ as follows:
\begin{equation}
 \begin{aligned}
\langle {\est_n}, {u_n - u_n^*} \rangle 
=&
\sum_{h \in \Bcal} u_n[h] \rbr{\EE_{(s,\pie(s)) \sim \incrdata_n} \sbr{I(h(s) \neq \pie(s))}}_{h \in \Bcal}\\
&-
\sum_{h \in \Bcal} u_n^*[h] \rbr{\EE_{(s,\pie(s)) \sim \incrdata_n} \sbr{I(h(s) \neq \pie(s))}}_{h \in \Bcal} \\
=&
\EE_{(s,\pie(s)) \sim \incrdata_n} \EE_{a \sim \pi_{n}(\cdot|s)} \sbr{I(a \neq \pie(s))}  
-
\EE_{(s,\pie(s)) \sim \incrdata_n} \EE_{a \sim \pi^*_{n}(\cdot|s)} \sbr{I(a \neq \pie(s))}   \\
=&
\EE_{(s,\pie(s)) \sim \incrdata_n} \sbr{\langle \pi_{n}(\cdot|s),\overset{\rightarrow}{c}(\pie(s))\rangle }  
-
\EE_{(s,\pie(s)) \sim \incrdata_n} \sbr{\langle \pi^*_{n}(\cdot|s),\overset{\rightarrow}{c}(\pie(s))\rangle } \\
= &
\EE_{(s,\pie(s)) \sim \incrdata_n} \sbr{\langle \pi_{n}(\cdot|s)-\pi^*_{n}(\cdot|s),\overset{\rightarrow}{c}(\pie(s))\rangle }  \\
\leq &
\EE_{(s,\pie(s)) \sim \incrdata_n} \sbr{\| \pi_{n}(\cdot|s)-\pi^*_{n}(\cdot|s)\|_1 \| \overset{\rightarrow}{c}(\pie(s))\|_\infty }  \\
=&
\EE_{(s,\pie(s)) \sim \incrdata_n} \sbr{\| \pi_{n}(\cdot|s)-\pi^*_{n}(\cdot|s)\|_1 }  .
\end{aligned}
\end{equation}
Now, it suffices to bound $\| \pi_{n}(\cdot|s)-\pi^*_{n}(\cdot|s)\|_1$, which follows Lemma 7 of~\cite{li2022efficient}.
First notice that $\EE_{\pertdata_n} \sbr{u_{n,\ensem}} = u_n^*$, which is by the definition of $u_n^*$ in equation~\eqref{eqn:u_star}. Since $h_{n,\ensem}$ corresponds to $u_{n,\ensem}$ in \algms, this implies $\EE_{\pertdata_n} \sbr{h_{n,\ensem}(\cdot|s)} = \pi^*_{n}(\cdot|s)$. Now that each $h_{n,\ensem}(\cdot|s)$ can be viewed as Multinoulli random variable on $\Delta(\Acal)$ with expectation $\pi^*_{n}(\cdot|s)$, while $\pi_{n}(\cdot \mid s) := \frac{1}{\Ensem}\sum_{\ensem = 1}^\Ensem h_{n, e}(\cdot | s)$, we apply the concentration inequality for Multinoulli random variables~\cite{qian2020concentration, weissman2003inequalities} and conclude given $n\in[N]$ and $s \in \Scal$, for any $\delta_0 \in (0,1]$, $u_n$, $u^*_n$, $g_n$, satisfies that with probability at least $1-\frac{\delta_0}{NS}$,
$$\| \pi_{n}(\cdot|s)-\pi^*_{n}(\cdot|s)\|_1
\leq
\sqrt{\frac{ 2A\rbr{\ln(NS)+\ln(\frac{2}{\delta_0})}}{\Ensem}}.
$$
By applying union bound over all $n \in [N]$ and all $s \in \Scal$, we conclude that for any $\delta \in (0,1]$, the sequence of  $\{u_n\}$, $\{u^*_n\}$, $\{g_n\}$, satisfies that with probability at least $1-\delta/6$,
\begin{equation}
    \begin{aligned}
        \sum_{n=1}^N\langle {\est_n}, {u_n - u_n^*} \rangle 
\leq &
\sum_{n=1}^N \EE_{(s,\pie(s)) \sim \incrdata_n} \sbr{\| \pi_{n}(\cdot|s)-\pi^*_{n}(\cdot|s)\|_1 } \\
\leq &
\sum_{n=1}^N \EE_{(s,\pie(s)) \sim \incrdata_n} \sqrt{\frac{ 2A\rbr{\ln(NS)+\ln(\frac{12}{\delta})}}{\Ensem}}  \\
= &
N\sqrt{\frac{ 2A\rbr{\ln(NS)+\ln(\frac{12}{\delta})}}{\Ensem}}.
    \end{aligned}
\end{equation}

\end{proof}
\clearpage

\section{Experimental details}
\label{sec: experiment_appendix}

\subsection{Full version of \alg}
\label{sec: full_bd_alg}
We present the full version of \alg in Algorithm~\ref{alg:bd_full}.

\begin{algorithm}[t]
\caption{\alg}
\begin{algorithmic}
\STATE \textbf{Input:} MDP $\MDP$, expert $\pie$, policy class $\Bcal$, oracle $\Ocal$,
sample size per round $K$,  ensemble size $\Ensem$.
\STATE Initialize $\incrdata = \emptyset$. 
\FOR{$n=1,2,\ldots,N$}
\FOR{$\ensem=1,2,\ldots,\Ensem$}
\STATE $\pi_{n, e} \gets \trainbase(D)$
\ENDFOR
\STATE Set $\pi_{n}(a \mid s) := \frac{1}{\Ensem}\sum_{\ensem = 1}^\Ensem I( \pi_{n, e}(s) = a )$.
\STATE $\incrdata_n = \cbr{ (s_{n,k}, \pie(s_{n,k})) }_{k = 1}^K \gets$ sample $K$ states i.i.d. from $d_{\pi_{n}}$ by rolling out $\pi_{n}$ in $\MDP$, and query expert $\pie$ on these states. 
\STATE Aggregate datasets  $\incrdata \gets \incrdata \cup \incrdata_n$.
\ENDFOR
\STATE \textbf{Return} $\hat{\pi} \gets \aggpolicy( \cbr{ \pi_{n,e}}_{n=1, e=1}^{N+1, E} )$

\STATE \textbf{function} \trainbase($D$):
\STATE \hspace{1em} $\tilde{\incrdata}$ $\gets$ Sample $|\incrdata|$ i.i.d. samples $\sim \Unif(\incrdata)$ with replacement.
\STATE \hspace{1em} \textbf{Return} $h \gets \Ocal(\tilde{\incrdata})$.

\STATE \textbf{Return} $h \gets \Ocal( D \cup Q)$.


\STATE \textbf{function} \aggpolicy$(\cbr{ \pi_{n,e}}_{n=1, e=1}^{N+1, E})$:

\STATE \hspace{1em} Sample $\hat{n} \sim \Unif([N])$


\STATE \hspace{1em} \textbf{Return} $\pi_{\hat{n}}(a \mid s) := \frac{1}{\Ensem}\sum_{\ensem = 1}^\Ensem I( \pi_{\hat{n}, e}(s) = a )$.

\end{algorithmic}
\label{alg:bd_full}
\end{algorithm}

\subsection{Additional Implementation Details}
\label{sec: experiment_details}

All experiments were conducted on an Ubuntu machine equipped with a 3.3 GHz Intel Core i9 CPU and 4 NVIDIA GeForce RTX 2080 Ti GPUs. Our project is built upon the source code of Disagreement-Regularized Imitation Learning (https://github.com/xkianteb/dril) and shares the same environment dependencies. We have inherited some basic functions and implemented a new online learning pipeline that supports parallelized ensemble policies, in which we instantiate  \dagge, \algms, and \alg. For each algorithm and experimental setting, we executed ten runs using random seeds ranging from 1 to 10. The detailed control task names are ``HalfCheetahBulletEnv-v0'', ``AntBulletEnv-v0'', ``Walker2DBulletEnv-v0'', and ``HopperBulletEnv-v0''. For code and more information see \text{https://github.com/liyichen1998/BootstrapDagger-MFTPLP} 

\begin{table}[h]
\caption{Hyperparameters for Continuous Control Experiment} 
\label{bootstrap-dagger-hyperparams}
\begin{center}
\begin{tabular}{lll}
\toprule
\textbf{Hyperparameter}  & \textbf{Values Considered} & \textbf{Choosen Value} \\
\midrule
Ensemble Size     & [1,5,25] & [1,5,25]\\
Perturbation Size $X$ for \algms     & [7,15,31,62,125,250,500] & [0,7,15,31,62,125]\\
Hidden Layer Size (non-realizable)    & [2,4,8,12,16,24,32,64] & 4 (Ant), 8 (Hopper), 12 (Half-Chettah), 24 (Walker)\\
Learning Rounds (realizable)   & [10,20,50,100] & 20 (Ant \& Hopper), 50 (Half-Cheetah \& Walker)\\
Learning Rounds (non-realizable)   & [10,20,40,50,80,100,200] & 40 (Ant \& Hopper), 50 (Half-Cheetah \& Walker) \\
Data Per Round    & [10,20,50,100,200,1000] & 50 \\
Learning Rate     & $2.5 \times 10^{-4}$ & $2.5 \times 10^{-4}$ \\
Batch Size        & [50,100,200,500,1000] & 200 \\
Train Epoch       & [500,1000,2000,10000] & 2000 \\
Parallel Environments & [5,16,25] & 25 \\
\bottomrule
\end{tabular}
\end{center}
\end{table}

As shown in Table~\ref{bootstrap-dagger-hyperparams}, we present the considered hyperparameters along with their chosen values.
Overall, hyperparameters related to environment interactions, like learning rounds and data per round, are selected to generate a 'favorable' learning curve for \dagge, enabling it to learn relatively fast but not converge within just a few rounds. 
The perturbation sizes for \algms are chosen by the ratio of the perturbation dataset size to 1000 (the maximum size of the cumulative dataset for realizable Ant, Hopper), following the sequence of $\cbr{\lfloor 1000/2^i \rfloor | 1 \leq i \leq 7}$.
For hyperparameters related to neural network training, such as batch size, training epoch, etc., those are selected to ensure a faithful implementation of a offline learning oracle without imposing heavy computational overhead. 

For the justification of using 2000 SGD iterations without validation set, we provide the following reasoning:

1. On the performance of the oracle, 2000 SGD iterations suffices to support a comprehensive comparison over different algorithms, as shown in Figure~\ref{fig:500_10k_comparison}. 

2. On the faithfulness of implementing a offline learning oracle, the returned policy should be the best policy on the input data, where the generalization error is not considered. In this case, splitting out a portion of input data for validation may deviate from the definition of the computational oracle.

3. On the reproducibility of the experiment, the implementation of validation set may vary and provide additional noises, i.e. whether the validation set is resampled for each input dataset or gathered incrementally through $N$ rounds and kept unseen from the learner. Though it would be interesting to compare the difference between these and our oracle.

\begin{figure*}[h]
  \centering
  \includegraphics[width=1\linewidth]{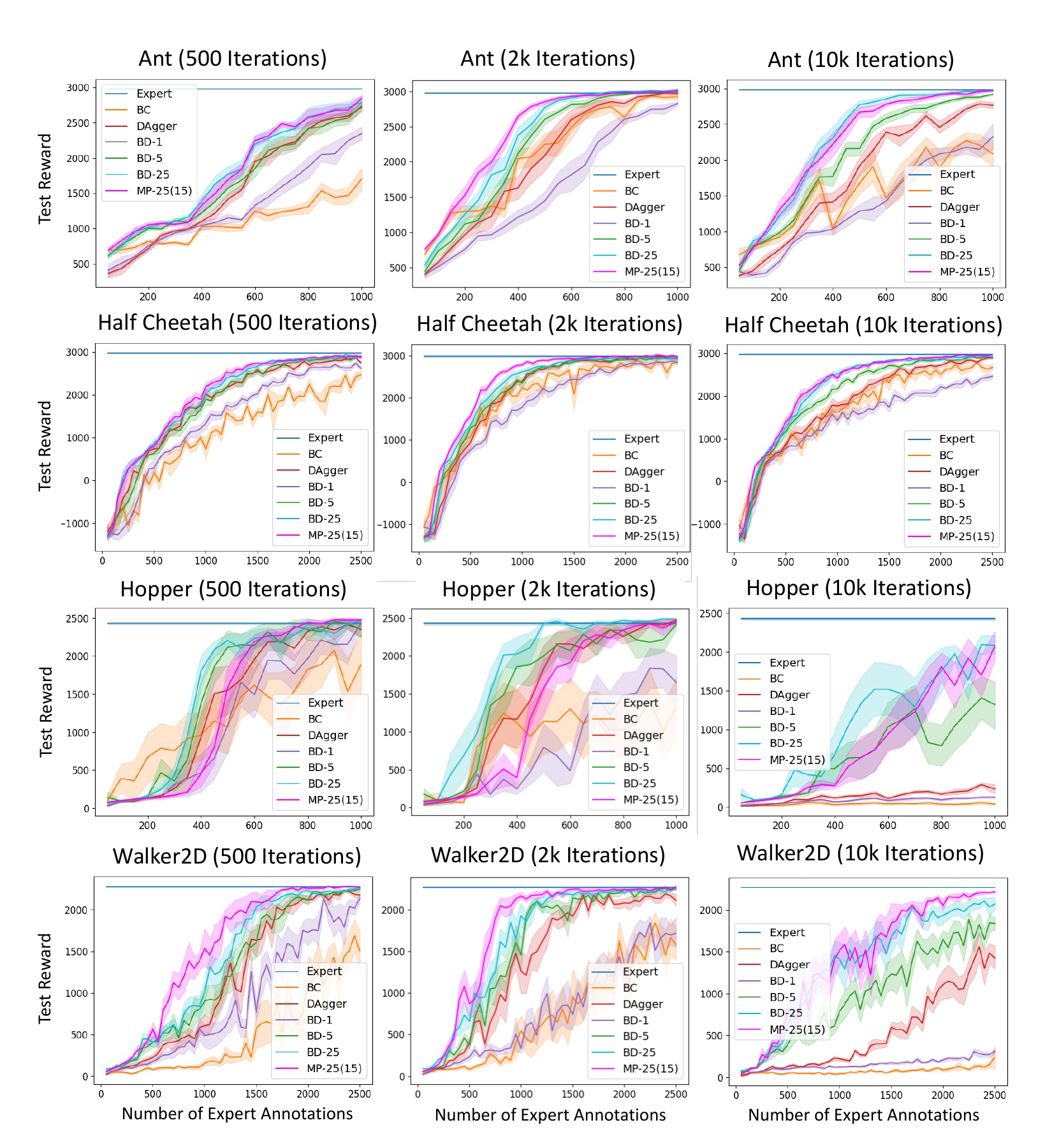}
  \caption{Comparison of 500, 2000, and 10000 iteration steps on continuous control tasks with realizable expert.
  } 
  \label{fig:500_10k_comparison}
\end{figure*}

\textbf{Running Time and Memory Comparison.}

We present the running time and memory comparisons in the realizable setting from Section~\ref{sec:main_experiment}. As shown in Figure~\ref{fig:time_mem}, \bc, \dagge, and \bda have similar running times across different tasks. Benefiting from the parallel implementation of ensemble models, \bdb only takes twice as long as the baselines, while those with an ensemble size of 25 require approximately 5 times longer. Considering overall performance and running time, \bdb is more favorable for practical applications.

\label{sec: time_and_memory}
\begin{figure}[H]
  \centering
  \includegraphics[width=1\linewidth]{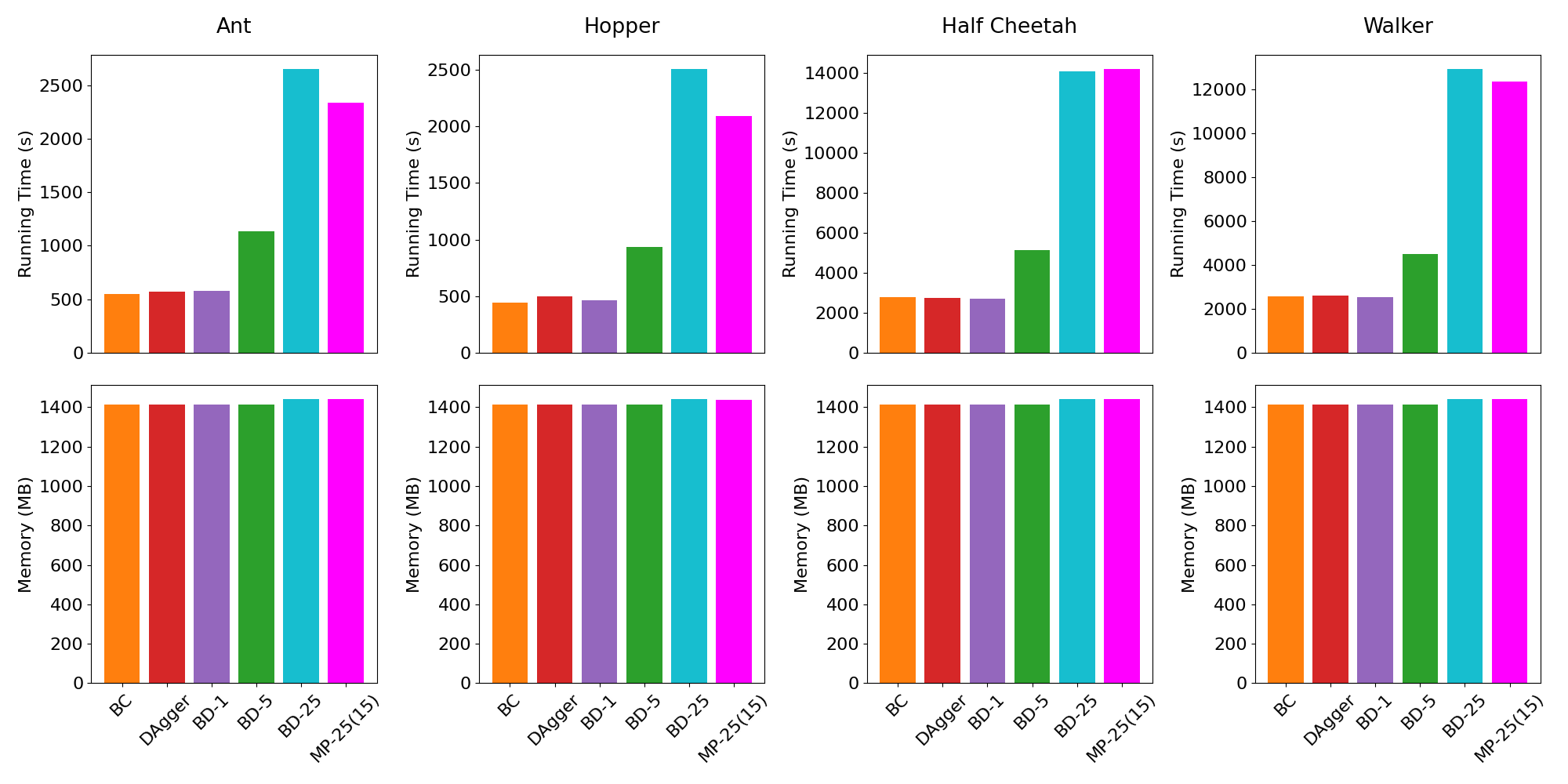}
  \caption{
  Comparison of running time and memory for different algorithms under the realizable setting.
  }
  \label{fig:time_mem}
\end{figure}

\textbf{Comparison of Alternative \aggpolicy Functions for \alg}

Additionally, we compare the performance of our \alg algorithm with the \aggpolicy part changed to randomization over the ensemble trained. Our results are shown in Figure~\ref{fig:ensemble-mean-random}, where BD-5 and BD-25 represent our \alg algorithm with ensemble sizes 5 and 25 respectively, and ``BD-5 random'' and ``BD-25 random'' are modifications of \alg with the \aggpolicy part changed to returning $\bar{\pi}_{N+1}(a \mid s) = \frac{1}{E} \sum_{e=1}^E \pi_{N+1,e}(a \mid s)$. 

\begin{figure}[H]
  \centering
  \includegraphics[width=1\linewidth]{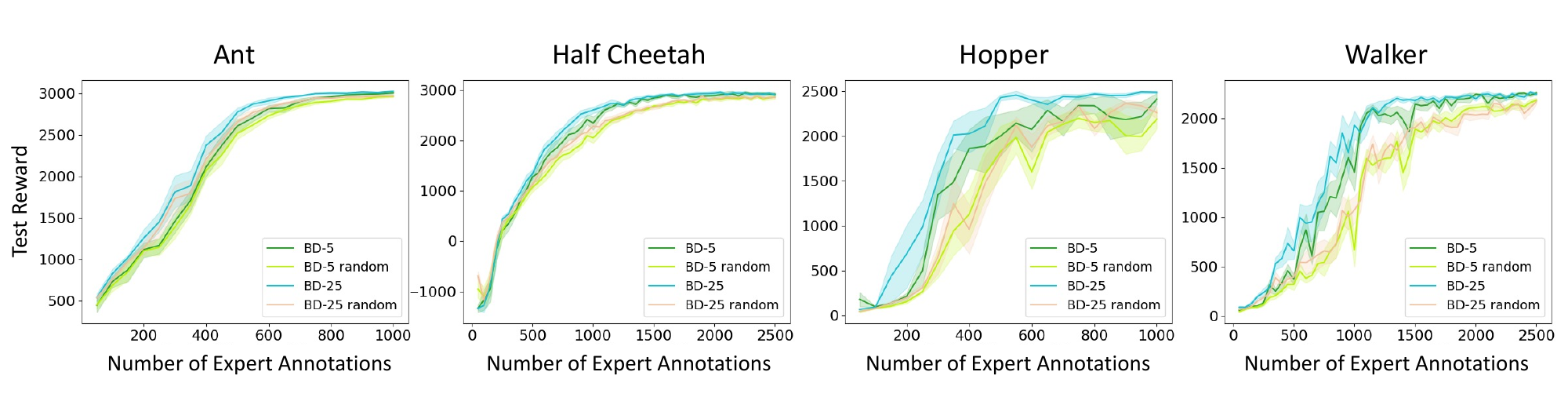}
  \caption{
  Comparison of the original \alg algorithm with its variant with \aggpolicy changed to returning a policy that is a uniform randomization over the ensemble.
  }
  \label{fig:ensemble-mean-random}
\end{figure}

Overall, using bootstrap mean (Bagging) proves marginally better than using bootstrap with randomization for final policy evaluation. This justifies the choice of using the ensemble mean instead of the original \aggpolicy for continuous control.





\subsection{Full Results from Section~\ref{sec: perturbation_utility}}
\label{sec:full_perturbation_utility}
In this section, we present all result plots from Section~\ref{sec: perturbation_utility}, including those omitted due to space constraints. As shown in Figure~\ref{fig:full_linear}, we include the performance of \alg and \bc for the linear nonrealizable experiment in Ant and Hopper. Evidently, the performance of \bd improves with increasing ensemble size, with \bdc achieving comparable performance to \algms. Notice that the performance of \bc varies significantly across tasks in the linear model setting, which 
we leave for further investigation.

\begin{figure*}[h]
  \vspace{.3in}
  \centering

  \adjustbox{trim={0.00\width} {0.00\height} {0.08\width} {0.05\height},clip,width=0.31\linewidth, valign=t}{
    \includegraphics[width=1.0\linewidth]{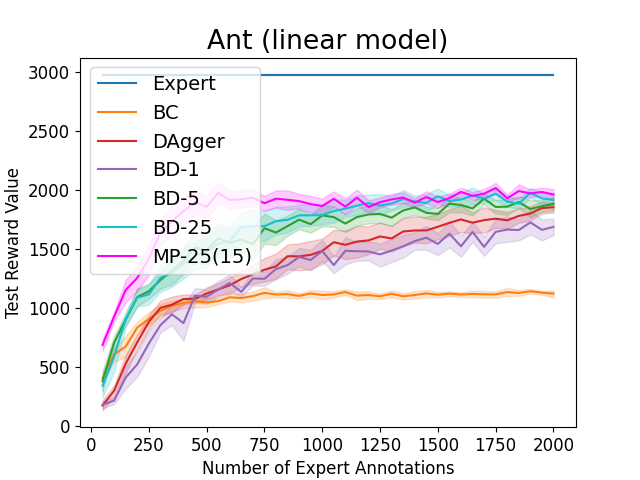}
  }
  \hspace{0.02\linewidth}
  \adjustbox{trim={0.00\width} {0.00\height} {0.08\width} {0.05\height},clip,width=0.31\linewidth, valign=t}{
    \includegraphics[width=1.0\linewidth]{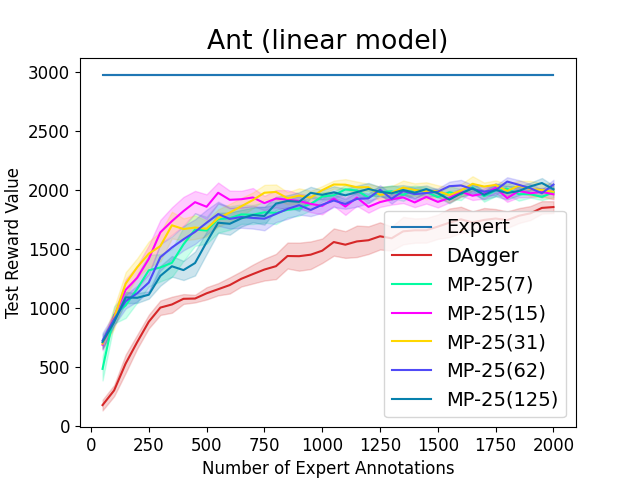}
  }
  \hspace{0.02\linewidth}
  \adjustbox{trim={0.00\width} {0.00\height} {0.08\width} {0.05\height},clip,width=0.31\linewidth, valign=t}{
    \includegraphics[width=1.0\linewidth]{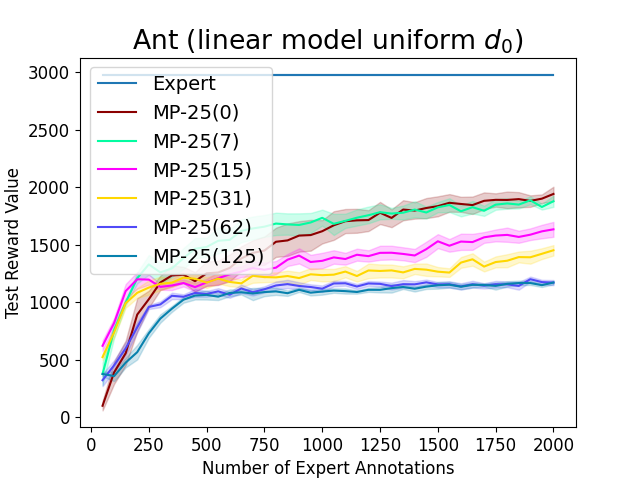}
  }
  
  \vspace{0.2in}

  \adjustbox{trim={0.00\width} {0.00\height} {0.08\width} {0.05\height},clip,width=0.31\linewidth, valign=t}{
    \includegraphics[width=1.0\linewidth]{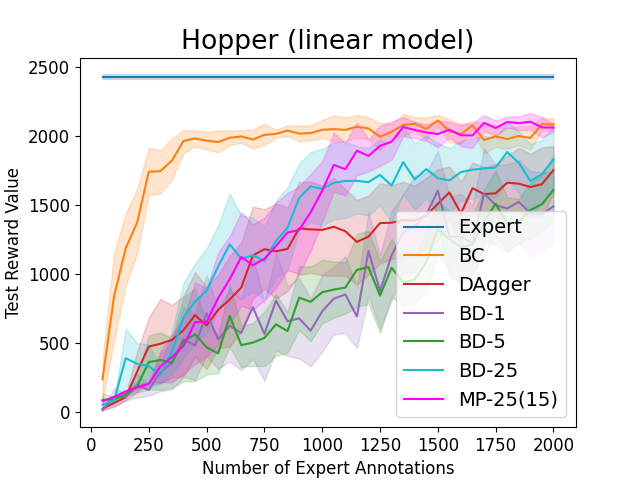}
  }
  \hspace{0.02\linewidth}
  \adjustbox{trim={0.00\width} {0.00\height} {0.08\width} {0.05\height},clip,width=0.31\linewidth, valign=t}{
    \includegraphics[width=1.0\linewidth]{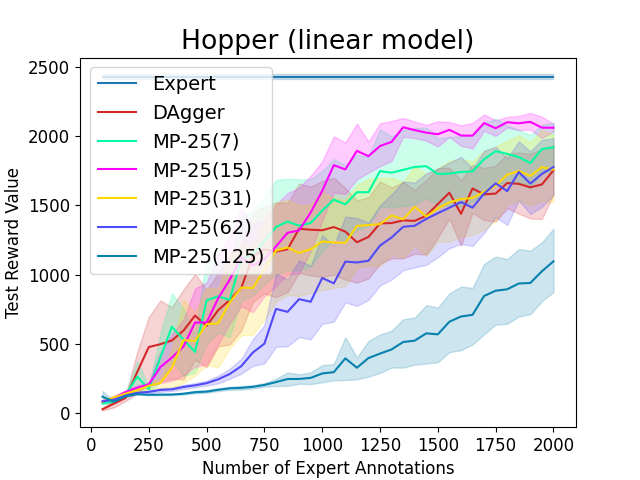}
  }
  \hspace{0.02\linewidth}
  \adjustbox{trim={0.00\width} {0.00\height} {0.08\width} {0.05\height},clip,width=0.31\linewidth, valign=t}{
    \includegraphics[width=1.0\linewidth]{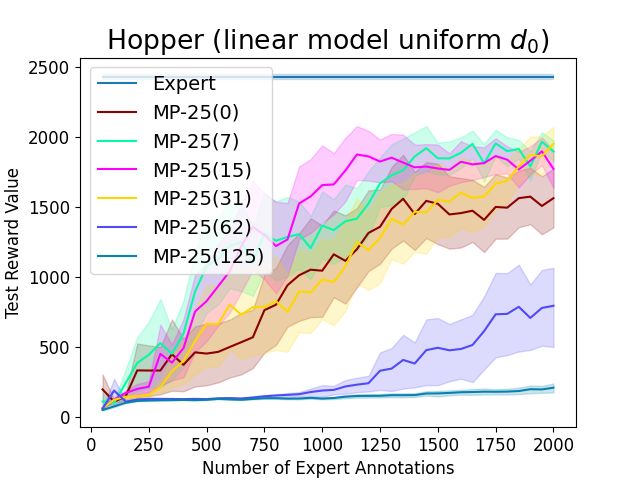}
  }

  \vspace{.3in}
  \caption{Continuous control experiments with linear model and nonrealizable noisy expert.}
  \label{fig:full_linear}
\end{figure*}

\subsection{Full results from Section~\ref{sec:main_experiment}}
\label{sec:full_main_experiment}


In this section, we present all result plots from Section~\ref{sec:main_experiment}, including those omitted due to space constraints. 
Additionally, since imitation learning agents do not usually have access to the ground truth reward,
we also evaluate $\pi$ 
using a more ``objective'' performance measure, i.e, its average imitation loss: 
\[\text{Imitation Loss}(\pi) = \frac{1}{\nrepeat} \sum_{i =1}^{\nrepeat}\frac{1}{|\tau^\pi_i|}\sum_{s \in \tau^\pi_i}\tilde{\ell}(\pi(s),\bar{\pi}^\text{exp}(s)),\]
this measures the policy's  deviation from the mean action of the expert policy $\bar{\pi}^\text{exp}(s)$.
Its expectation, 
$L(\pi) := \EE_{s \sim d_{\pi}, a \sim \pi(\cdot \mid s)} \ell(a, \pie(s))$, has been a central optimization objective in imitation learning works such as 
\dagge~\citep[][Eq. (1)]{ross2011reduction} and subsequent works~\cite{ross2014reinforcement,sun2017deeply,cheng2019predictor} including ours.


In the following, Figures~\ref{fig:full_realizable} and~\ref{fig:nonrealizable_full} present the results of experiments with realizable and non-realizable experts using MLP as the base class. These include the performance of \MPX with varying perturbation sizes and the imitation loss of different algorithms as a function of expert annotation size. Although the advantages of sample-based perturbation are less apparent in the MLP-based experiments, \MPF still noticeably outperforms \MP{}(0) in realizable Ant. 

By examining the correlations between test rewards and imitation losses, a discrepancy in imitation loss usually correlates with
a gap in test reward, which shows the practical relevance of minimizing a policy's imitation loss -- it is a reasonable proxy of policy's expected performance. 




\begin{figure*}[h]
  \vspace{.3in}
  \centering

  \adjustbox{trim={0.00\width} {0.00\height} {0.08\width} {0.05\height},clip,width=0.31\linewidth, valign=t}{
    \includegraphics[width=1.0\linewidth]{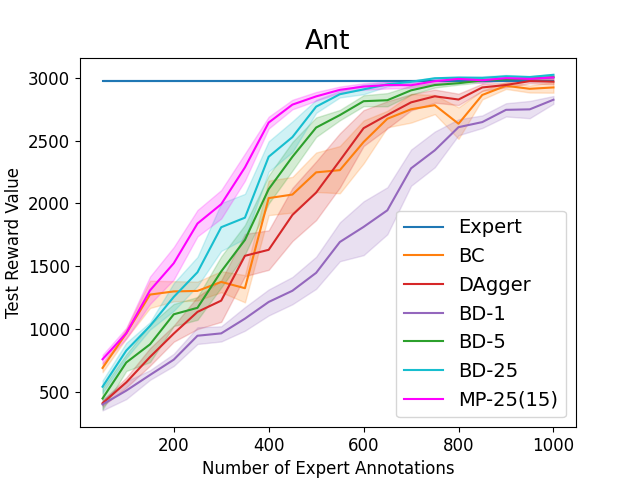}
  }
  \hspace{0.02\linewidth}
  \adjustbox{trim={0.00\width} {0.00\height} {0.08\width} {0.05\height},clip,width=0.31\linewidth, valign=t}{
    \includegraphics[width=1.0\linewidth]{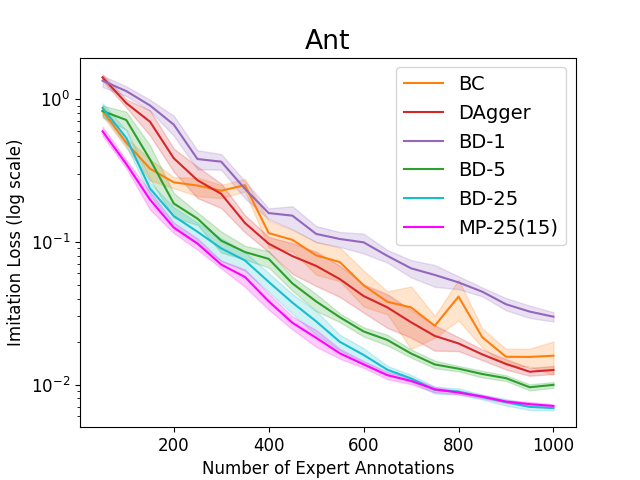}
  }
  \hspace{0.02\linewidth}
  \adjustbox{trim={0.00\width} {0.00\height} {0.08\width} {0.05\height},clip,width=0.31\linewidth, valign=t}{
    \includegraphics[width=1.0\linewidth]{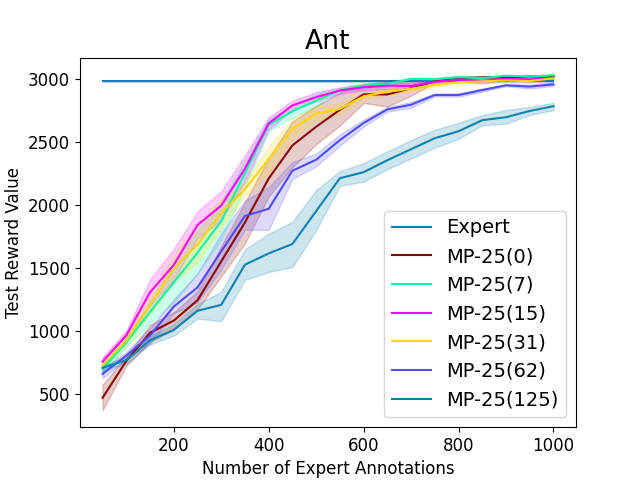}
  }
  
  \vspace{0.2in}

  \adjustbox{trim={0.00\width} {0.00\height} {0.08\width} {0.05\height},clip,width=0.31\linewidth, valign=t}{
    \includegraphics[width=1.0\linewidth]{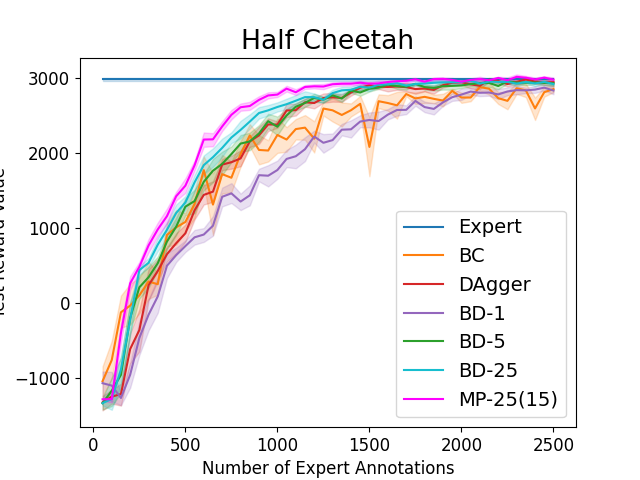}
  }
  \hspace{0.02\linewidth}
  \adjustbox{trim={0.00\width} {0.00\height} {0.08\width} {0.05\height},clip,width=0.31\linewidth, valign=t}{
    \includegraphics[width=1.0\linewidth]{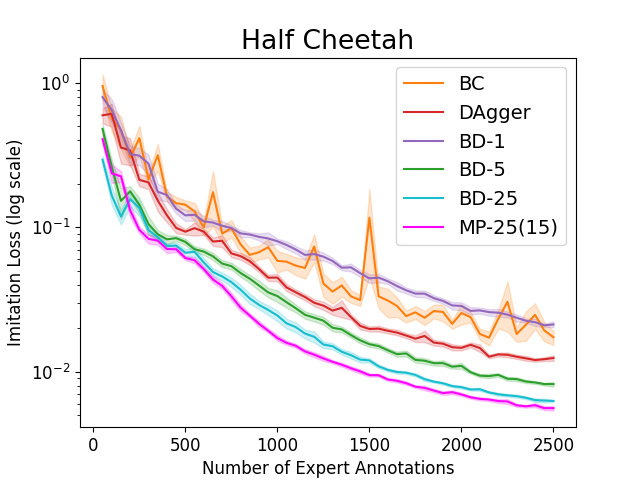}
  }
  \hspace{0.02\linewidth}
  \adjustbox{trim={0.00\width} {0.00\height} {0.08\width} {0.05\height},clip,width=0.31\linewidth, valign=t}{
    \includegraphics[width=1.0\linewidth]{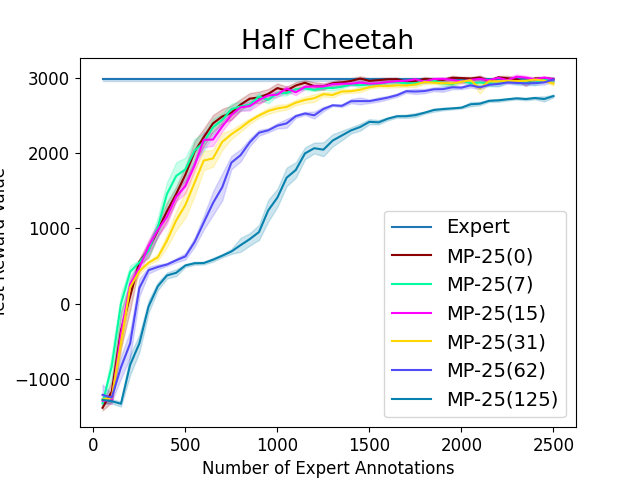}
  }

  \vspace{0.2in}

  \adjustbox{trim={0.00\width} {0.00\height} {0.08\width} {0.05\height},clip,width=0.31\linewidth, valign=t}{
    \includegraphics[width=1.0\linewidth]{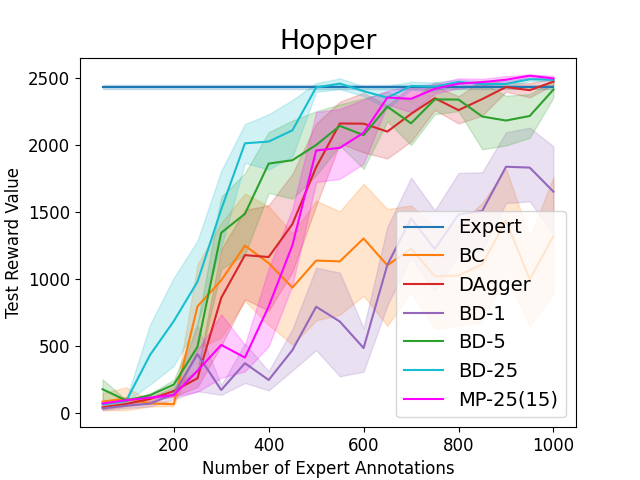}
  }
  \hspace{0.02\linewidth}
  \adjustbox{trim={0.00\width} {0.00\height} {0.08\width} {0.05\height},clip,width=0.31\linewidth, valign=t}{
    \includegraphics[width=1.0\linewidth]{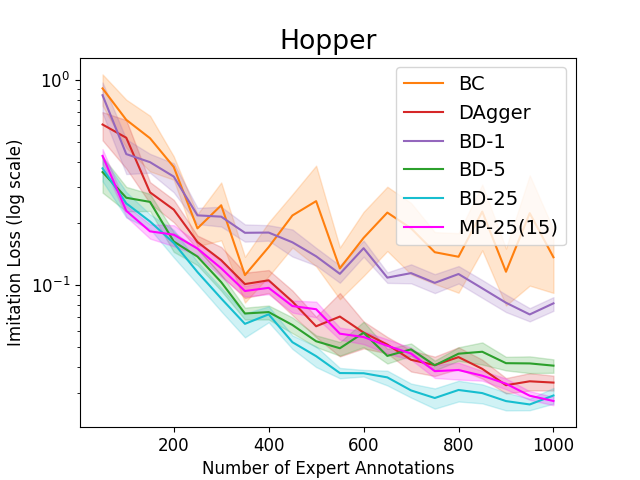}
  }
  \hspace{0.02\linewidth}
  \adjustbox{trim={0.00\width} {0.00\height} {0.08\width} {0.05\height},clip,width=0.31\linewidth, valign=t}{
    \includegraphics[width=1.0\linewidth]{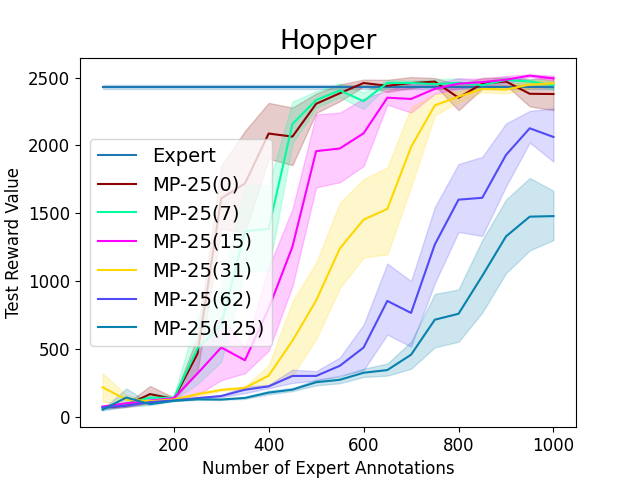}
  }

  \vspace{0.2in}

  \adjustbox{trim={0.00\width} {0.00\height} {0.08\width} {0.05\height},clip,width=0.31\linewidth, valign=t}{
    \includegraphics[width=1.0\linewidth]{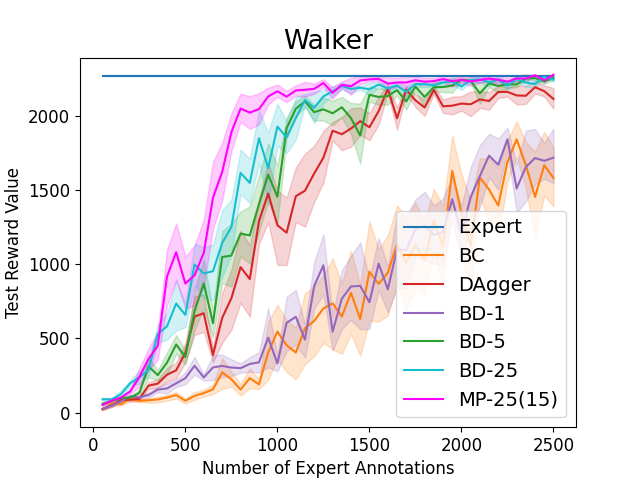}
  }
  \hspace{0.02\linewidth}
  \adjustbox{trim={0.00\width} {0.00\height} {0.08\width} {0.05\height},clip,width=0.31\linewidth, valign=t}{
    \includegraphics[width=1.0\linewidth]{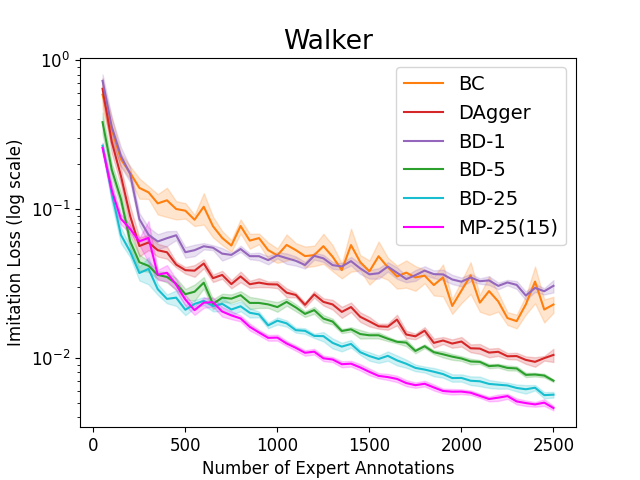}
  }
  \hspace{0.02\linewidth}
  \adjustbox{trim={0.00\width} {0.00\height} {0.08\width} {0.05\height},clip,width=0.31\linewidth, valign=t}{
    \includegraphics[width=1.0\linewidth]{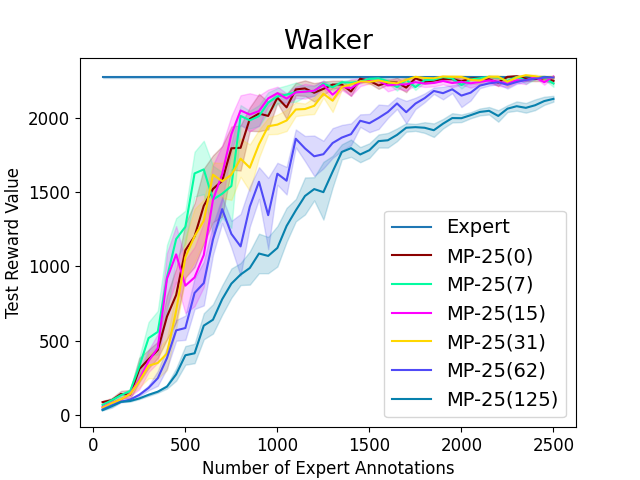}
  }

  \vspace{.3in}
  \caption{Continuous control experiments with realizable noisy expert.}
  \label{fig:full_realizable}
\end{figure*}

\begin{figure*}[h]
  \centering

  \adjustbox{trim={0.00\width} {0.00\height} {0.08\width} {0.05\height},clip,width=0.31\linewidth, valign=t}{
    \includegraphics[width=1.0\linewidth]{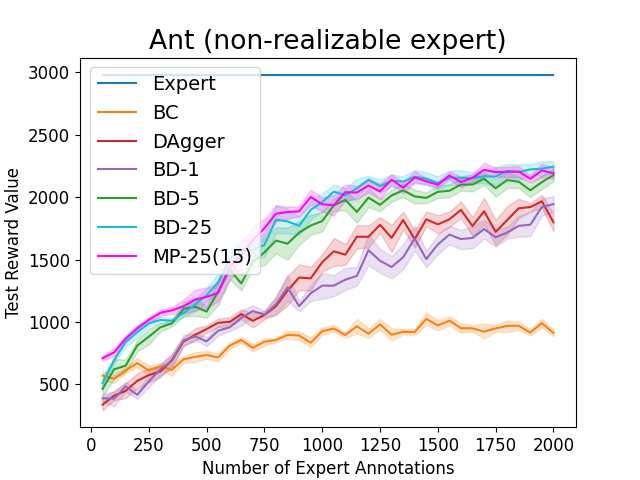}
  }
  \hspace{0.02\linewidth}
  \adjustbox{trim={0.00\width} {0.00\height} {0.08\width} {0.05\height},clip,width=0.31\linewidth, valign=t}{
    \includegraphics[width=1.0\linewidth]{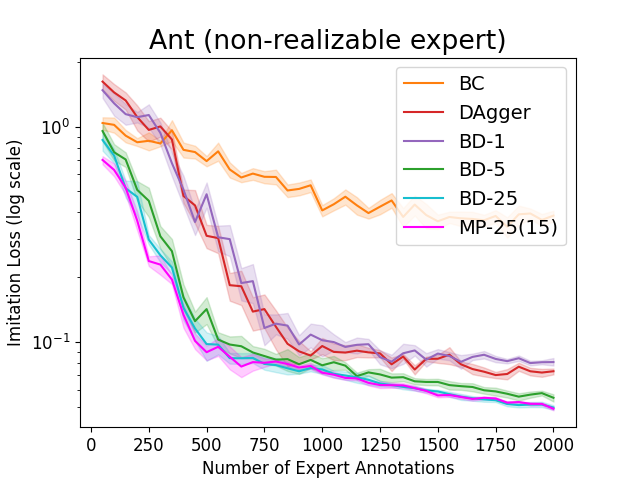}
  }
  \hspace{0.02\linewidth}
  \adjustbox{trim={0.00\width} {0.00\height} {0.08\width} {0.05\height},clip,width=0.31\linewidth, valign=t}{
    \includegraphics[width=1.0\linewidth]{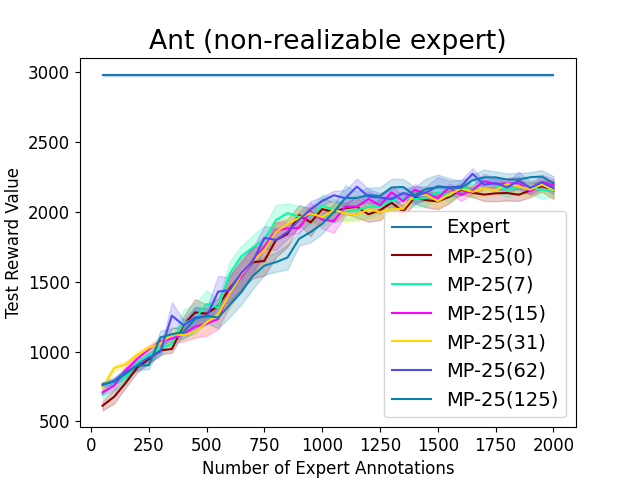}
  }

  \vspace{.1in}

  \adjustbox{trim={0.00\width} {0.00\height} {0.08\width} {0.05\height},clip,width=0.31\linewidth, valign=t}{
    \includegraphics[width=1.0\linewidth]{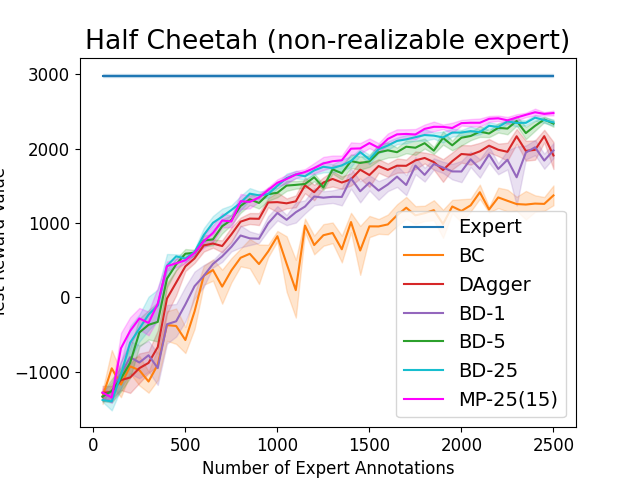}
  }
  \hspace{0.02\linewidth}
  \adjustbox{trim={0.00\width} {0.00\height} {0.08\width} {0.05\height},clip,width=0.31\linewidth, valign=t}{
    \includegraphics[width=1.0\linewidth]{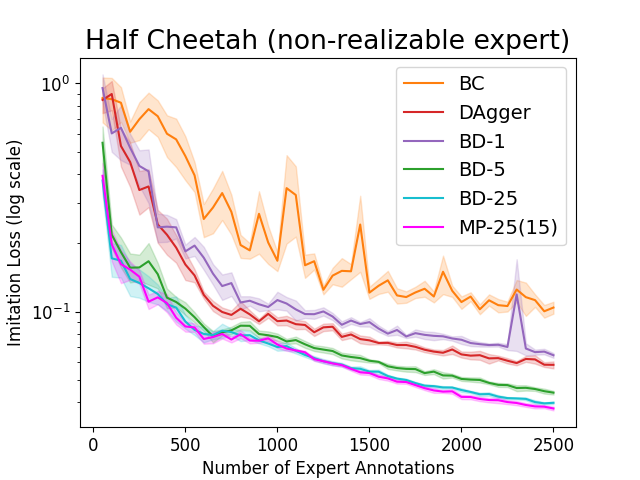}
  }
  \hspace{0.02\linewidth}
  \adjustbox{trim={0.00\width} {0.00\height} {0.08\width} {0.05\height},clip,width=0.31\linewidth, valign=t}{
    \includegraphics[width=1.0\linewidth]{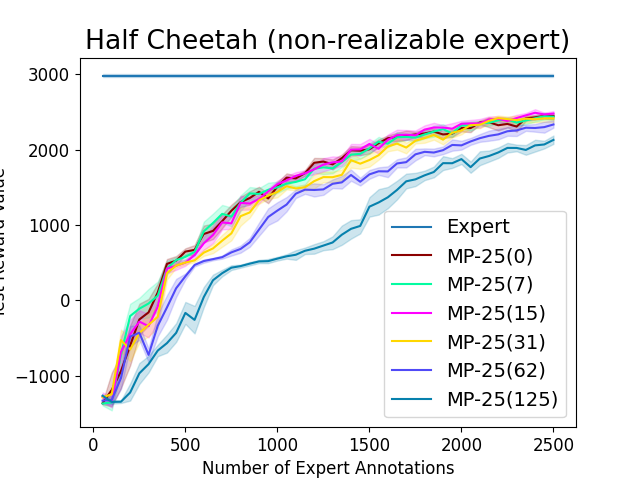}
  }

  \vspace{.1in}

  \adjustbox{trim={0.00\width} {0.00\height} {0.08\width} {0.05\height},clip,width=0.31\linewidth, valign=t}{
    \includegraphics[width=1.0\linewidth]{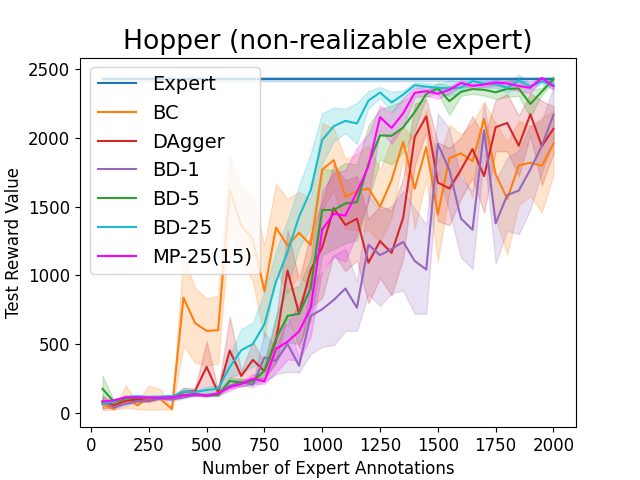}
  }
  \hspace{0.02\linewidth}
  \adjustbox{trim={0.00\width} {0.00\height} {0.08\width} {0.05\height},clip,width=0.31\linewidth, valign=t}{
    \includegraphics[width=1.0\linewidth]{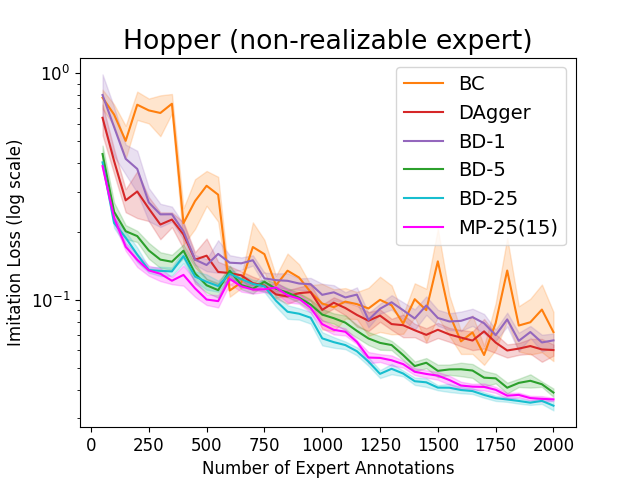}
  }
  \hspace{0.02\linewidth}
  \adjustbox{trim={0.00\width} {0.00\height} {0.08\width} {0.05\height},clip,width=0.31\linewidth, valign=t}{
    \includegraphics[width=1.0\linewidth]{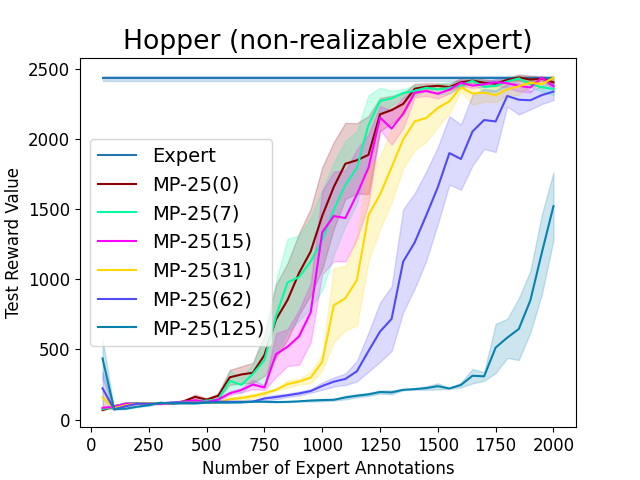}
  }

  \vspace{.1in}

  \adjustbox{trim={0.00\width} {0.00\height} {0.08\width} {0.05\height},clip,width=0.31\linewidth, valign=t}{
    \includegraphics[width=1.0\linewidth]{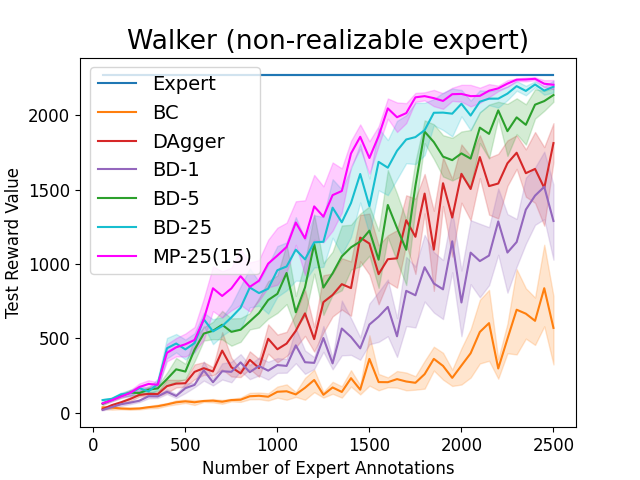}
  }
  \hspace{0.02\linewidth}
  \adjustbox{trim={0.00\width} {0.00\height} {0.08\width} {0.05\height},clip,width=0.31\linewidth, valign=t}{
    \includegraphics[width=1.0\linewidth]{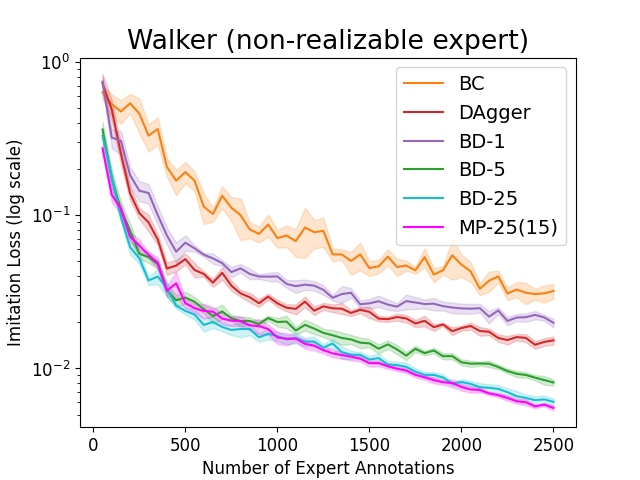}
  }
  \hspace{0.02\linewidth}
  \adjustbox{trim={0.00\width} {0.00\height} {0.08\width} {0.05\height},clip,width=0.31\linewidth, valign=t}{
    \includegraphics[width=1.0\linewidth]{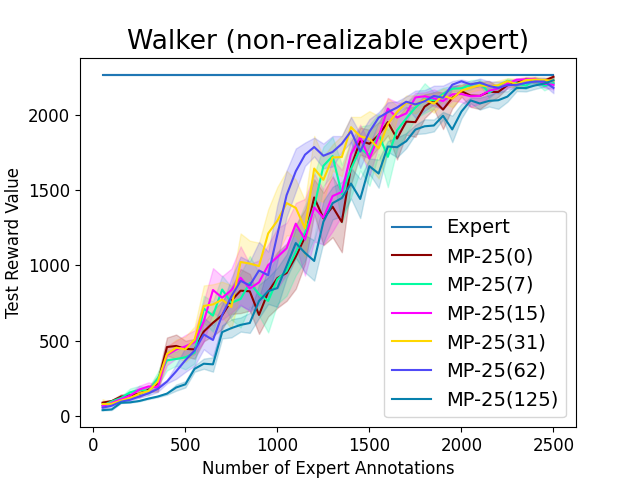}
  }

  \caption{Comparison between algorithms with non-realizable noisy expert.}
  \label{fig:nonrealizable_full}
\end{figure*}






\clearpage

\subsection{Full results from Section~\ref{sec:benefit_of_ensemble} and Data Visualization via t-SNE~\cite{van2008visualizing}.}
\label{sec: tsne}
In this section, we present all result plots from Section~\ref{sec:benefit_of_ensemble}, including those omitted due to space constraints, as shown in Figure~\ref{fig:ablation_full}.

\begin{figure}[h]
  \centering
  \includegraphics[width=1\linewidth]{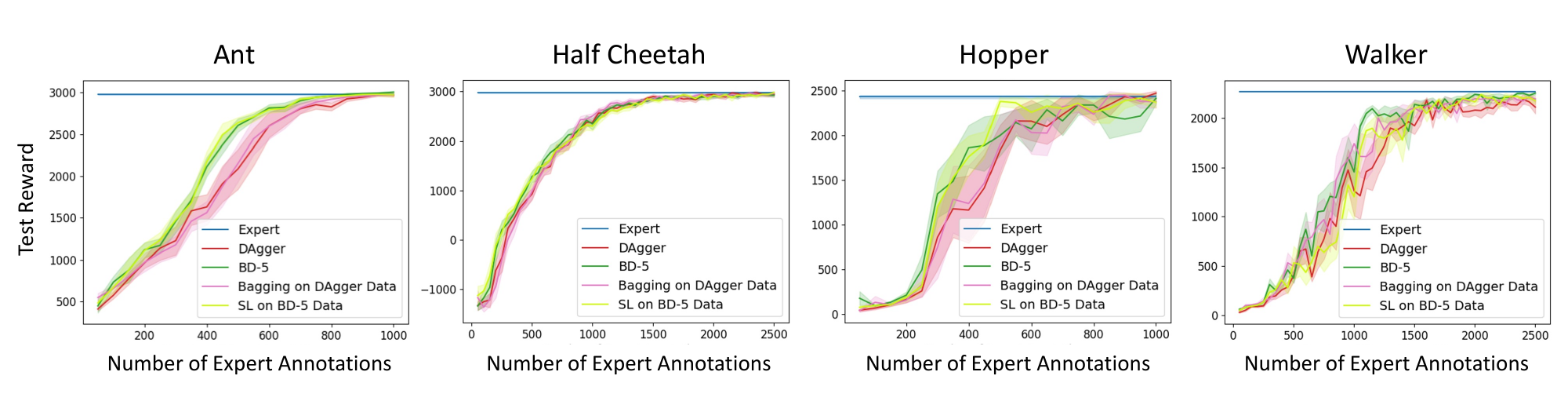}
  \caption{
  Full results on comparing \bdb and \dagge, along with the two additional approaches in Section~\ref{sec:benefit_of_ensemble}.
  }
  \label{fig:ablation_full}
\end{figure}

To further understand how the data quality collected by \bdb improves over \dagge, we visualize the states collected and queried by different algorithms in Section~\ref{sec:main_experiment} via t-SNE. As motivation, Figure~\ref{fig:expert-dagger-states} shows a comparison between states of offline expert demonstration and states queried by \dagge in Hopper. It can be seen that with the same expert annotation budget, \dagge collects a dataset that encompasses a broader support compared to the expert, while the policy trained over it achieves a higher average reward.

\begin{figure}[h]
  \vspace{.1in}
  \centering
\includegraphics[width=0.6\linewidth]{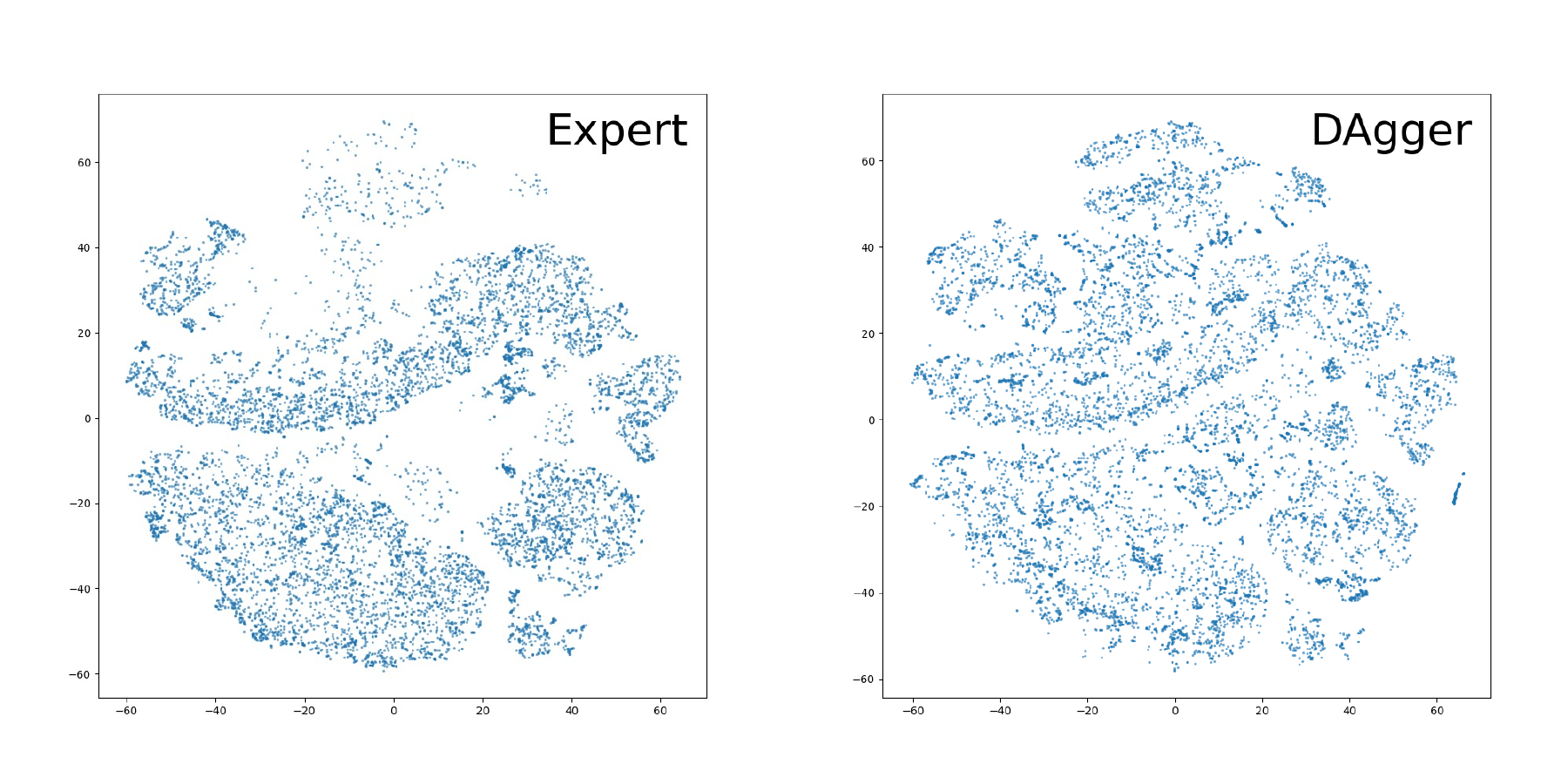}
  \caption{Two-dimensional t-SNE visualization of states collected by expert and \dagge in continuous control task Hopper, using the same mapping.
  It can be observed that the support of state distribution by \dagge contains regions (top and middle) that are not supported by the expert's state distribution. 
  Over 10 repeated trials, supervised learning over the datasets collected by expert and \dagge have average reward of 1320 and 2470, respectively.
  }
  
  \label{fig:expert-dagger-states}
\end{figure}

Figure~\ref{fig:ant_tsne},\ref{fig:half_cheetah_tsne},\ref{fig:hopper_tsne},\ref{fig:walker_tsne} showcase the t-SNE visualization of states obtained by different algorithms across four environments under the realizable expert setting, using the same mapping for the same environment.
State points are color-mapped from blue to red based on their arrival rounds. 
As presented in these figures, the observations reaffirm the findings of Section~\ref{sec:benefit_of_ensemble}.
For example, in the state visualization of Ant (Figure~\ref{fig:ant_tsne}), we notice similar state coverage among \dagge-style algorithms, which is distinct from the expert's distribution. This suggests that \bdb may not collect annotations over different state distributions than
\dagge. Meanwhile, the color of points within the zoomed-in area for \bdb, \bdc, and \MPF appears bluer than \dagge, indicating a more efficient exploration by ensembles 
in regions beyond the support of the expert's state distribution.
From these results, we can see that \bd actively explores the state space, swiftly adapting to and rectifying its errors, ensuring a more rapid and efficient learning process compared to \dagge.

\begin{figure*}[h]
  \centering

  \adjustbox{trim={0.05\width} {0.05\height} {0.05\width} {0.05\height},clip,width=0.40\linewidth, valign=t}{
    \includegraphics[width=\linewidth]{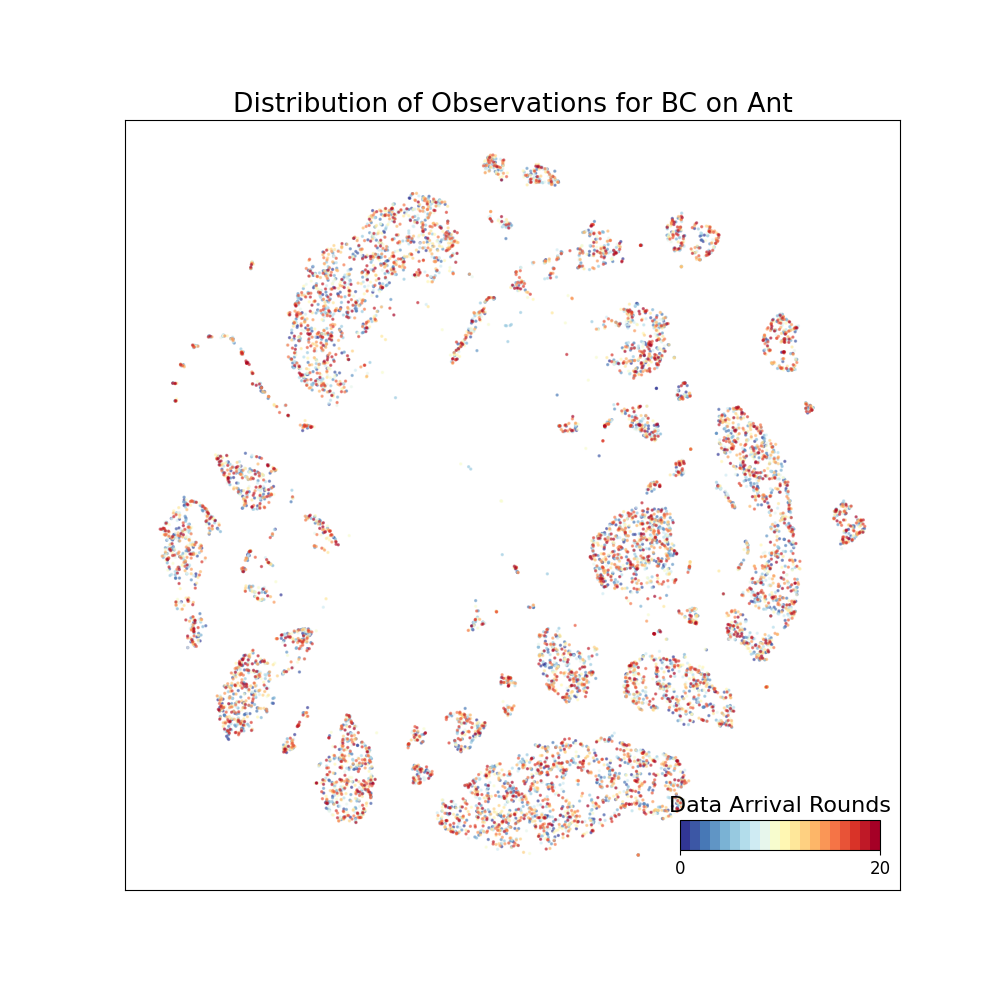}
  }
  \hspace{0.015\linewidth}
    \adjustbox{trim={0.05\width} {0.05\height} {0.05\width} {0.05\height},clip,width=0.40\linewidth, valign=t}{
    \includegraphics[width=\linewidth]{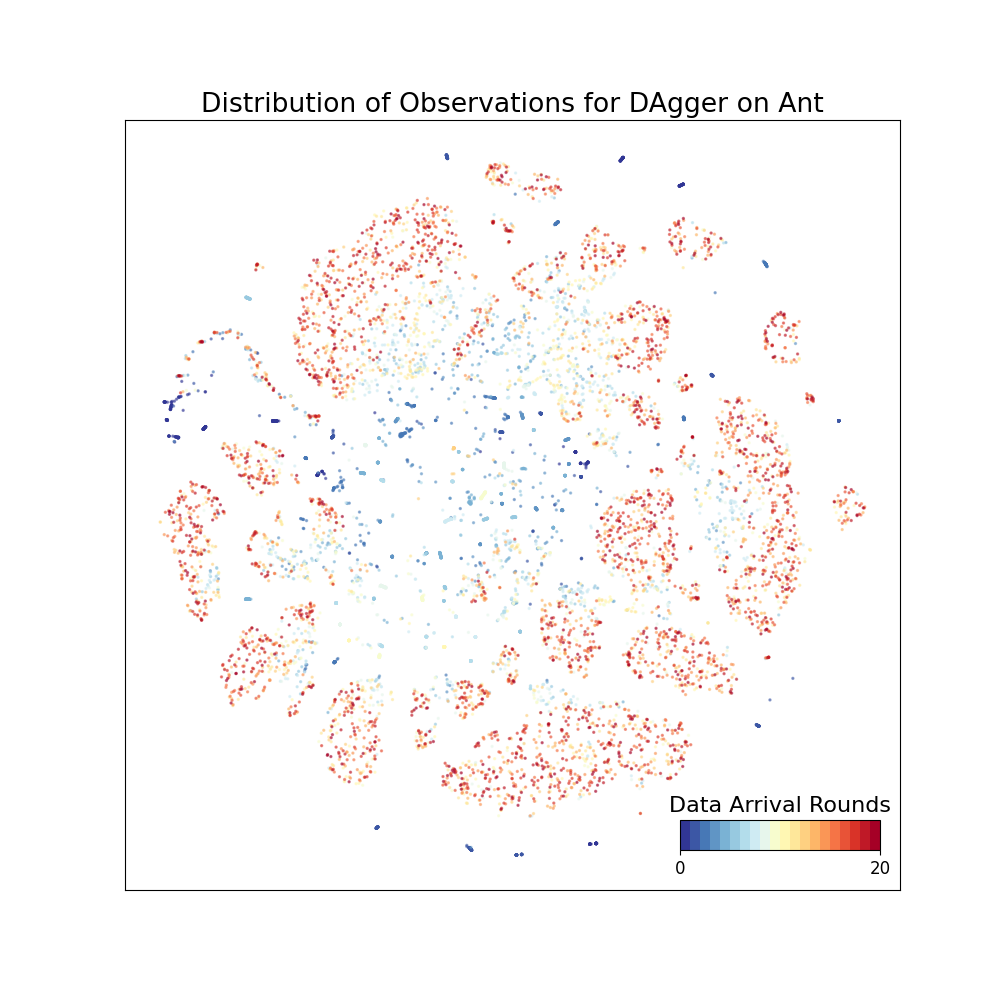}
  }
  \adjustbox{trim={0.05\width} {0.05\height} {0.05\width} {0.05\height},clip,width=0.40\linewidth, valign=t}{
    \includegraphics[width=\linewidth]{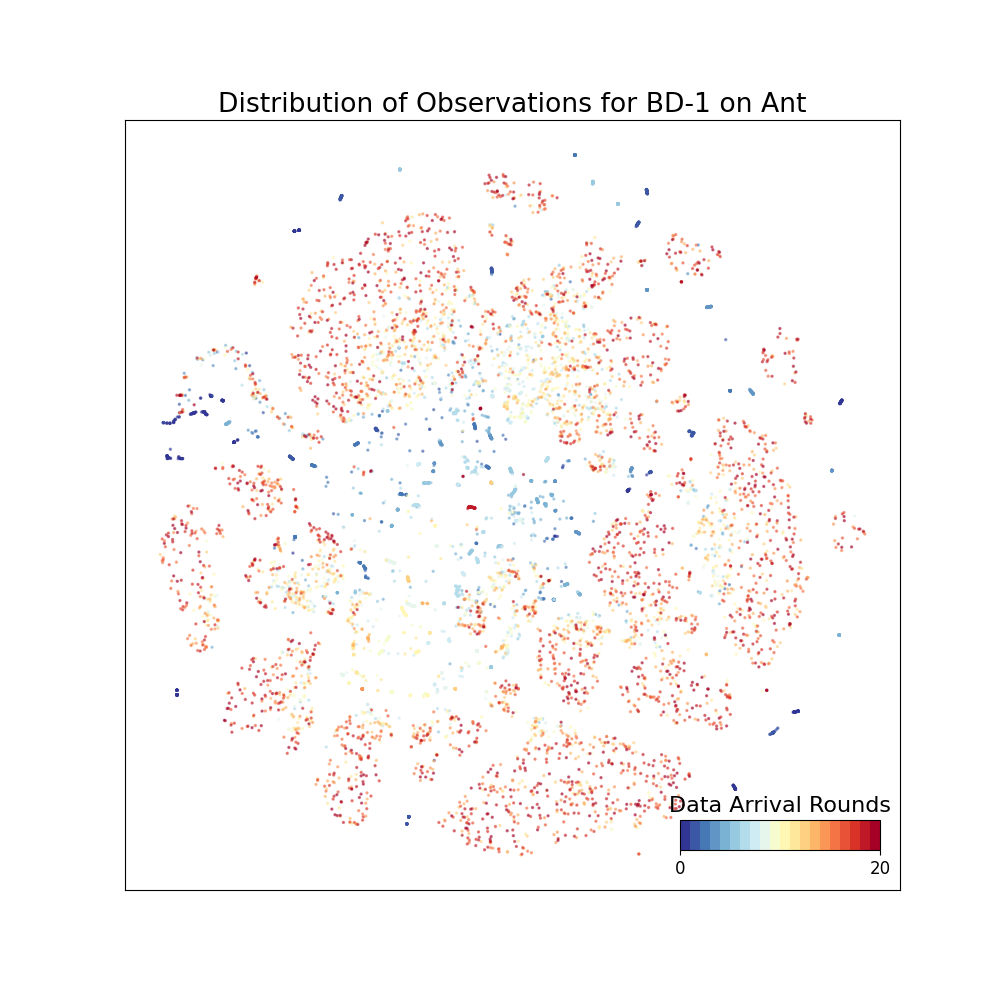}
  }
  \hspace{0.015\linewidth}
  \adjustbox{trim={0.05\width} {0.05\height} {0.05\width} {0.05\height},clip,width=0.40\linewidth, valign=t}{
    \includegraphics[width=\linewidth]{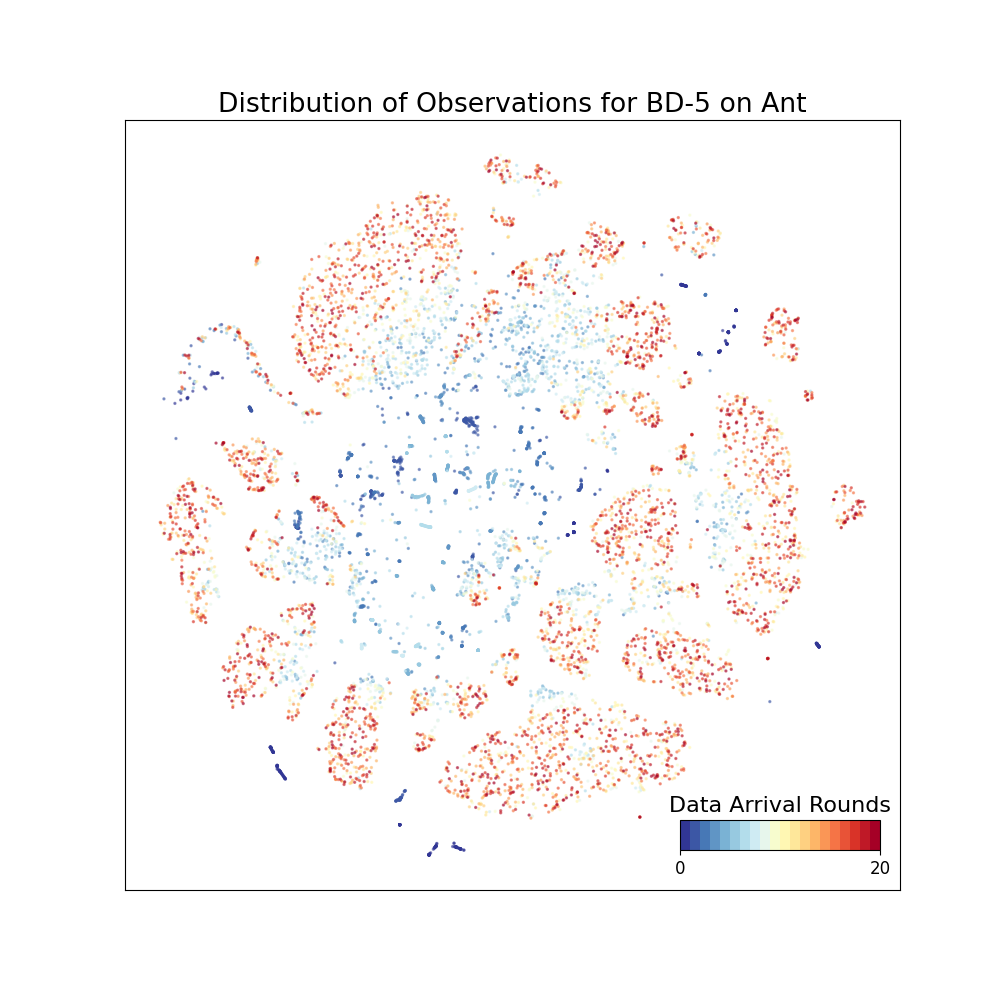}
  }


  \adjustbox{trim={0.05\width} {0.05\height} {0.05\width} {0.05\height},clip,width=0.40\linewidth, valign=t}{
    \includegraphics[width=\linewidth]{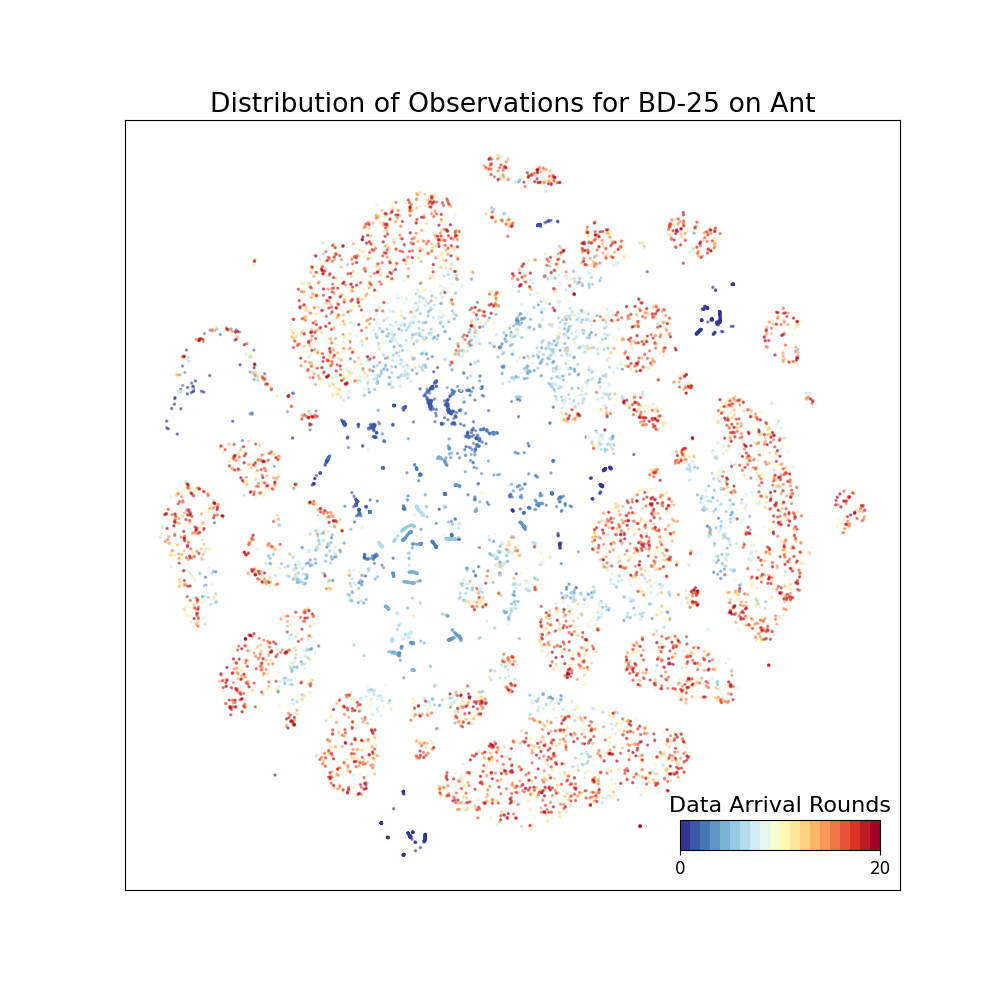}
  }
\hspace{0.015\linewidth}
\adjustbox{trim={0.05\width} {0.05\height} {0.05\width} {0.05\height},clip,width=0.40\linewidth, valign=t}{
    \includegraphics[width=\linewidth]{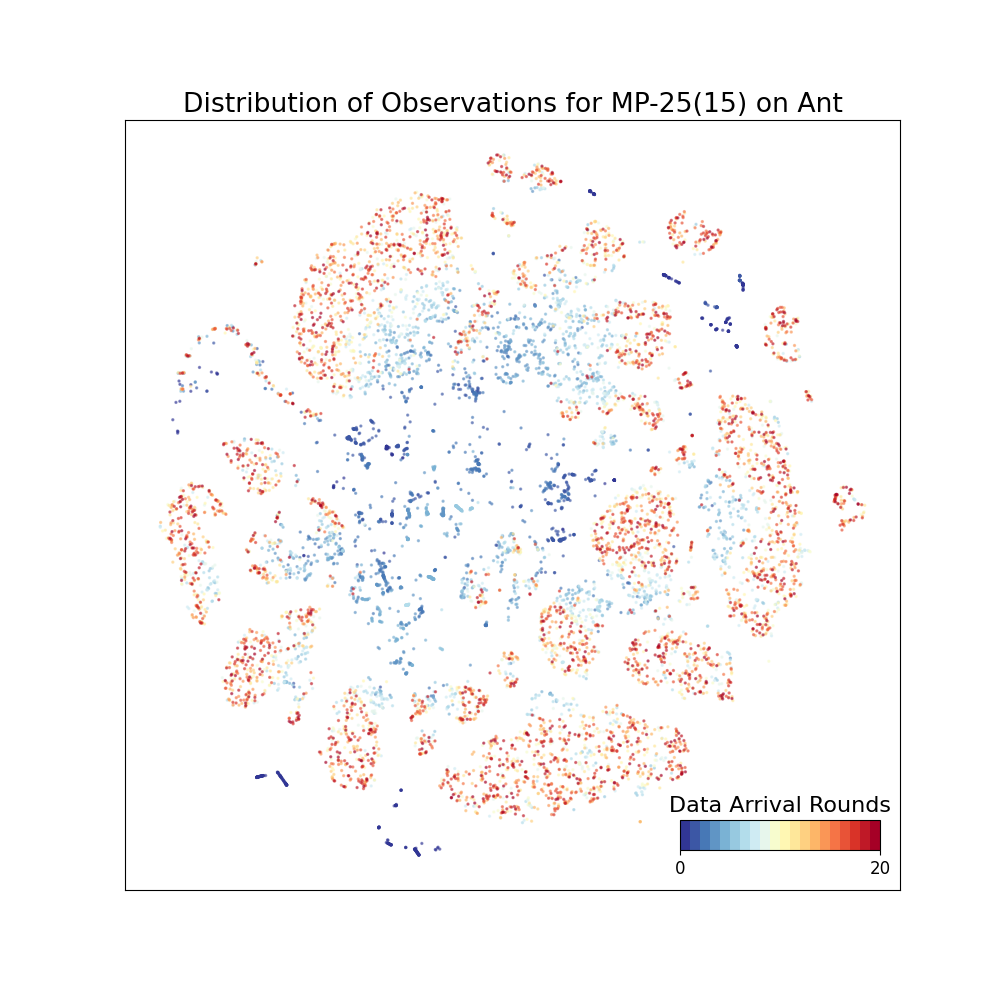}
  }
  \caption{Two-dimensional t-SNE visualizations of Ant environment ss collected by different algorithms.}
  \label{fig:ant_tsne}
\end{figure*}

\begin{figure*}[h]
  \centering

  \adjustbox{trim={0.05\width} {0.05\height} {0.05\width} {0.05\height},clip,width=0.40\linewidth, valign=t}{
    \includegraphics[width=\linewidth]{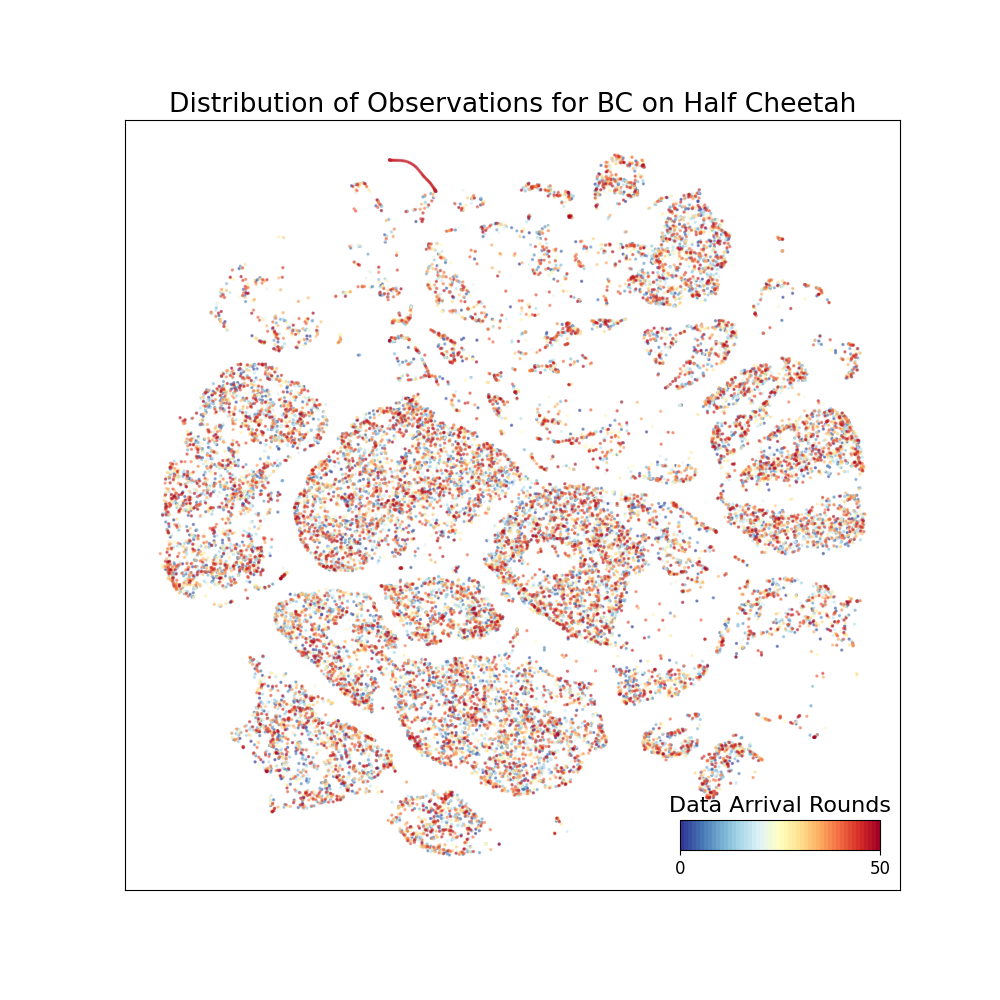}
  }
  \hspace{0.015\linewidth}
    \adjustbox{trim={0.05\width} {0.05\height} {0.05\width} {0.05\height},clip,width=0.40\linewidth, valign=t}{
    \includegraphics[width=\linewidth]{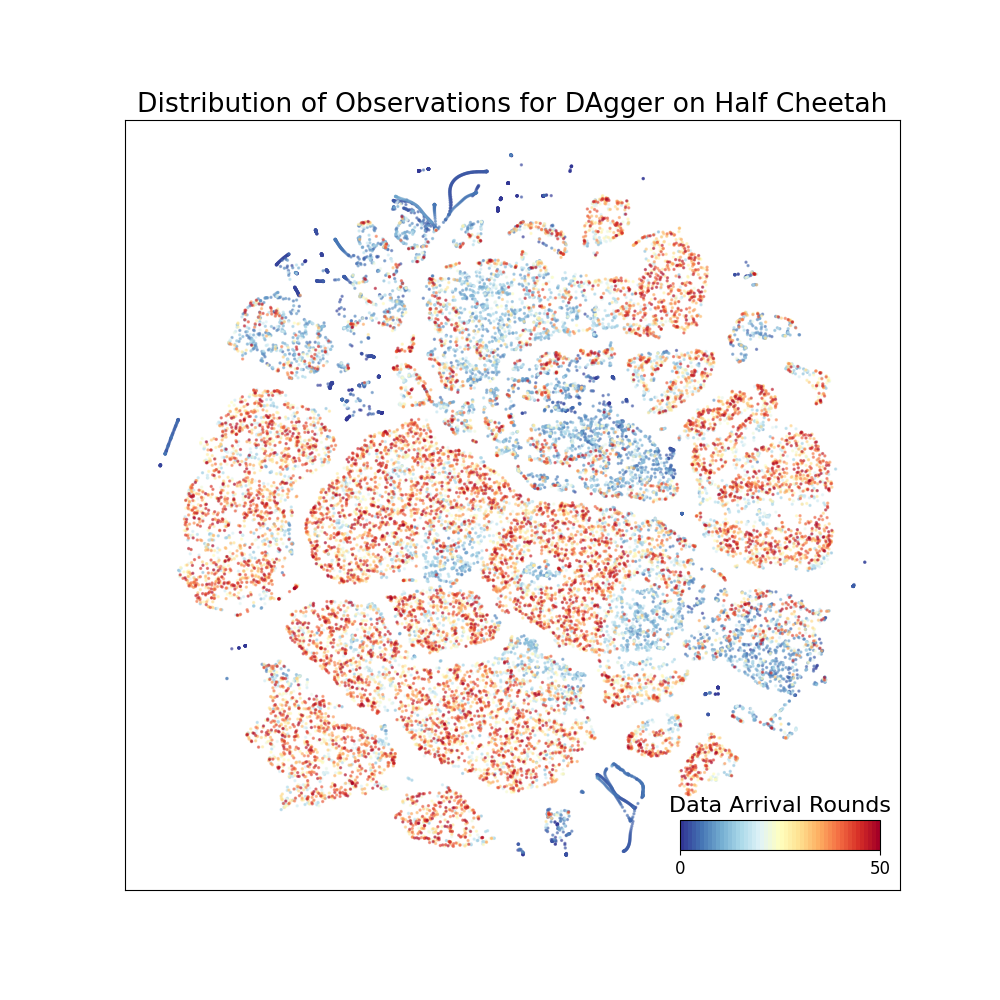}
  }

  \adjustbox{trim={0.05\width} {0.05\height} {0.05\width} {0.05\height},clip,width=0.40\linewidth, valign=t}{
    \includegraphics[width=\linewidth]{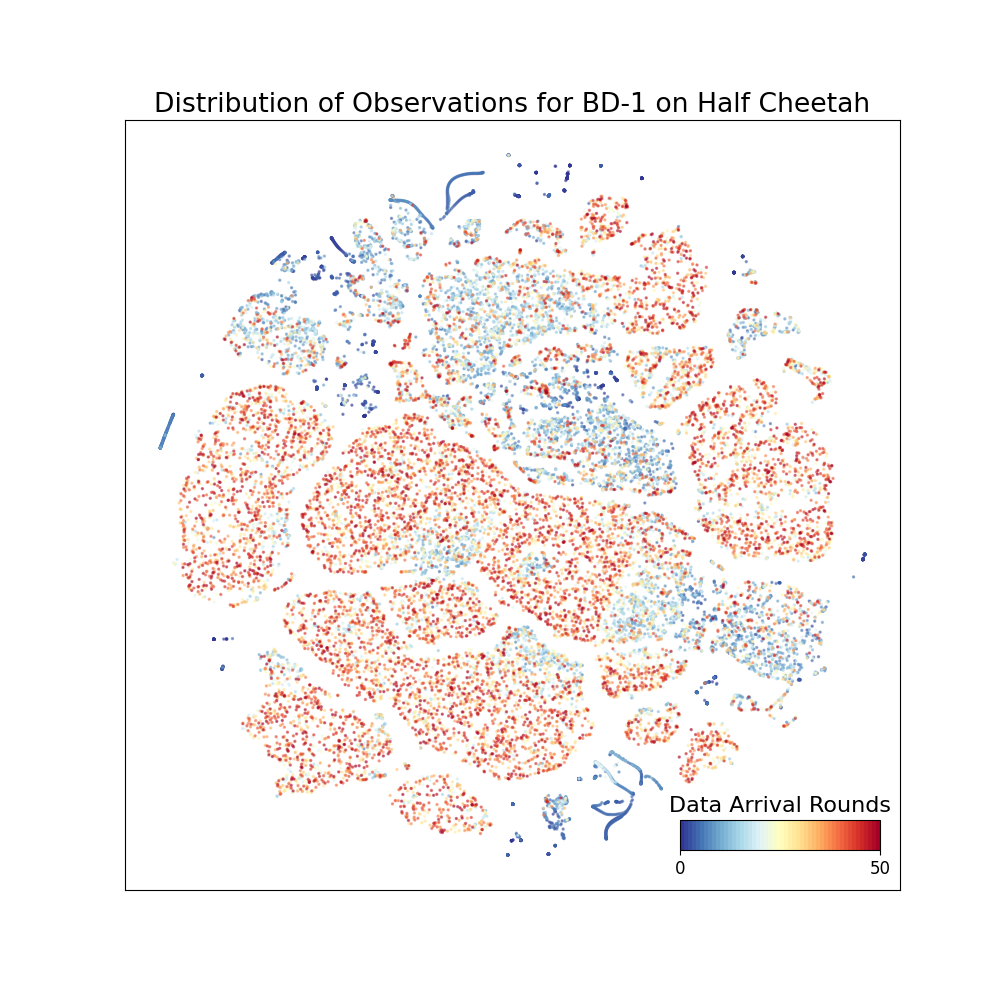}
  }
\hspace{0.015\linewidth}
\adjustbox{trim={0.05\width} {0.05\height} {0.05\width} {0.05\height},clip,width=0.40\linewidth, valign=t}{
    \includegraphics[width=\linewidth]{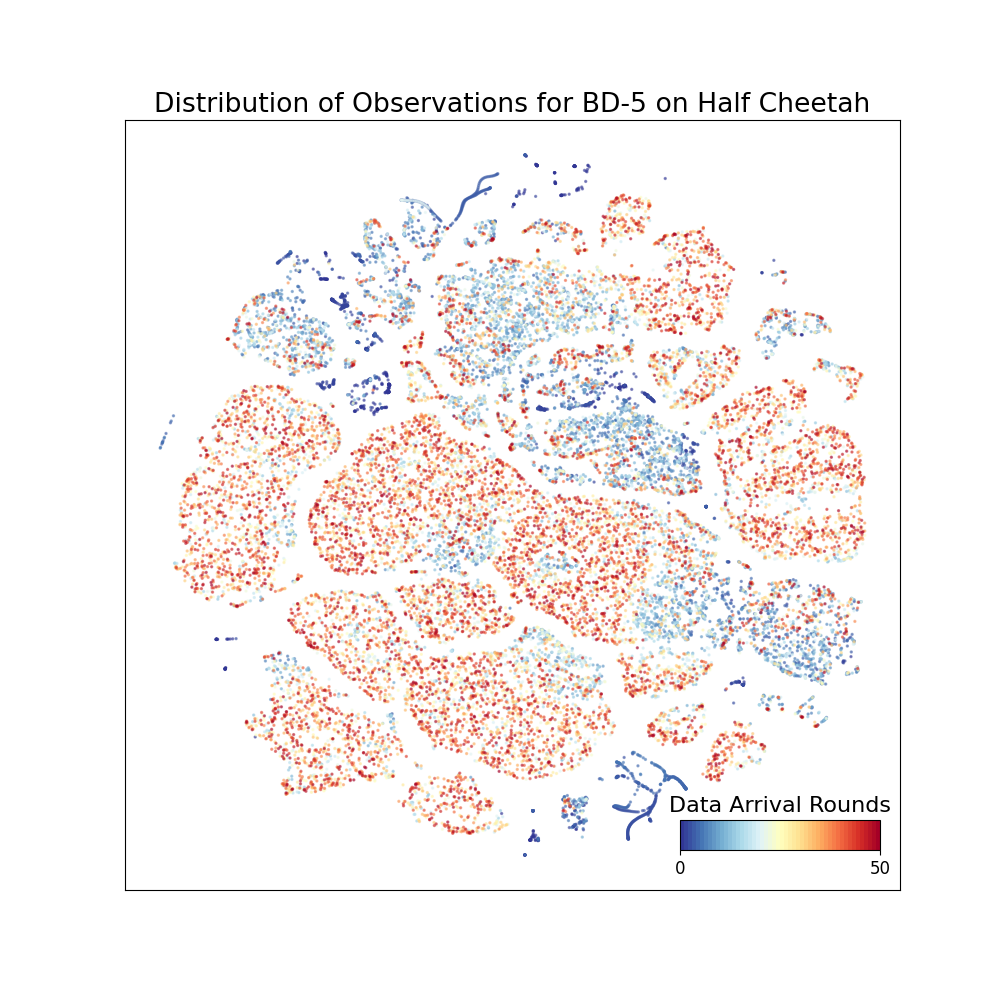}
  }

  
  \adjustbox{trim={0.05\width} {0.05\height} {0.05\width} {0.05\height},clip,width=0.40\linewidth, valign=t}{
    \includegraphics[width=\linewidth]{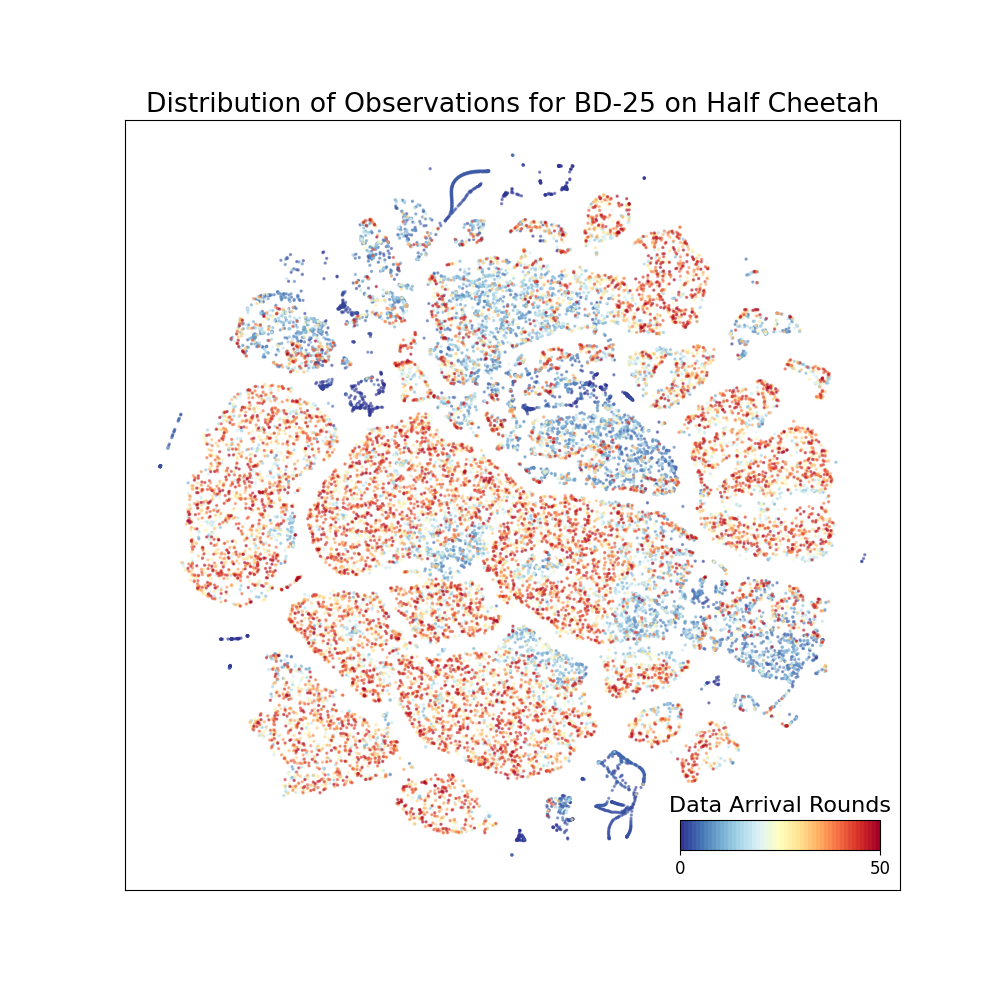}
  }
   \hspace{0.015\linewidth}
\adjustbox{trim={0.05\width} {0.05\height} {0.05\width} {0.05\height},clip,width=0.40\linewidth, valign=t}{
    \includegraphics[width=\linewidth]{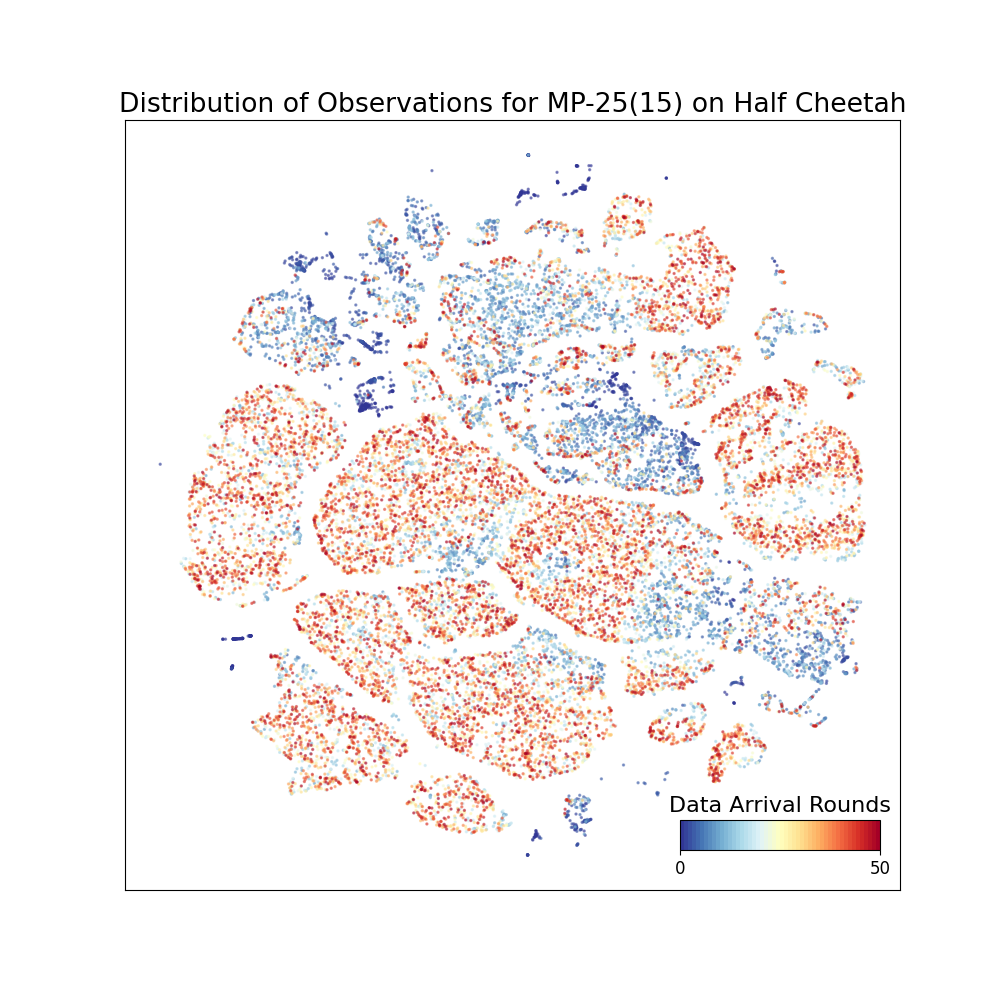}
  }
  \caption{Two-dimensional t-SNE visualizations of Half-Cheetah environment states collected by different algorithms.}
  \label{fig:half_cheetah_tsne}
\end{figure*}

\begin{figure*}[h]
  \centering

  \adjustbox{trim={0.05\width} {0.05\height} {0.05\width} {0.05\height},clip,width=0.40\linewidth, valign=t}{
    \includegraphics[width=\linewidth]{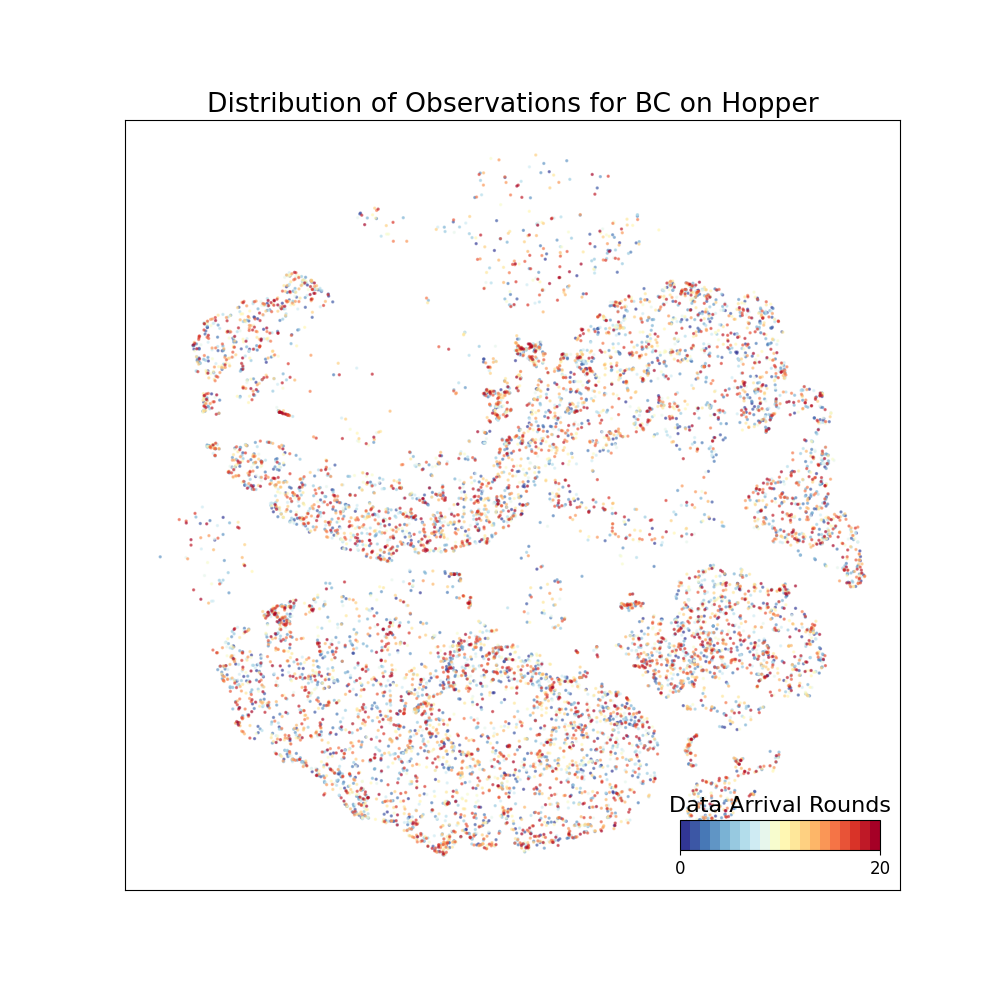}
  }
  \hspace{0.015\linewidth}
    \adjustbox{trim={0.05\width} {0.05\height} {0.05\width} {0.05\height},clip,width=0.40\linewidth, valign=t}{
    \includegraphics[width=\linewidth]{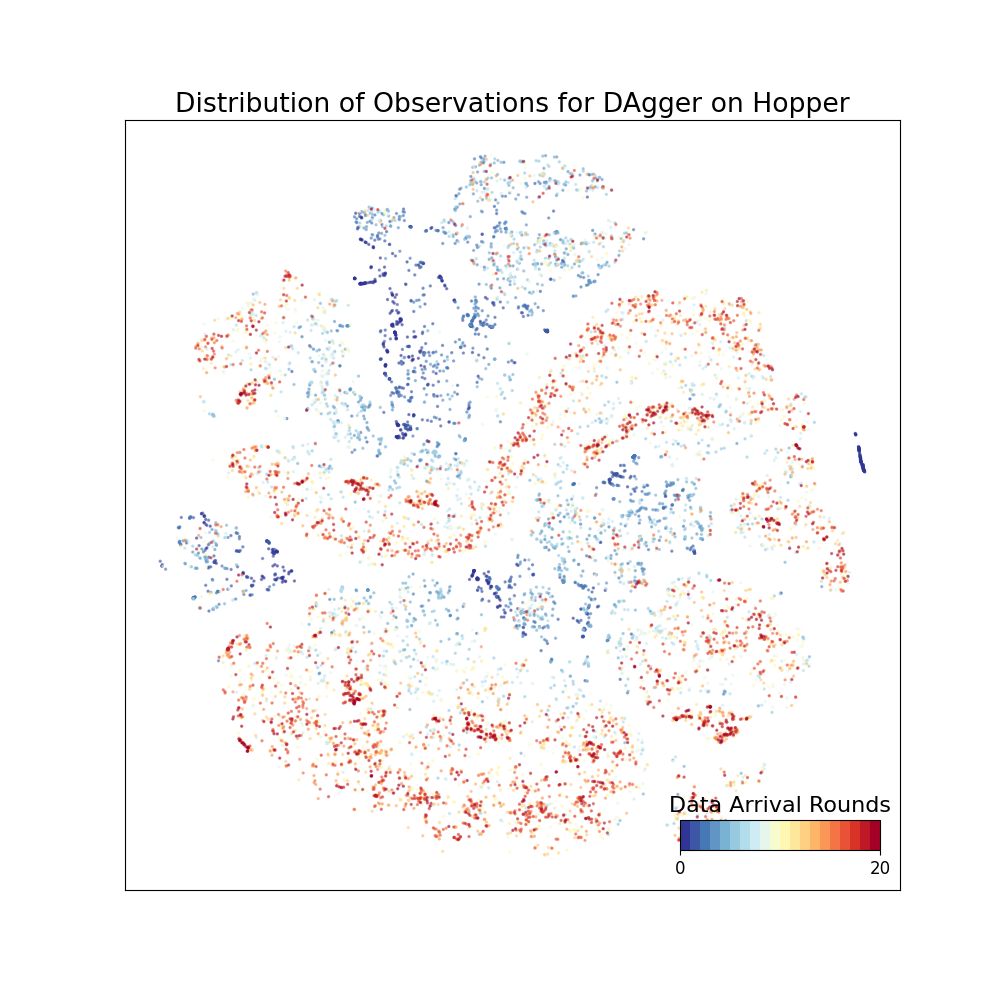}
  }
  
  \adjustbox{trim={0.05\width} {0.05\height} {0.05\width} {0.05\height},clip,width=0.40\linewidth, valign=t}{
    \includegraphics[width=\linewidth]{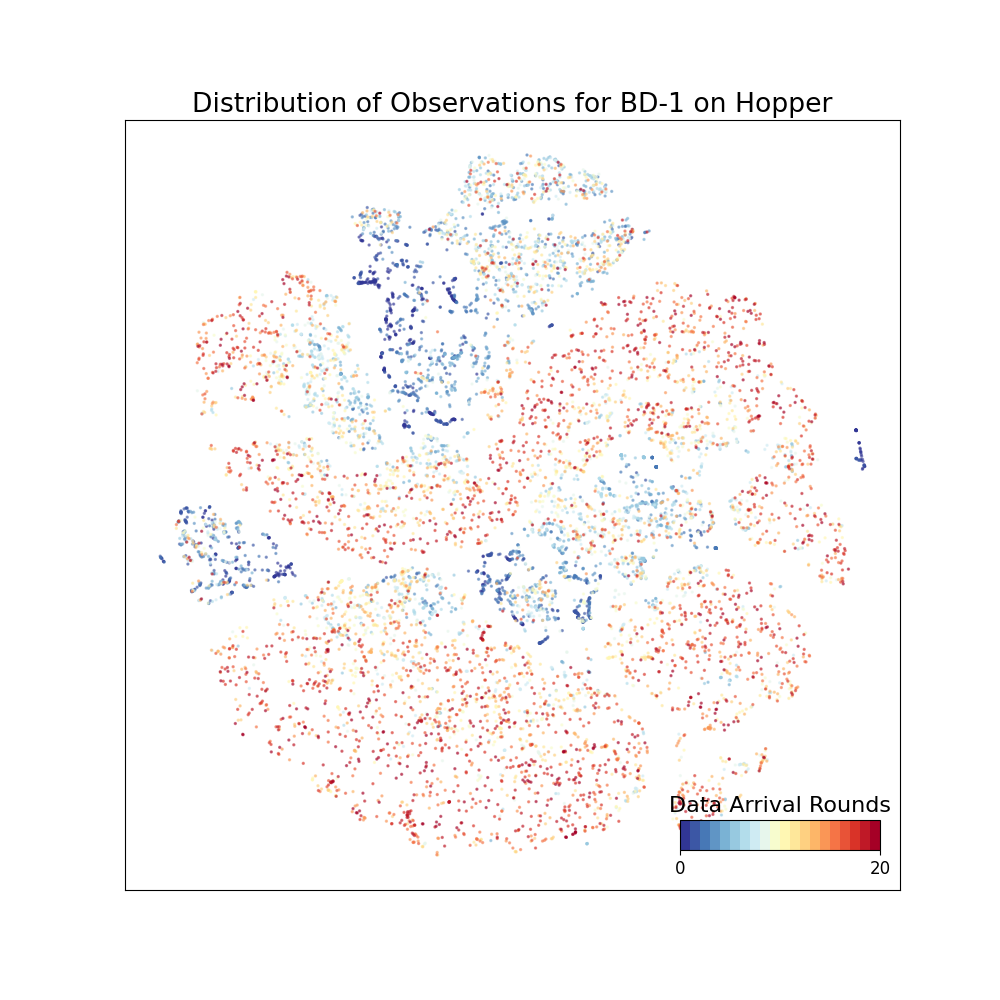}
  }
\hspace{0.015\linewidth}
\adjustbox{trim={0.05\width} {0.05\height} {0.05\width} {0.05\height},clip,width=0.40\linewidth, valign=t}{
    \includegraphics[width=\linewidth]{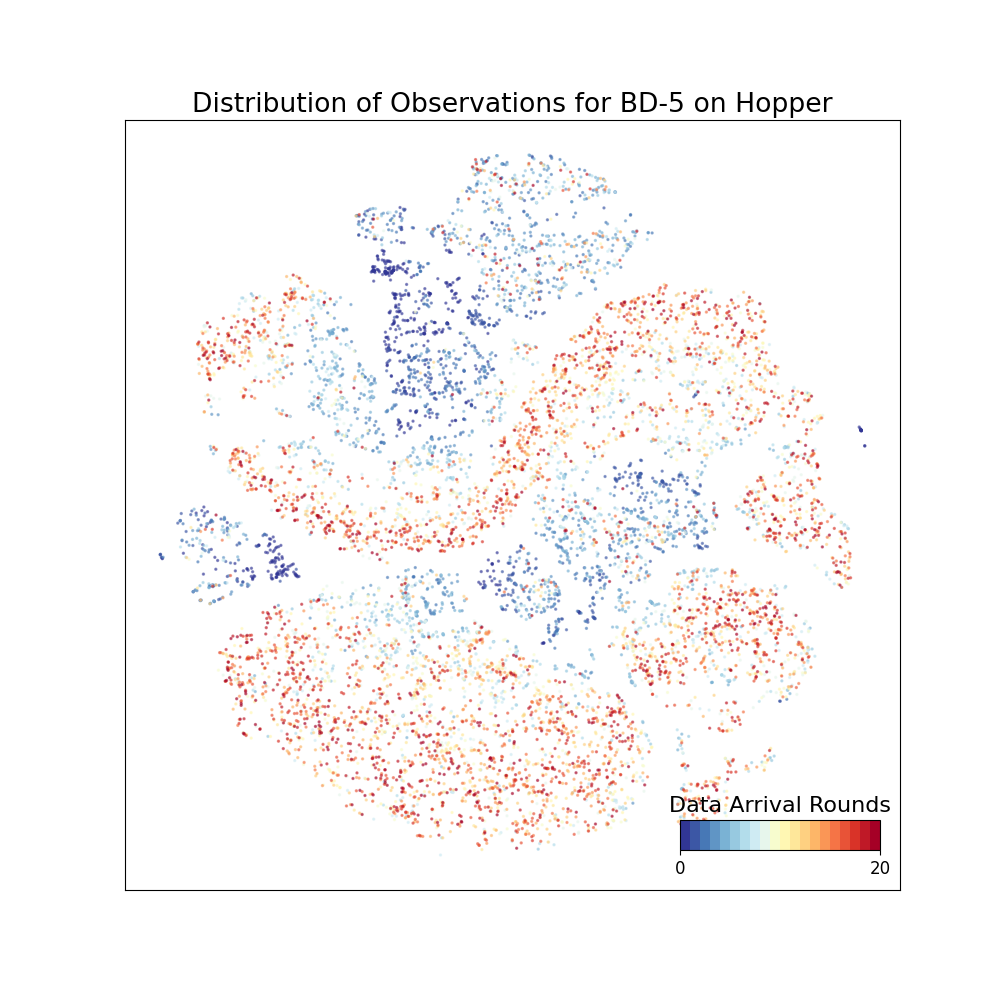}
  }

  
  \adjustbox{trim={0.05\width} {0.05\height} {0.05\width} {0.05\height},clip,width=0.40\linewidth, valign=t}{
    \includegraphics[width=\linewidth]{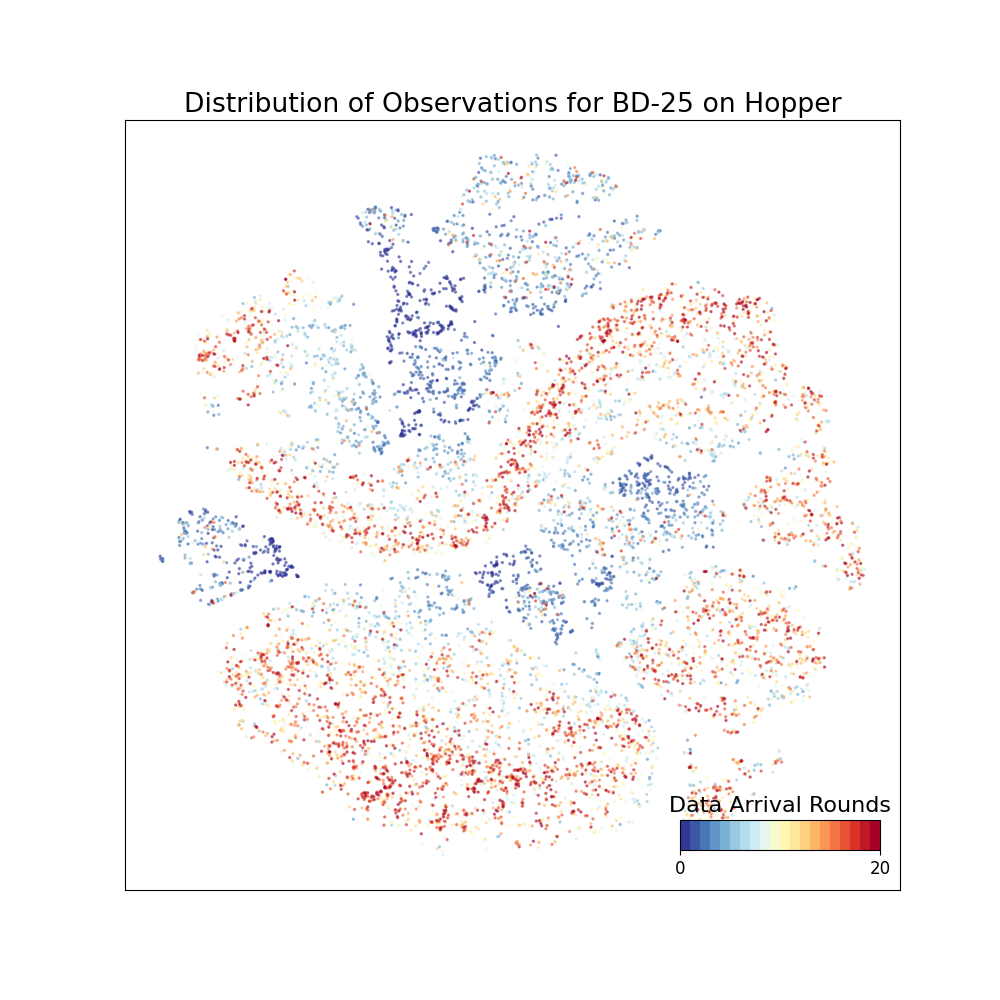}
  }
\hspace{0.015\linewidth}
\adjustbox{trim={0.05\width} {0.05\height} {0.05\width} {0.05\height},clip,width=0.40\linewidth, valign=t}{
    \includegraphics[width=\linewidth]{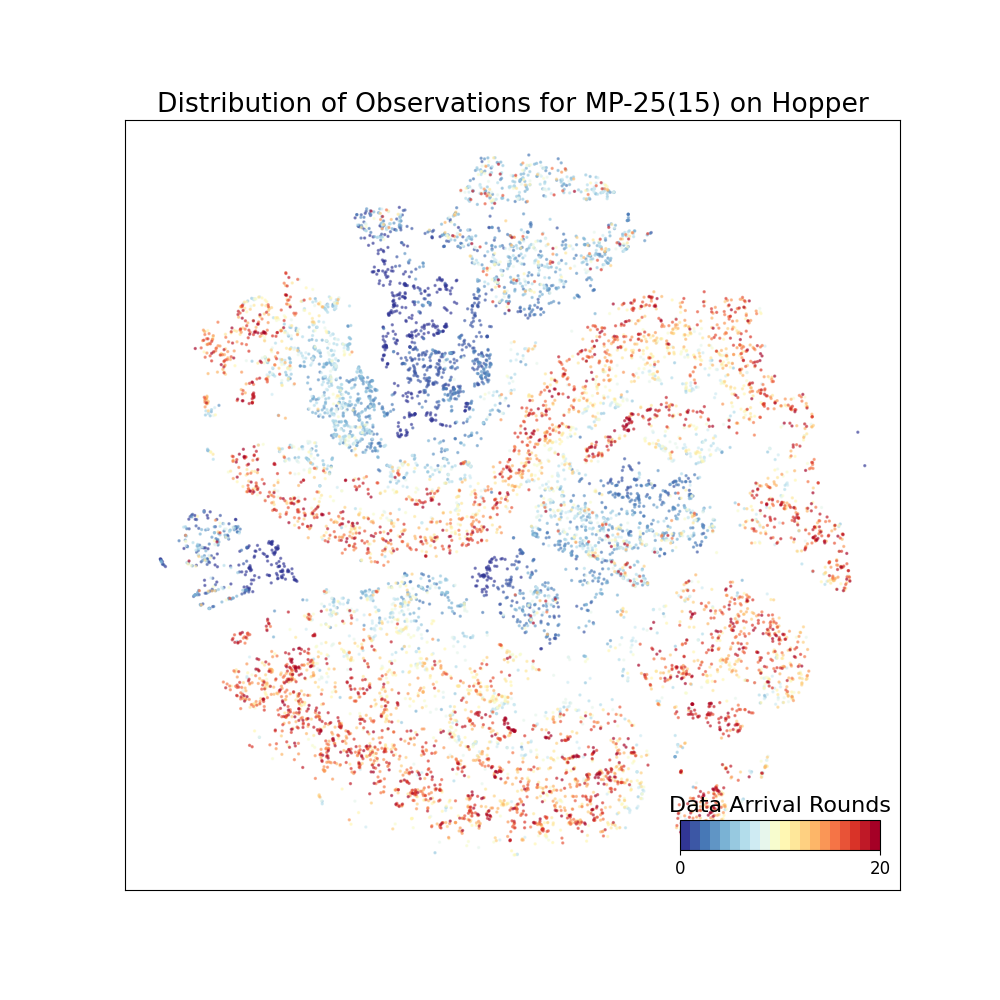}
  }

  \caption{Two-dimensional t-SNE visualizations of Hopper environment states collected by different algorithms.}
  \label{fig:hopper_tsne}
\end{figure*}

\begin{figure*}[h]
  \centering

  \adjustbox{trim={0.05\width} {0.05\height} {0.05\width} {0.05\height},clip,width=0.40\linewidth, valign=t}{
    \includegraphics[width=\linewidth]{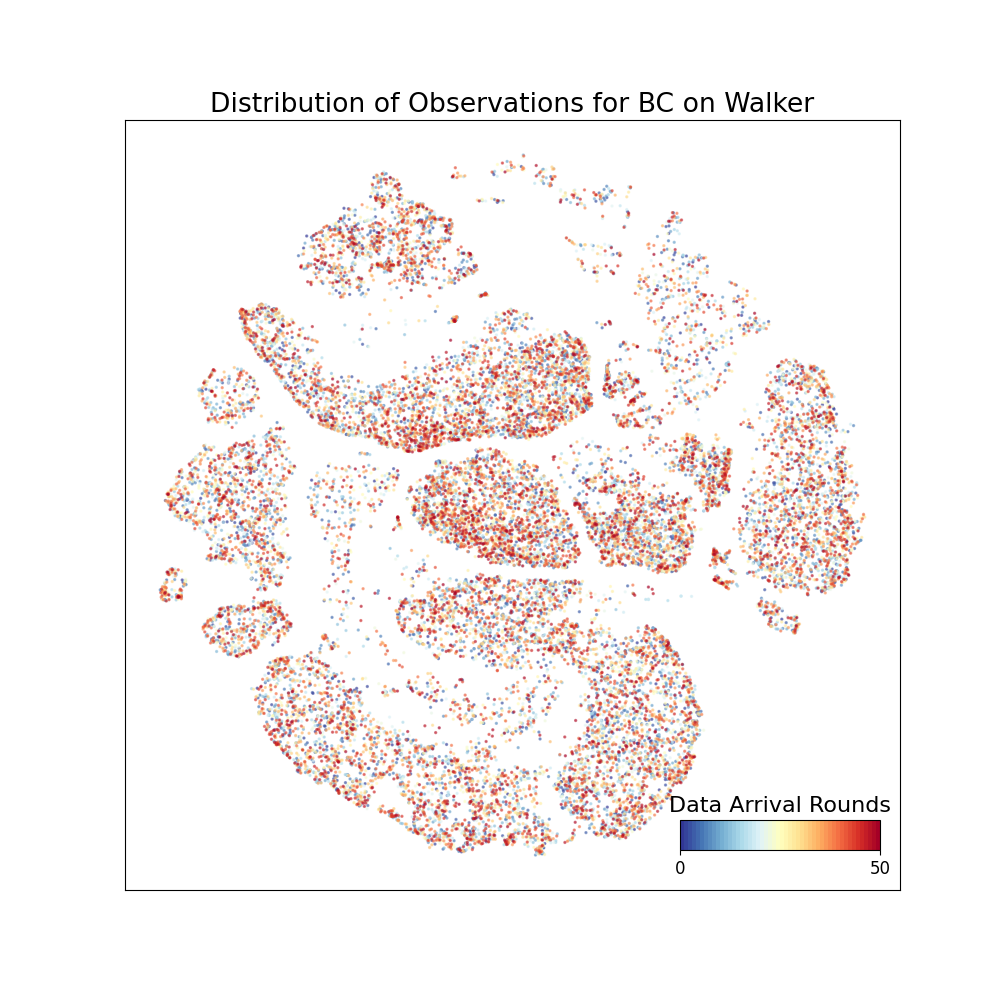}
  }
  \hspace{0.015\linewidth}
  \adjustbox{trim={0.05\width} {0.05\height} {0.05\width} {0.05\height},clip,width=0.40\linewidth, valign=t}{
    \includegraphics[width=\linewidth]{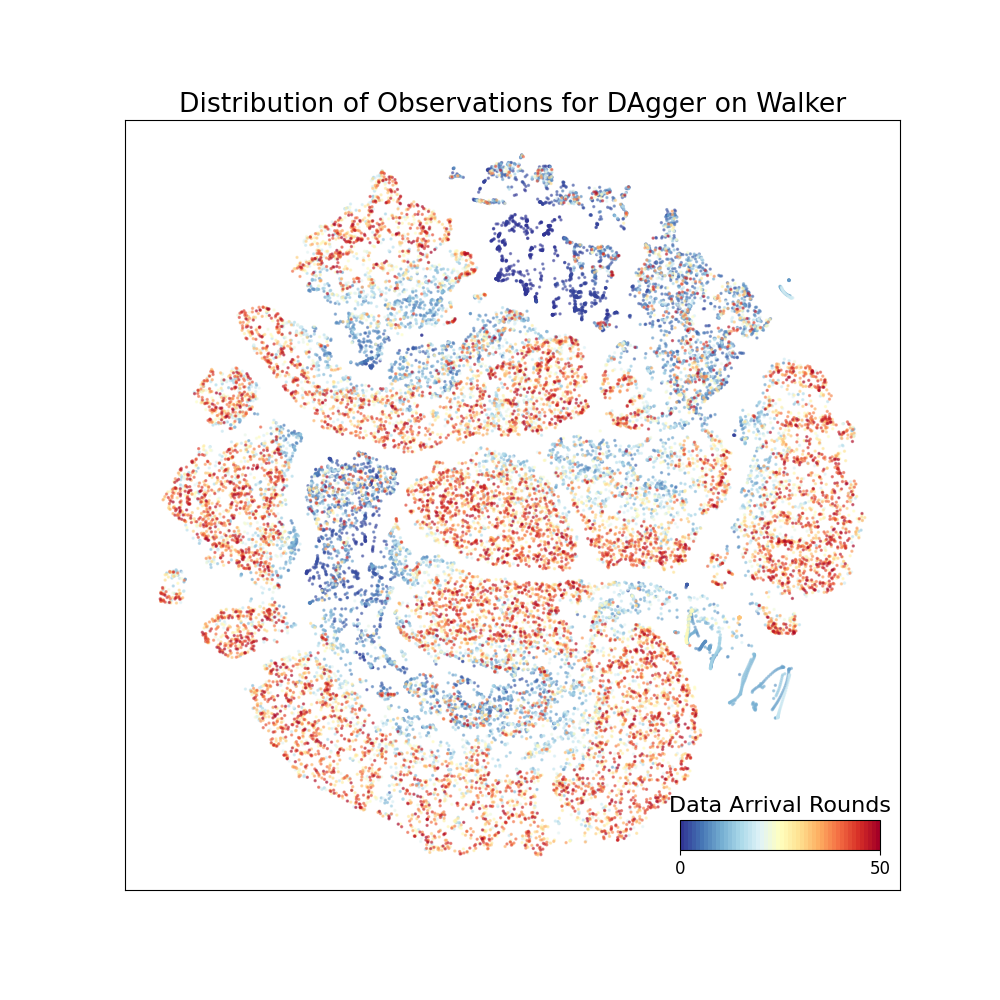}
  }  

  \adjustbox{trim={0.05\width} {0.05\height} {0.05\width} {0.05\height},clip,width=0.40\linewidth, valign=t}{
    \includegraphics[width=\linewidth]{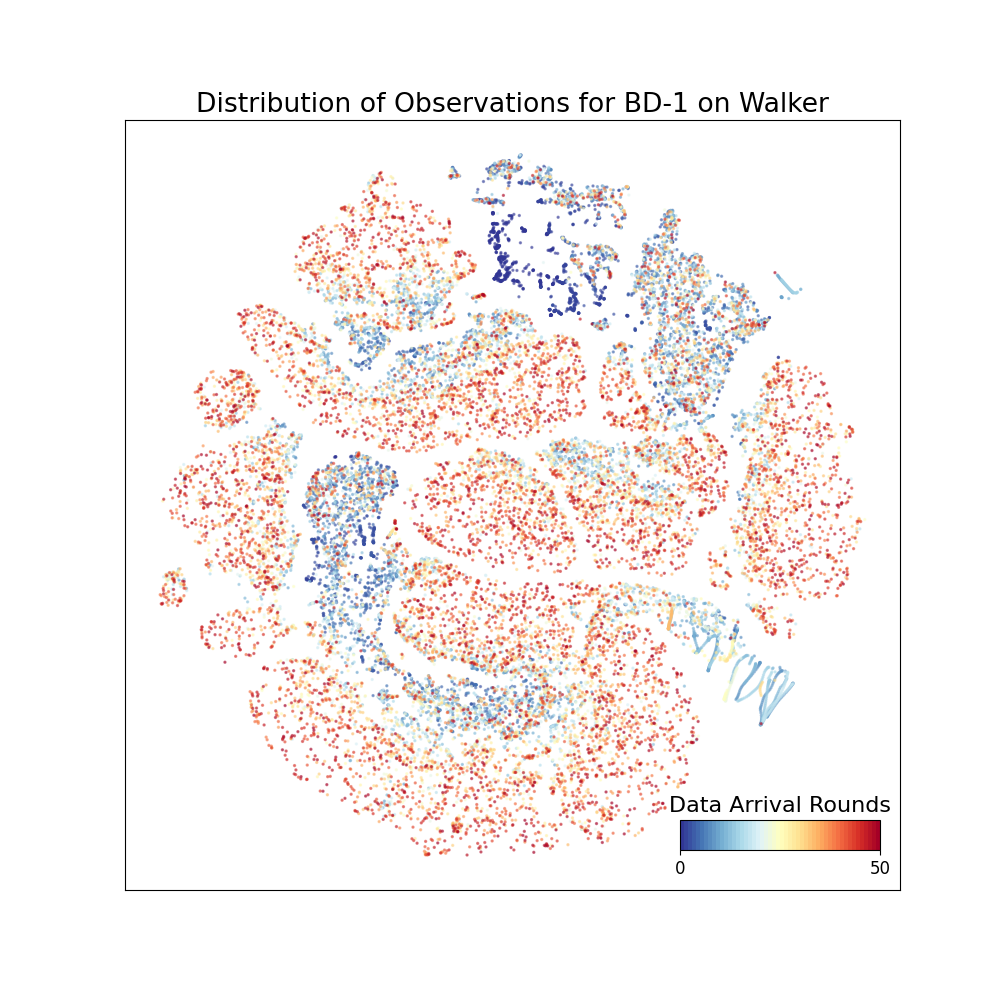}
  }
  \hspace{0.015\linewidth}
  \adjustbox{trim={0.05\width} {0.05\height} {0.05\width} {0.05\height},clip,width=0.40\linewidth, valign=t}{
    \includegraphics[width=\linewidth]{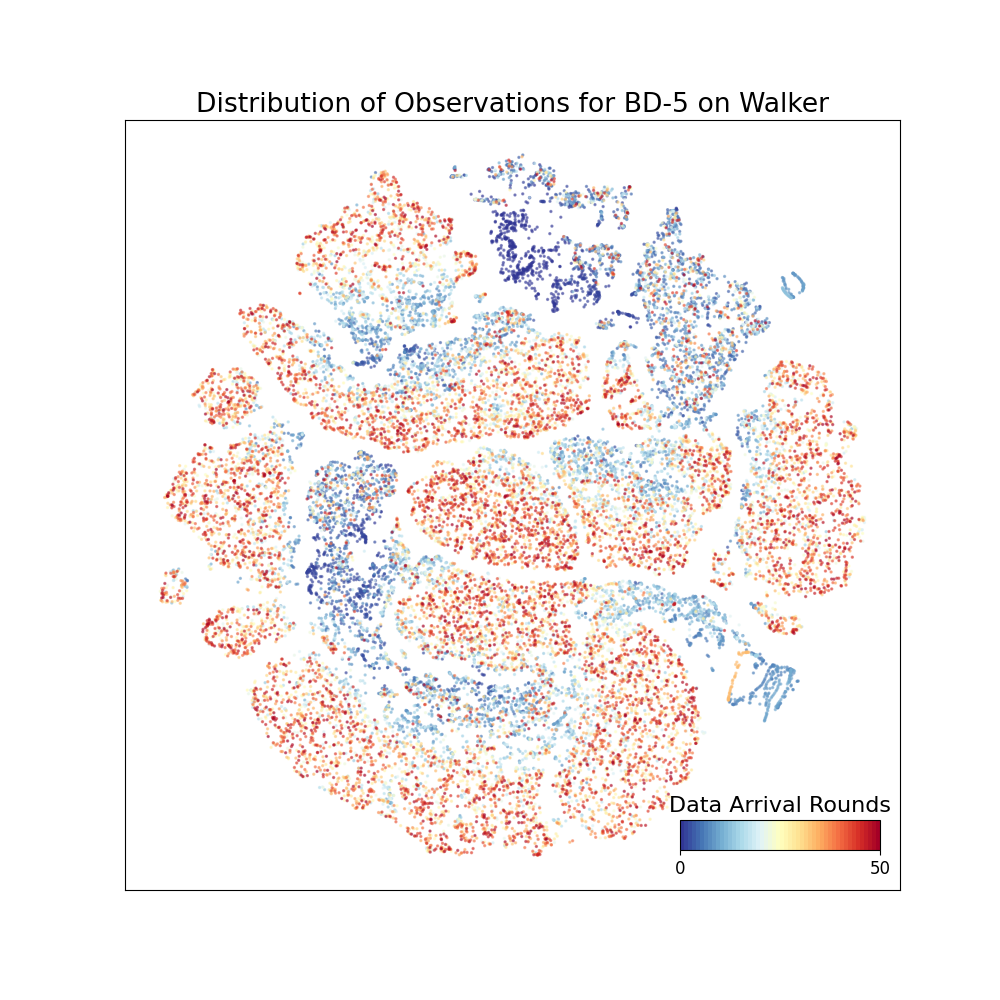}
  }
  \vspace{.01in}

  
  \hspace{0.015\linewidth}
  \adjustbox{trim={0.05\width} {0.05\height} {0.05\width} {0.05\height},clip,width=0.40\linewidth, valign=t}{
    \includegraphics[width=\linewidth]{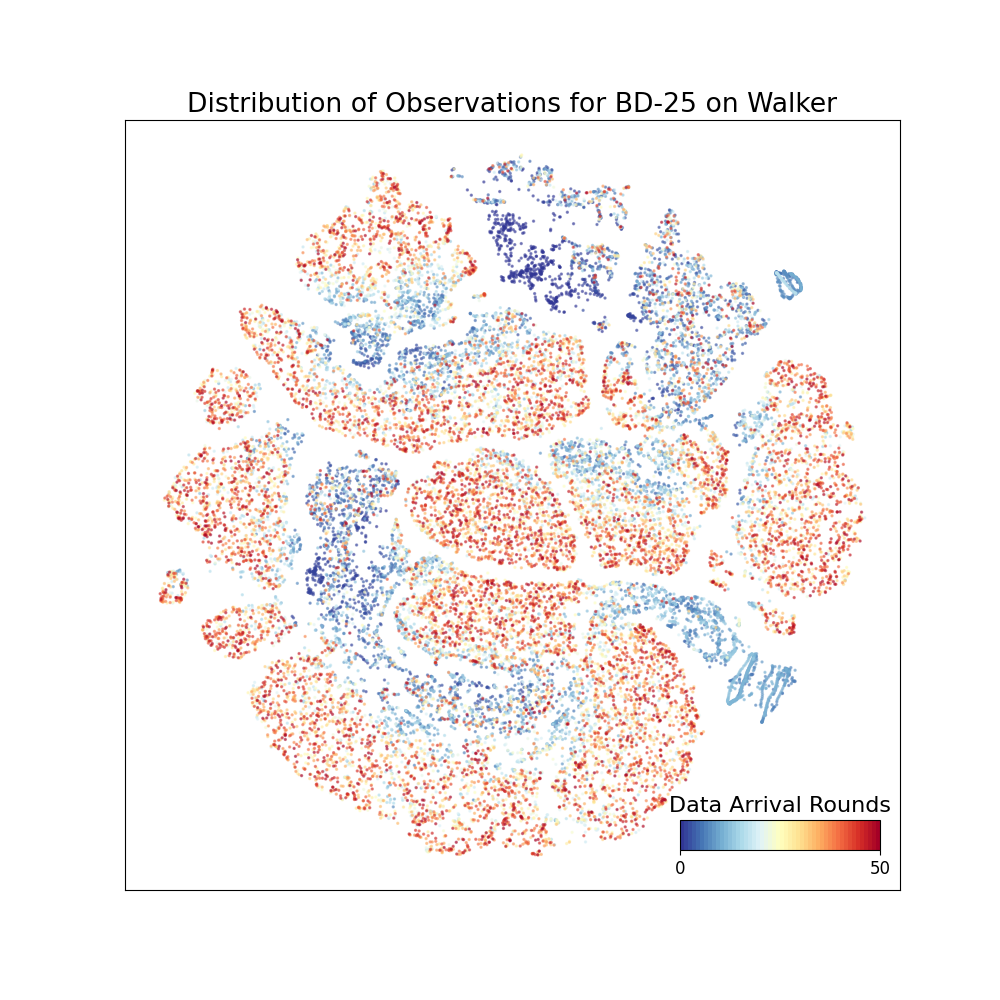}
  }
  \hspace{0.015\linewidth}
  \adjustbox{trim={0.05\width} {0.05\height} {0.05\width} {0.05\height},clip,width=0.40\linewidth, valign=t}{
    \includegraphics[width=\linewidth]{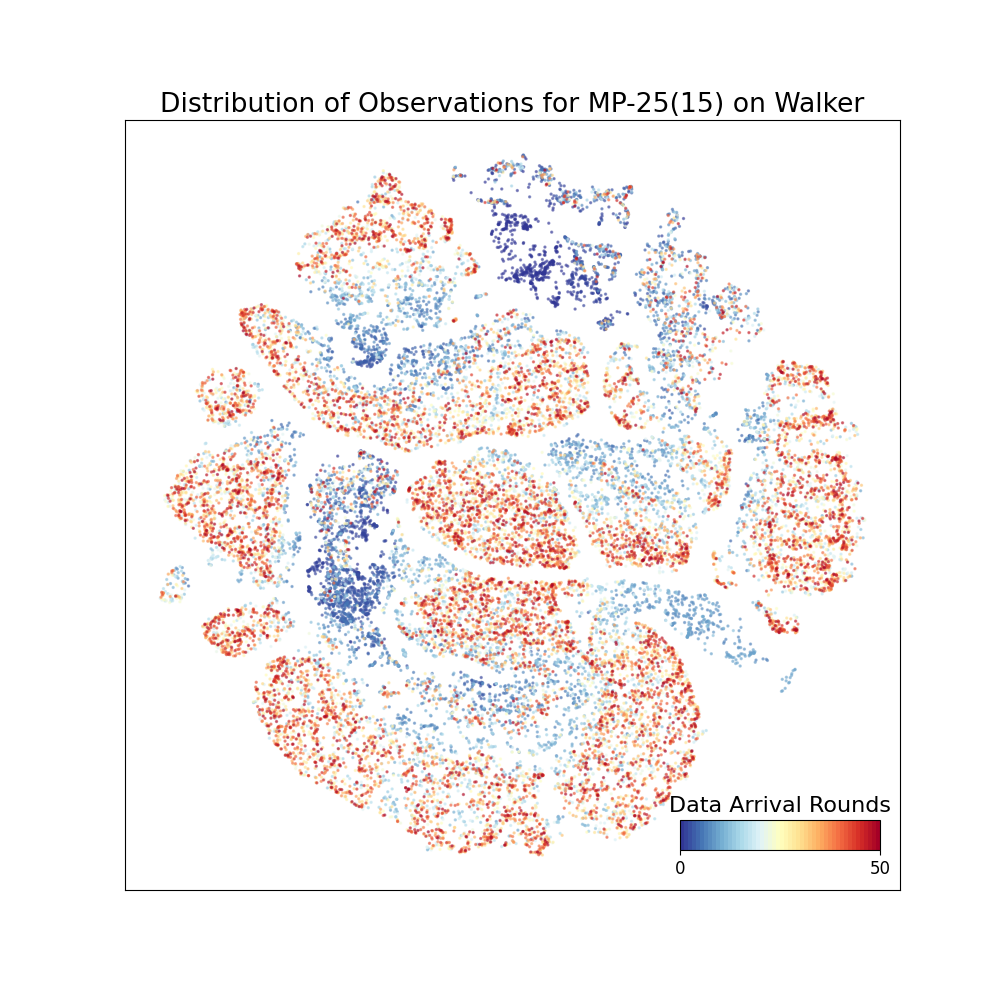}
  }
  \caption{Two-dimensional t-SNE visualizations of Walker environment states collected by different algorithms.}
  \label{fig:walker_tsne}
\end{figure*}


\end{document}